\documentclass[twoside,11pt,english]{article} 
\usepackage{jmlr2e} 
\usepackage{times}
\newcommand{\remove}[1]{}
\usepackage{paralist}
\usepackage{float}
\usepackage[T1]{fontenc}
\usepackage{amsmath}
\usepackage{color}
\usepackage{babel}
\usepackage{chngcntr}

\usepackage{url}% simple URL typesetting
\usepackage{booktabs}% professional-quality tables
\usepackage{amsfonts}% blackboard math symbols
\usepackage{nicefrac}% compact symbols for 1/2, etc.
\usepackage{microtype}% microtypography
\newcommand{\snote}[1]{{[\color{red}Suriya: #1]}}
\newcommand{\dnote}[1]{{[\color{blue}Daniel: #1]}}

\DeclareMathOperator*{\argmax}{argmax}

\newtheorem{thm}{\protect\theoremname}
\newtheorem{lem}[thm]{\protect\lemmaname}
\newtheorem{defn}[thm]{\protect\definitionname}
\usepackage{thm-restate}

\declaretheorem[name=Corollary,numberlike=thm]{corR} 
\newtheorem{assm}{Assumption}
\providecommand{\definitionname}{Definition}
\providecommand{\lemmaname}{Lemma}
\providecommand{\theoremname}{Theorem}
\providecommand{\definitionname}{Definition}
\providecommand{\lemmaname}{Lemma}
\providecommand{\theoremname}{Theorem}

\global\long\def\pmpi{\dagger}
\global\long\def\set{\mathcal{S}}
\global\long\def\w{\mathbf{w}}
\global\long\def\limitw{\mathbf{w}_{\infty}}
\global\long\def\x{\mathbf{x}}
\global\long\def\what{\hat{\mathbf{w}}}
\global\long\def\wvec{\mathbf{w}}
\global\long\def\sumn{\sum_{n=1}^{N}}
\global\long\def\xn{\mathbf{x}_{n}}
\global\long\def\e{\exp}

\global\long\def\wtilde{\tilde{\mathbf{w}}}
\global\long\def\rvec{\mathbf{r}\left(t\right)}
\global\long\def\derL{\nabla\mathcal{L}\left(\mathbf{w}\left(t\right)\right)}
\global\long\def\argmin{\mathrm{argmin}}
\global\long\def\rvec{\mathbf{r}}
\global\long\def\xnT{\mathbf{x}_{n}^{\top}}
\global\long\def\tvec{\boldsymbol{\tau}}
\global\long\def\wcheck{\check{\w}}

\newcommand{\norm}[1]{\left\lVert{#1}\right\rVert}
\newcommand{\sumk}{\sum_{k=1}^{K}}
\newcommand{\sumr}{\sum_{r=1}^{K}}
\newcommand{\sumrney}{\sum\limits_{r=1 \atop r \ne y_n}^{K}}

\newcommand{\xtilde}{\tilde{\vect{x}}_{n,k}}
\newcommand{\xtildef}[1]{\tilde{\vect{x}}_{n,#1}}
\newcommand{\xtilder}{\tilde{\vect{x}}_{n,r}}
\newcommand{\vect}[1]{\mathbf{#1}}
\newcommand{\bm}[1]{\ensuremath{\boldsymbol{#1}}}
\newcommand{\indicator}[1]{\bm{1}_{\{ #1 \}}}

\usepackage{lastpage}
\jmlrheading{19}{2018}{1-\pageref{LastPage}}{4/18}{11/18}{18-188}{Daniel Soudry, Elad Hoffer, Mor Shpigel Nacson, Suriya Gunasekar, and Nathan Srebro}
\ShortHeadings{Gradient descent on separable data}{Soudry, Hoffer, Nacson, Gunasekar, and Srebro}
\begin{document}

%\ShortHeadings{Gradient descent on separable data}{}
%\firstpageno{1}
%\begin{document}
\title{The Implicit Bias of Gradient Descent on Separable Data}
\author{\name Daniel Soudry \email {daniel.soudry@gmail.com}\\
\name Elad Hoffer \email {elad.hoffer@gmail.com}\\
\name Mor Shpigel Nacson \email {mor.shpigel@gmail.com} \\
\addr  Department of Electrical Engineering,Technion \\ 
   Haifa, 320003, Israel 
\AND
\name Suriya Gunasekar \email {suriya@ttic.edu}\\
\name Nathan Srebro \email {nati@ttic.edu}\\
\addr  Toyota Technological Institute at Chicago \\ 
   Chicago, Illinois 60637, USA 
}
\editor{Leon Bottou}

\maketitle
\begin{abstract}%
We examine gradient descent on unregularized logistic regression problems, with homogeneous linear predictors on linearly separable datasets. We show the predictor converges to the
direction of the max-margin (hard margin SVM) solution. The result also generalizes
to other monotone decreasing loss functions with an infimum at infinity, to multi-class problems, and to training a weight layer in a deep network in a certain restricted setting.  Furthermore, we show this convergence is very slow, and only
logarithmic in the convergence of the loss itself. This can help explain
the benefit of continuing to optimize the logistic or cross-entropy
loss even after the training error is zero and the training loss is
extremely small, and, as we show, even if the validation loss increases.
Our methodology can also aid in understanding implicit regularization
in more complex models and with other optimization methods. 
\end{abstract}


\section{Introduction}

It is becoming increasingly clear that implicit biases introduced
by the optimization algorithm play a crucial role in deep learning
and in the generalization ability of the learned models \citep{neyshabur2014search,neyshabur2015path,Zhang2016,Keskar2016,Neyshabur2017a,Wilson2017}.
In particular, minimizing the training error, without explicit
regularization, over models with more parameters and capacity
than the number of training examples, often yields good generalization. This is despite the fact that the empirical optimization problem being highly underdetermined.
That is, there are many global minima of the training objective, most
of which will not generalize well, but the optimization algorithm
(\emph{e.g.}~gradient descent) biases us toward a particular minimum
that {\em does} generalize well. Unfortunately, we still do not
have a good understanding of the biases introduced by different optimization
algorithms in different situations.

We do have an understanding of the implicit regularization introduced
by early stopping of stochastic methods or, at an extreme, of one-pass
(no repetition) stochastic gradient descent \citep{Hardt2015a}. However, as discussed above,
in deep learning we often benefit from implicit bias even when optimizing
the training error to convergence (without early stopping) using stochastic
or batch methods. For loss functions with attainable, finite minimizers,
such as the squared loss, we have some understanding of this: in particular,
when minimizing an underdetermined least squares problem using gradient
descent starting from the origin, it can be shown that we will converge to the
minimum Euclidean norm solution. However, the logistic loss, and its generalization
the cross-entropy loss which is often used in deep learning, do not
admit finite minimizers on separable problems. Instead, to drive
the loss toward zero and thus minimize it, the norm of the predictor must diverge
toward infinity.

Do we still benefit from implicit regularization when minimizing the
logistic loss on separable data? Clearly the norm of the predictor
itself is not minimized, since it grows to infinity. However, for
prediction, only the direction of the predictor, \emph{i.e.}~the
normalized $\mathbf{w}(t)/\norm{\mathbf{w}(t)}$, is important. How
does $\mathbf{w}(t)/\norm{\mathbf{w}(t)}$ behave as $t\rightarrow\infty$
when we minimize the logistic (or similar) loss using gradient descent
on separable data, \emph{i.e.},~when it is possible to get zero misclassification
error and thus drive the loss to zero?

In this paper, we show that even without any explicit regularization,
for all linearly separable datasets, when minimizing logistic regression problems using gradient descent, we have that $\mathbf{w}(t)/\norm{\mathbf{w}(t)}$ converges to the
$L_{2}$ maximum margin separator, \emph{i.e.}~to the solution of
the hard margin SVM for homogeneous linear predictors. This happens even though neither the norm $\left\lVert \mathbf{w}\right\rVert $,
nor the margin constraint, are part of the objective or
explicitly introduced into optimization. More generally, we show the
same behavior for generalized linear problems with any smooth, monotone
strictly decreasing, lower bounded loss with an exponential tail.
Furthermore, we characterize the rate of this convergence, and show
that it is rather slow, wherein for almost all datasets, the distance to the max-margin predictor
decreasing only as ${O}(1/\log(t))$, and in some degenerate datasets, the rate further slows down to ${O}(\log\log(t)/\log(t))$. This explains why the predictor
continues to improve even when the training loss is already extremely
small. We emphasize that this bias is specific to
gradient descent, and changing the optimization algorithm, \emph{e.g.}~using
adaptive learning rate methods such as ADAM \citep{Kingma2015}, changes
this implicit bias.

\section{Main Results}

%\subsection{Gradient descent on linearly separable problems}

Consider a dataset $\left\{ \mathbf{x}_{n},y_{n}\right\} _{n=1}^{N}$,
with $\mathbf{x}_n\in\mathbb{R}^d$ and binary labels $y_{n}\in\left\{ -1,1\right\} $. We analyze learning
by minimizing an empirical loss of the form 
\begin{equation}
\mathcal{L}\left(\mathbf{w}\right)=\sum_{n=1}^{N}\ell\left(y_{n}\mathbf{w}^{\top}\mathbf{x}_{n}\right)\,.\label{eq: general loss functions}
\end{equation}
where $\mathbf{w}\in\mathbb{R}^{d}$ is the weight vector. To simplify notation, we assume that all the labels are positive:
$\forall n:\,y_{n}=1$ \textemdash{} this is true without loss of
generality, since we can always re-define $y_{n}\mathbf{x}_{n}$ as
$\mathbf{x}_{n}$.

We are particularly interested in problems that are linearly separable,
and the loss is smooth strictly decreasing and non-negative:

{\assm{The dataset is linearly separable: $\exists\mathbf{w}_{*}$
such that $\forall n:\,\mathbf{w}_{*}^{\top}\mathbf{x}_{n}>0$ .\label{assum: Linear sepereability}}}

{\assm{$\ell\left(u\right)$ is a positive, differentiable, monotonically
decreasing to zero\footnote{The requirement of non-negativity and that the loss asymptotes to zero
is purely for convenience. It is enough to require the loss is monotone
decreasing and bounded from below. Any such loss asymptotes to some
constant, and is thus equivalent to one that satisfies this assumption,
up to a shift by that constant.}, (so $\forall u:\,\ell\left(u\right)>0,\ell^{\prime}\left(u\right)<0$, 
$\lim_{u\rightarrow\infty}\ell\left(u\right)=\lim_{u\rightarrow\infty}\ell^{\prime}\left(u\right)=0$), a $\beta$-smooth function, \emph{i.e.} its derivative is $\beta$-Lipshitz, and $\lim\sup_{u\rightarrow-\infty}\ell^{\prime}\left(u\right)<0$.\label{assum: loss properties}}}

Assumption \ref{assum: loss properties} includes many common loss functions, including the logistic, exp-loss\footnote{The exp-loss does not have a global $\beta$ smoothness parameter. However, if we initialize with $\eta < 1/ \mathcal{L}(\w(0))$ then it is straightforward to show the gradient descent iterates maintain bounded local smoothness.} and probit losses.
Assumption \ref{assum: loss properties} implies
that $\mathcal{L}\left(\mathbf{w}\right)$ is a $\beta\sigma_{\max}^{2}\left(\text{\ensuremath{\mathbf{X}} }\right)$-smooth
function, where 
$\sigma_{\max}\left(\text{\ensuremath{\mathbf{X}} }\right)$ is the
maximal singular value of the data matrix $\mathbf{X}\in\mathbb{R}^{d\times N}$.

Under these conditions, the infimum of the optimization problem is
zero, but it is not attained at any finite $\mathbf{w}$. Furthermore,
no finite critical point $\mathbf{w}$ exists. We consider minimizing
eq. \ref{eq: general loss functions} using Gradient Descent (GD)
with a fixed learning rate $\eta$, \emph{i.e., }with steps of the
form: 
\begin{equation}
\mathbf{w}\left(t+1\right)=\mathbf{w}\left(t\right)-\eta\nabla\mathcal{L}\left(\mathbf{w}(t)\right)=\mathbf{w}\left(t\right)-\eta\sum_{n=1}^{N}\ell^{\prime}\left(\mathbf{w}\left(t\right)^{\top}\mathbf{x}_{n}\right)\mathbf{x}_{n}.\label{eq: gradient descent linear}
\end{equation}
We do not require convexity. Under Assumptions 1 and 2, gradient descent
converges to the global minimum (\emph{i.e.} to zero loss) even without
it: 
\begin{lem}
\label{lem: convergence of linear classifiers}Let $\mathbf{w}\left(t\right)$
be the iterates of gradient descent (eq. \ref{eq: gradient descent linear})
with $\eta<2\beta^{-1}\sigma_{\max}^{-2}\left(\text{\ensuremath{\mathbf{X}} }\right)$
and any starting point $\w(0)$. Under Assumptions \ref{assum: Linear sepereability}
and \ref{assum: loss properties}, we have: (1) $\lim_{t\rightarrow\infty}\mathcal{L}\left(\mathbf{w}\left(t\right)\right)=0$,
(2) $\lim_{t\rightarrow\infty}\left\Vert \mathbf{w}\left(t\right)\right\Vert =\infty$,
and (3) $\forall n:\,\lim_{t\rightarrow\infty}\mathbf{w}\left(t\right)^{\top}\mathbf{x}_{n}=\infty$.
\end{lem}

\begin{proof}
Since the data is linearly separable, $\exists\mathbf{w}_{*}$
which linearly separates the data, and therefore 
\[
\mathbf{w}_{*}^{\top}\nabla\mathcal{L}\left(\mathbf{w}\right)=\sum_{n=1}^{N}\ell^{\prime}\left(\mathbf{w}^{\top}\mathbf{x}_{n}\right)\mathbf{w}_{*}^{\top}\mathbf{x}_{n}.
\]
For any finite $\mathbf{w}$, this sum cannot be equal to zero, as
a sum of negative terms, since $\forall n:\,\mathbf{w}_{*}^{\top}\mathbf{x}_{n}>0$
and $\forall u:\,\ell^{\prime}\left(u\right)<0$. Therefore, there
are no finite critical points $\mathbf{w}$, for which $\nabla\mathcal{L}\left(\mathbf{w}\right)=\mathbf{0}$.
But gradient descent on a smooth loss with an appropriate stepsize
is always guaranteed to converge to a critical point: $\nabla\mathcal{L}\left(\mathbf{w}\left(t\right)\right)\rightarrow\mathbf{0}$
(see, \emph{e.g.} Lemma \ref{lem: GD convergence} in Appendix \ref{sec:Proof of GD convergence},
slightly adapted from \citet{Ganti2015}, Theorem 2). This necessarily
implies that $\left\Vert \mathbf{w}\left(t\right)\right\Vert \rightarrow\infty$
while $\forall n:\,\mathbf{w}\left(t\right)^{\top}\mathbf{x}_{n}>0$
for large enough $t$\textemdash since only then $\ell^{\prime}\left(\mathbf{w}\left(t\right)^{\top}\mathbf{x}_{n}\right)\rightarrow0$.
Therefore, $\mathcal{L}\left(\mathbf{w}\right)\rightarrow0$, so GD
converges to the global minimum. 
\end{proof}
The main question we ask is: can we characterize the direction in
which $\mathbf{w}(t)$ diverges? That is, does the limit $\lim_{t\rightarrow\infty}\mathbf{w}\left(t\right)/\left\lVert \mathbf{w}\left(t\right)\right\rVert $
always exist, and if so, what is it?

In order to analyze this limit, we will need to make a further assumption
on the tail of the loss function: 
\begin{defn}
\label{def: exponential tail}A function $f\left(u\right)$ has a
``tight exponential tail'', if there exist positive constants $c,a,\mu_{+},\mu_{-},u_{+}$
and $u_{-}$ such that 
\begin{align*}
\forall u>u_{+}: & f\left(u\right)\leq c\left(1+\exp\left(-\mu_{+}u\right)\right)e^{-au}\\
\forall u>u_{-}: & f\left(u\right)\geq c\left(1-\exp\left(-\mu_{-}u\right)\right)e^{-au}\,.
\end{align*}
\end{defn}

{\assm{\label{assum: exponential tail}The negative loss derivative
$-\ell^{\prime}\left(u\right)$ has a tight exponential tail (Definition
\ref{def: exponential tail}).}}

For example, the exponential loss $\ell\left(u\right)=e^{-u}$ and
the commonly used logistic loss $\ell\left(u\right)=\log\left(1+e^{-u}\right)$
both follow this assumption with $a=c=1$. We will assume $a=c=1$
\textemdash{} without loss of generality, since these constants can
be always absorbed by re-scaling $\mathbf{x}_{n}$ and $\eta$.

We are now ready to state our main result:
\begin{restatable}{thmR}{LRasymptotic}

\label{thm: main theorem} For any dataset which is linearly separable (Assumption
\ref{assum: Linear sepereability}), any $\beta$-smooth decreasing
loss function (Assumption \ref{assum: loss properties}) with an exponential
tail (Assumption \ref{assum: exponential tail}), any stepsize $\eta<2\beta^{-1}\sigma_{\max}^{-2}\left(\text{\ensuremath{\mathbf{X}} }\right)$
and any starting point $\w(0)$, the gradient descent iterates (as in eq.~\ref{eq: gradient descent linear}) will behave as: 
\begin{equation}
\mathbf{w}\left(t\right)=\hat{\mathbf{w}}\log t+\boldsymbol{\rho}\left(t\right)\,,\label{eq: asymptotic form}
\end{equation}
where $\hat{\mathbf{w}}$ is the $L_{2}$ max margin vector (the solution
to the hard margin SVM): 
\begin{equation}
\hat{\mathbf{w}}=\underset{\mathbf{\mathbf{w}}\in\mathbb{R}^{d}}{\mathrm{argmin}}\left\lVert \mathbf{w}\right\rVert ^{2}\,\,\mathrm{s.t.}\,\,\mathbf{w}^{\top}\mathbf{x}_{n}\geq1,\label{eq: max margin vector}
\end{equation}
and the residual grows at most as $\norm{\boldsymbol{\rho}\left(t\right)}=O(\log \log (t) )$, and
so 
\[
\lim_{t\rightarrow\infty}\frac{\mathbf{w}\left(t\right)}{\left\Vert \mathbf{w}\left(t\right)\right\Vert }=\frac{\hat{\mathbf{w}}}{\left\Vert \hat{\mathbf{w}}\right\Vert }.
\]
Furthermore, for almost all data sets (all except measure zero), the residual $\rho(t)$ is bounded.
\end{restatable}

\paragraph{Proof Sketch (complete proof in the appendix)}

We first understand intuitively why an exponential tail of the loss
entail asymptotic convergence to the max margin vector: Assume for
simplicity that $\ell\left(u\right)=e^{-u}$ exactly, and examine
the asymptotic regime of gradient descent in which $\forall n:\,\mathbf{w}\left(t\right)^{\top}\mathbf{x}_{n}\rightarrow\infty$,
as is guaranteed by Lemma \ref{lem: convergence of linear classifiers}.
Suppose $\mathbf{w}\left(t\right)/\left\Vert \mathbf{w}\left(t\right)\right\Vert $
converges to some limit $\limitw$ such so we can write $\mathbf{w}\left(t\right)=g\left(t\right)\limitw+\boldsymbol{\rho}\left(t\right)$
such that $g\left(t\right)\rightarrow\infty$, $\forall n:$$\mathbf{x}_{n}^{\top}\limitw>0$,
and $\lim_{t\rightarrow\infty}\boldsymbol{\rho}\left(t\right)/g\left(t\right)=0$.
The gradient can then be written as: 
\begin{equation}
-\nabla\mathcal{L}\left(\mathbf{w}\right)=\sum_{n=1}^{N}\exp\left(-\mathbf{w}\left(t\right)^{\top}\mathbf{x}_{n}\right)\mathbf{x}_{n}=\sum_{n=1}^{N}\exp\left(-g\left(t\right)\limitw^{\top}\mathbf{x}_{n}\right)\exp\left(-\boldsymbol{\rho}\left(t\right)^{\top}\mathbf{x}_{n}\right)\mathbf{x}_{n}\,.\label{eq: gradient exp}
\end{equation}
As $g(t)\rightarrow\infty$ and the exponents become more negative,
only those samples with the largest (\emph{i.e.}, least negative)
exponents will contribute to the gradient. These are precisely the
samples with the smallest margin $\mathrm{argmin}_{n}\mathbf{w}_{\infty}^{\top}\mathbf{x}_{n}$,
aka the ``support vectors''. The negative gradient (eq. \ref{eq: gradient exp})
would then asymptotically become a non-negative linear combination
of support vectors. The limit $\limitw$ will then be dominated by
these gradients, since any initial conditions become negligible as
$\left\Vert \mathbf{w}\left(t\right)\right\Vert \rightarrow\infty$
(from Lemma \ref{lem: convergence of linear classifiers}). Therefore,
$\limitw$ will also be a non-negative linear combination of support
vectors, and so will its scaling $\hat{\mathbf{w}}=\limitw/\left(\min_{n}\mathbf{w}_{\infty}^{\top}\mathbf{x}_{n}\right)$.
We therefore have: 
\begin{equation}
\what=\sum_{n=1}^{N}\alpha_{n}\x_{n}\quad\quad\forall n\;\left(\alpha_{n}\geq0\;\textrm{and}\;\what^{\top}\x_{n}=1\right)\;\;\textrm{OR}\;\;\left(\alpha_{n}=0\;\textrm{and}\;\what^{\top}\x_{n}>1\right)\label{eq:kkt}
\end{equation}
These are precisely the KKT conditions for the SVM problem (eq. \ref{eq: max margin vector})
and we can conclude that $\what$ is indeed its solution and $\limitw$
is thus proportional to it.

To prove Theorem \ref{thm: main theorem} rigorously, we need to show
that $\mathbf{w}\left(t\right)/\left\Vert \mathbf{w}\left(t\right)\right\Vert $
has a limit, that  $\forall n: \limitw^{\top}\mathbf{x}_{n}>0$, that $g\left(t\right)=\log\left(t\right)$ and to bound
the effect of various residual errors, such as gradients of non-support
vectors and the fact that the loss is only approximately exponential.
To do so, we substitute eq. \ref{eq: asymptotic form} into the gradient
descent dynamics (eq. \ref{eq: gradient descent linear}), with $\limitw=\hat{\mathbf{w}}$
being the max margin vector and $g(t)=\log t$. We then show that, except when certain degeneracies occur, 
the increment in the norm of $\boldsymbol{\rho}\left(t\right)$ is
bounded by $C_{1}t^{-\nu}$ for some $C_{1}>0$ and $\nu>1$, which
is a converging series. This happens because the increment in the
max margin term, $\hat{\mathbf{w}}\left[\log\left(t+1\right)-\log\left(t\right)\right]\approx\hat{\mathbf{w}}t^{-1}$,
cancels out the dominant $t^{-1}$ term in the gradient $-\nabla\mathcal{L}\left(\mathbf{w}\left(t\right)\right)$
(eq. \ref{eq: gradient exp} with $g\left(t\right)=\log\left(t\right)$
and $\limitw^{\top}\mathbf{x}_{n}=1$).

\paragraph{Degenerate and Non-Degenerate Data Sets}  An earlier conference version of this paper \citep{Soudry2017a} included a partial version of Theorem \ref{thm: main theorem}, which only applies to almost all data sets, in which case we can ensure the residual $\rho(t)$ is bounded.  This partial statement (for almost all data sets) is restated and proved as Theorem \ref{thm: main theorem almost everywhere} in Appendix \ref{sec:proof}.  It applies, \emph{e.g.}~with probability one for data sampled from any absolutely continuous distribution.  It does not apply  in ``degenerate'' cases where some of the support vectors $\x_n$ (for which $\what^\top \x_n=1$) are associated with dual variables that are zero ($\alpha_n=0$) in the dual optimum of \ref{eq: max margin vector}.  As we show in Appendix \ref{sec:alpha}, this only happens on measure zero data sets.  Here, we prove the more general result which applies for all data sets, including degenerate data sets.  To do so, in Theorem \ref{theorem: main2} in Appendix \ref{sec: proof of degenerate case} we provide a more complete characterization of the iterates $\w(t)$ that explicitly specifies all unbounded components even in the degenerate case.  We then prove the Theorem by plugging in this more complete characterization and showing that the residual is bounded, thus also establishing Theorem \ref{thm: main theorem}.

\paragraph{Parallel Work on the Degenerate Case} Following the publication of our initial version, and while preparing this revised version for publication, we learned of parallel work by Ziwei Ji and Matus Telgarsky that also closes this gap.  \cite{Ji2018}  provide an analysis of the degenerate case, establishing converges to the max margin predictor by showing that $\norm{ \frac{\w(t)}{\norm{\w(t)}} - \frac{\what}{\norm{\what}}}=O\left(\sqrt{\frac{\log \log t}{\log t}}\right)$.  Our analysis provides a more precise characterization of the iterates, and also shows the convergence is actually quadratically faster (see Section \ref{sec: convergence rates}).  However, Ji and Telgarsky go even further and provide a characterization also when the data is non-separable but $\w(t)$ still goes to infinity.

\paragraph{More Refined Analysis of the Residual}

In some non-degenerate cases, we can further characterize the asymptotic behavior of $\boldsymbol{\rho}\left(t\right)$.  To do so, we need to refer to the KKT conditions (eq.~\ref{eq:kkt})
of the SVM problem (eq.~\ref{eq: max margin vector}) and the associated
support vectors $\set=\mathrm{argmin}_{n}\what^{\top}\mathbf{x}_{n}$.  We then have the following Theorem, 
proved in Appendix \ref{sec:proof}:

\begin{restatable}{thmR}{LRasymptotic2}

\label{thm: refined Theorem} Under the conditions and notation of
Theorem \ref{thm: main theorem}, for almost all datasets, if in addition the support vectors
span the data (\emph{i.e.}~$\mathrm{rank}\left(\mathbf{X}_{\set}\right)=\mathrm{rank}\left(\mathbf{X}\right)$, 
where $\mathbf{X}_{\set}$ is a matrix whose columns are only those data points $\x_n$ s.t.~$\what^\top\x_n=1$), then $\lim_{t\rightarrow\infty}\boldsymbol{\rho}\left(t\right)=\tilde{\mathbf{w}}$,
where $\tilde{\mathbf{w}}$ is a solution to 
\begin{equation}
\forall n\in\set:\,\eta\exp\left(-\mathbf{x}_{n}^{\top}\tilde{\mathbf{w}}\right)=\alpha_{n}\,\label{eq: w tilde}
\end{equation}

\end{restatable}

\paragraph{Analogies with Boosting}  Perhaps most similar to our study is the line of work on understanding AdaBoost in terms its implicit bias toward large $L_1$-margin solutions, starting with the seminal work of \cite{schapire1998boosting}.  Since AdaBoost can be viewed as coordinate descent on the exponential loss of a linear model, these results can be interpreted as analyzing the bias of coordinate descent, rather then gradient descent, on a monotone decreasing loss with an exact exponential tail.  Indeed, with small enough step sizes, such a coordinate descent procedure does converge precisely to the maximum $L_1$-margin solution \citep{zhang2005boosting,telgarsky2013margins}.  In fact, \cite{telgarsky2013margins} also generalizes these results to other losses with tight exponential tails, similar to the class of losses we consider here.

Also related is the work of \citet{Rosset2004}.
They considered the regularization path $\w_{\lambda}=\arg\min\mathcal{L}(\w)+\lambda\norm{\w}_{p}^{p}$
for similar loss functions as we do, and showed that $\lim_{\lambda\rightarrow0}\w_{\lambda}/\norm{\w_{\lambda}}_{p}$
is proportional to the maximum $L_{p}$ margin solution.  That is, they showed how adding infinitesimal $L_p$ (\emph{e.g.}~$L_1$ and $L_2$) regularization to logistic-type losses gives rise to the corresponding max-margin predictor.\footnote{In contrast, with non-vanishing regularization (\emph{i.e.}, $\lambda>0$), $\arg\min_{\w}\mathcal{L}(\w)+\lambda\norm{\w}_{p}^{p}$ is generally \emph{not} a max margin solution.} However, \citeauthor{Rosset2004}
do not consider the effect of the optimization algorithm, and instead
add explicit regularization. Here we are specifically interested
in the bias implied by the algorithm {\em not} by adding (even
infinitesimal) explicit regularization.  We see that coordinate descent gives rise to the max $L_1$ margin predictor, while gradient descent gives rise to the max $L_2$ norm predictor.  In Section \ref{sec:other_opt} and in follow-up work \citep{gunasekar2018characterizing} we discuss also other optimization algorithms, and their implied biases.

\paragraph{Non-homogeneous linear predictors} In this paper we focused on homogeneous linear predictors of the form $\wvec^\top \x$, similarly to previous works (\emph{e.g.}, \citet{Rosset2004,telgarsky2013margins}). Specifically, we did not have the common intercept term: $\wvec^\top \x + b$. One may be tempted to introduce the intercept in the usual way, \emph{i.e.}, by extending all the input vectors $\x_n$  with an additional $'1'$ component. In this extended input space, naturally, all our results hold. Therefore, we converge in direction to the $L_2$ max margin solution (eq. \ref{eq: max margin vector}) in the extended space. However, if we translate this solution to the original $\x$ space  we obtain $$
\underset{\mathbf{\mathbf{w}}\in\mathbb{R}^{d},b\in\mathbb{R}}{\mathrm{argmin}}\left\lVert \mathbf{w}\right\rVert ^{2}+b^2\,\,\mathrm{s.t.}\,\,\mathbf{w}^{\top}\mathbf{x}_{n}+b\geq1, $$
which is not the $L_2$ max margin (SVM) solution
$$
\underset{\mathbf{\mathbf{w}}\in\mathbb{R}^{d},b\in\mathbb{R}}{\mathrm{argmin}}\left\lVert \mathbf{w}\right\rVert ^{2}\,\,\mathrm{s.t.}\,\,\mathbf{w}^{\top}\mathbf{x}_{n}+b\geq1,$$
\remove{What rule does $b$ have in this equation? In https://openreview.net/pdf?id=ByfbnsA9Km they said (p.4 eq. (P1)) that the svm solution with intercept should solve this equation with $\wvec^\top\xn+b\ge 1$}where we do not have a $b^2$ penalty in the objective\remove{As stated above, I don't think this should be the only difference}.

\section{Implications: Rates of convergence\label{sec: convergence rates}}

The solution in eq. \ref{eq: asymptotic form} implies that $\mathbf{w}\left(t\right)/\left\Vert \mathbf{w}\left(t\right)\right\Vert $
converges to the normalized max margin vector $\hat{\mathbf{w}}/\left\Vert \hat{\mathbf{w}}\right\Vert .$
Moreover, this convergence is very slow\textemdash{} logarithmic in
the number of iterations. Specifically, our results imply the following
tight rates of convergence:
\begin{restatable}{thmR}{ratetheorem}%$\!\!\!$\textbf{\emph{A}}
\label{thm: rates} Under the conditions and notation of Theorem \ref{thm: main theorem}, for any linearly separable data set, the normalized weight vector converges to the normalized max margin vector
in $L_{2}$ norm 
\begin{equation}
\left\Vert \frac{\mathbf{w}\left(t\right)}{\left\Vert \mathbf{w}\left(t\right)\right\Vert }-\frac{\hat{\mathbf{w}}}{\left\Vert \hat{\mathbf{w}}\right\Vert }\right\Vert =O\left(\frac{\log\log t}{\log t}\right)\,,\label{eq: normalized weight vector}
\end{equation}
with this rate improving to $O(1/\log(t))$ for almost every dataset; and in angle 
\begin{equation}
1-\frac{\mathbf{w}\left(t\right)^{\top}\hat{\mathbf{w}}}{\left\Vert \mathbf{w}\left(t\right)\right\Vert \left\Vert \hat{\mathbf{w}}\right\Vert }=O\left(\left(\frac{\log\log t}{\log t}\right)^2\right)\,,\label{eq: angle}
\end{equation}
with this rate improving to $O(1/\log^2(t))$ for almost every dataset;
and the margin converges as 
\begin{equation}
\frac{1}{\left\Vert \hat{\mathbf{w}}\right\Vert }-\frac{\min_{n}\mathbf{x}_{n}^{\top}\mathbf{w}\left(t\right)}{\left\Vert \mathbf{w}\left(t\right)\right\Vert }=O\left(\frac{1}{\log t}\right)\,.\label{eq: margin}
\end{equation}
On the other hand, the loss itself decreases as
\begin{equation}
\mathcal{L}\left(\mathbf{w}\left(t\right)\right)=O\left({\frac{1}{t}}\right)\,.\label{eq: logistic loss convergence}
\end{equation}
\end{restatable}
All the rates in the above Theorem are a direct consequence of Theorem \ref{thm: main theorem}, except for avoiding the $\log \log t$ factor for the degenerate cases in eq. \ref{eq: margin} and eq. \ref{eq: logistic loss convergence} (\emph{i.e.}, establishing that the rates  $1/\log t$ and $1/t$ always hold)---this additional improvement is a consequence of the more complete characterization of Theorem \ref{theorem: main2}.  Full details are provided in Appendix \ref{sec:Calculation-of-convergence rates}.  In this appendix, we also provide a simple construction showing all the rates in Theorem \ref{thm: rates} are tight (except possibly for the $\log\log t$ factors).  

The sharp contrast between the tight logarithmic and $1/t$ rates in Theorem \ref{thm: rates} implies that the convergence
of $\mathbf{w}(t)$ to the max-margin $\what$ can be logarithmic
in the loss itself, and we might need to wait until the loss is exponentially
small in order to be close to the max-margin solution. This can help
explain why continuing to optimize the training loss, even after the
training error is zero and the training loss is extremely small, still
improves generalization performance\textemdash our results suggests
that the margin could still be improving significantly in this regime.

A numerical illustration of the convergence is depicted in Figure
\ref{fig:Synthetic-dataset}. As predicted by the theory, the norm
$\left\Vert \mathbf{w}(t)\right\Vert $ grows logarithmically (note
the semi-log scaling), and $\w(t)$ converges to the max-margin separator,
but only logarithmically, while the loss itself decreases very rapidly
(note the log-log scaling).

An important practical consequence of our theory, is that although
the margin of $\mathbf{w}(t)$ keeps improving, and so we can expect
the population (or test) misclassification error of $\mathbf{w}(t)$
to improve for many datasets, the same cannot be said about the expected
population loss (or test loss)! At the limit, the direction of $\mathbf{w}(t)$
will converge toward the max margin predictor $\what$. Although $\what$
has zero training error, it will not generally have zero misclassification
error on the population, or on a test or a validation set. Since the
norm of $\mathbf{w}(t)$ will increase, if we use the logistic loss
or any other convex loss, the loss incurred on those misclassified
points will also increase. More formally, consider the logistic loss
$\ell(u)=\log(1+e^{-u})$ and define also the hinge-at-zero loss $h(u)=\max(0,-u)$.
Since $\what$ classifies all training points correctly, we have that
on the training set $\sum_{n=1}^{N}h(\what^{\top}\x_{n})=0$. However,
on the population we would expect some errors and so $\mathbb{E}[h(\what^{\top}\x)]>0$.
Since $\mathbf{w}(t)\approx\what\log t$ and $\ell(\alpha u)\rightarrow\alpha h(u)$
as $\alpha\rightarrow\infty$, we have: 
\begin{equation}
\mathbb{E}[\ell(\mathbf{w}(t)^{\top}\x)]\approx\mathbb{E}[\ell((\log t)\what^{\top}\x)]\approx(\log t)\mathbb{E}[h(\what^{\top}\x)]=\Omega(\log t).\label{eq:poploss}
\end{equation}
That is, the population loss increases logarithmically while the margin
and the population misclassification error improve. Roughly speaking,
the improvement in misclassification does not out-weight the increase
in the loss of those points still misclassified.

The increase in the test loss is practically important because the
loss on a validation set is frequently used to monitor progress and
decide on stopping. Similar to the population loss, the validation
loss will increase
logarithmically with $t$, if there is at least one sample in the validation set which is classified incorrectly by the max margin vector (since we would not expect zero validation
error). More precisely, as a direct consequence of Theorem \ref{thm: main theorem} (as shown on Appendix \ref{sec:Calculation-of-convergence rates}):
\begin{corR}
Let $\ell$ be the logistic loss, and $\mathcal{V}$ be an independent validation set, for which $\exists \x\in \mathcal{V}$ such that $\x^{\top} \what < 0$. Then the validation loss increases as
\[\mathcal{L}_{\mathrm{val}}\left(\mathbf{w}\left(t\right)\right)=\sum_{\x\in\mathcal{V}}\ell\left(\mathbf{w}\left(t\right)^{\top}\mathbf{x}\right)=\Omega(\log (t) ).\]
\end{corR}

This behavior might cause us to think we are over-fitting or otherwise
encourage us to stop the optimization. However, this increase does not
actually represent the model getting worse, merely $\norm{\mathbf{w}(t)}$
getting larger, and in fact the model might be getting better (increasing the margin and possibly decreasing the error rate).

\begin{figure}[t!]
\begin{centering}
\includegraphics[width=1\columnwidth]{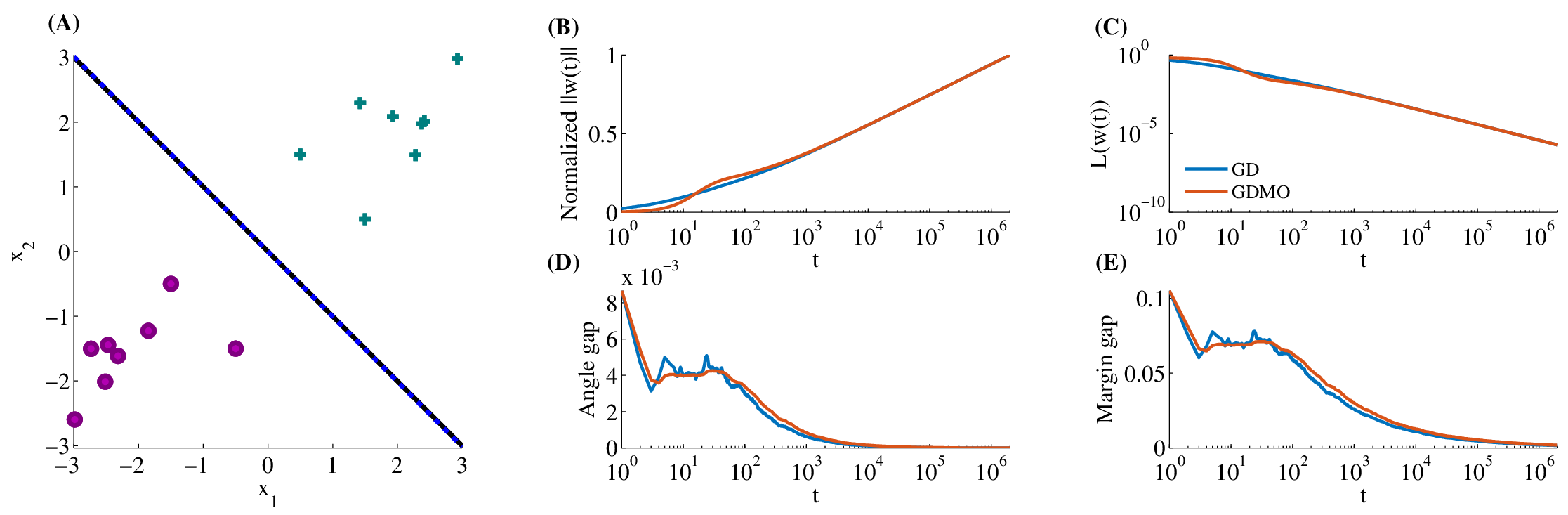} 
\par\end{centering}
\caption{Visualization of or main results on a synthetic dataset in which the
$L_{2}$ max margin vector $\hat{\mathbf{w}}$ is precisely known.
\textbf{(A)} The dataset (positive and negatives samples ($y=\pm1$)
are respectively denoted by $'+'$ and $'\circ'$), max margin separating
hyperplane (black line), and the asymptotic solution of GD (dashed
blue). For both GD and GD with momentum (GDMO), we show: \textbf{(B)
}The norm of $\mathbf{w}\left(t\right)$, normalized so it would equal
to $1$ at the last iteration, to facilitate comparison. As expected
(eq. \ref{eq: asymptotic form}), the norm increases logarithmically;
\textbf{(C) }the training loss. As expected, it decreases as $t^{-1}$
(eq. \ref{eq: logistic loss convergence}); and \textbf{(D\&E) }the
angle and margin gap of $\mathbf{w}\left(t\right)$ from $\hat{\mathbf{w}}$
(eqs. \ref{eq: angle} and \ref{eq: margin}). As expected, these
are logarithmically decreasing to zero. \textbf{Implementation details:}
The dataset includes four support vectors: $\mathbf{x}_{1}=\left(0.5,1.5\right),\mathbf{x}_{2}=\left(1.5,0.5\right)$
with $y_{1}=y_{2}=1$, and $\mathbf{x}_{3}=-\mathbf{x}_{1}$, $\mathbf{x}_{4}=-\mathbf{x}_{2}$
with $y_{3}=y_{4}=-1$ (the $L_{2}$ normalized max margin vector
is then $\hat{\mathbf{w}}=\left(1,1\right)/\sqrt{2}$ with margin
equal to $\sqrt{2}$ ), and $12$ other random datapoints ($6$ from
each class), that are not on the margin. We used a learning rate $\eta=1/\sigma^2_{\max}\left(\mathbf{X}\right)$,
where $\sigma^2_{\max}\left(\mathbf{X}\right)$ is the maximal singular
value of $\mathbf{X}$, momentum $\gamma=0.9$ for GDMO, and initialized
at the origin. \label{fig:Synthetic-dataset}}
\end{figure}

\section{Extensions}

\subsection{Multi-Class Classification with Cross-Entropy Loss}

So far, we have discussed the problem of binary classification, but in many practical situations, we have more than two classes. For
multi-class problems, the labels are the
class indices $y_{n}\in[K]\triangleq\left\{ 1,\dots,K\right\}$ and we learn a predictor
$\mathbf{w}_{k}$ for each class $k\in[K]$. A common loss function in multi-class classification is the following cross-entropy loss with
a softmax output, which is a generalization of the logistic loss:
\begin{align}\label{eq: multi loss}
\mathcal{L}\left(\{\mathbf{w}_{k}\}_{k\in[K]}\right) & =-\sum_{n=1}^{N}\log\left(\frac{\exp\left(\mathbf{w}_{y_{n}}^{\top}\mathbf{x}_{n}\right)}{\sum_{k=1}^{K}\exp\left(\mathbf{w}_{k}^{\top}\mathbf{x}_{n}\right)}\right)
\end{align}
What do the linear predictors $\mathbf{w}_{k}(t)$ converge to if
we minimize the cross-entropy loss by gradient descent on the predictors?
In Appendix \ref{sec:Softmax-output-with-cross-entropy-loss} we analyze
this problem for separable data and show that again, the predictors
diverge to infinity and the loss converges to zero. Next, to answer to which direction do these predictors converge, we define $\hat{\vect{w}}_k$ as the solution of the $K$-class SVM:
	\begin{equation} \label{eq: K-class SVM}
	\argmin_{\vect{w}_1,...,\vect{w}_K}\sumk ||\vect{w}_k||^2\,\textrm{s.t.}\,\forall n,\forall k \ne y_n : \vect{w}_{y_n}^\top\vect{x}_n \ge \vect{w}_k^\top\vect{x}_n+1,
	\end{equation} 
for each $k \in [K]$, define $\mathcal{S}_k=\arg \min_n(\hat{\mathbf{w}}_{y_n}-\hat{\mathbf{w}}_k)^\top\mathbf{x}_n 
=\{n :(\hat{\mathbf{w}}_{y_n}-\hat{\mathbf{w}}_k)^\top\mathbf{x}_n  =1\}
$, i.e., the $k^{th}$ class support vectors, and define $\alpha_{n,k}$ as some positive dual variables for $\mathcal{S}_k$ that together satisfy the $K$-class SVM KKT conditions. Using these definitions, we prove the following Theorem:
\begin{restatable}{thmR}{mainMulti}
\label{thm:mainMulti}
	For all multiclass datasets which are linearly separable (i.e.~the constraints in eq.~\ref{eq: K-class SVM} below are feasible) and for which the equation
     \begin{equation}
    \label{equation:w-tilde-defining-condition}
    \forall k \in [K], \forall n \in \mathcal{S}_k:\,\eta\exp\left(-\mathbf{x}_{n}^{\top}\left(\tilde{\mathbf{w}}_{y_n}-\tilde{\mathbf{w}}_{k}\right)\right)=\alpha_{n, k},
    \end{equation}
    has a solution $\{{\tilde{\mathbf{w}}_k}\}_{k=1}^K$, the following holds: for any starting point $\wvec(0)$ and any small enough stepsize, the iterates of gradient descent on eq. \ref{eq: multi loss} will behave as: 
	\begin{equation} \label{eq:5}
	\vect{w}_k(t) = \hat{\vect{w}}_k \log(t)+\bm{\rho}_k(t),
	\end{equation}
	where the residual $\bm{\rho}_k(t)$ is bounded.
\end{restatable}

Note that here we had to assume eq. \ref{equation:w-tilde-defining-condition} has a solution. In the binary case, we could prove that this equation has a solution for almost every dataset. In the original version of this manuscript, we incorrectly assumed that this proof in the binary case carries to the multiclass case (as was pointed to us by Yutong Wang). We therefore added the assumption that eq. \ref{equation:w-tilde-defining-condition} has a solution. We conjecture this assumption should also be true for almost all datasets in the multiclass case (see Appendix H in \cite{Ravi2024}), but we leave this proof for future work.

% \begin{restatable}{thmR}{mainMulti}
% \label{thm:mainMulti}
% 	For almost all multiclass datasets ({\em{i.e.}},
% 	except for a measure zero) which are linearly separable (i.e.~the constraints in eq.~\ref{eq: K-class SVM} below are feasible), any starting point $\wvec(0)$ and any small enough stepsize, the iterates of gradient descent on \ref{eq: multi loss} will behave as: 
% 	\begin{equation} \label{eq:5}
% 	\vect{w}_k(t) = \hat{\vect{w}}_k \log(t)+\bm{\rho}_k(t),
% 	\end{equation}
% 	where the residual $\bm{\rho}_k(t)$ is bounded and $\hat{\vect{w}}_k$ is the solution of the K-class SVM:
% 	\begin{equation} \label{eq: K-class SVM}
% 	\argmin_{\vect{w}_1,...,\vect{w}_K}\sumk ||\vect{w}_k||^2\,\textrm{s.t.}\,\forall n,\forall k \ne y_n : \vect{w}_{y_n}^\top\vect{x}_n \ge \vect{w}_k^\top\vect{x}_n+1.
% 	\end{equation}
% \end{restatable}

\subsection{Deep networks}

\begin{figure}
\begin{centering}
\begin{tabular}{ccc}
\includegraphics[width=0.31\columnwidth]{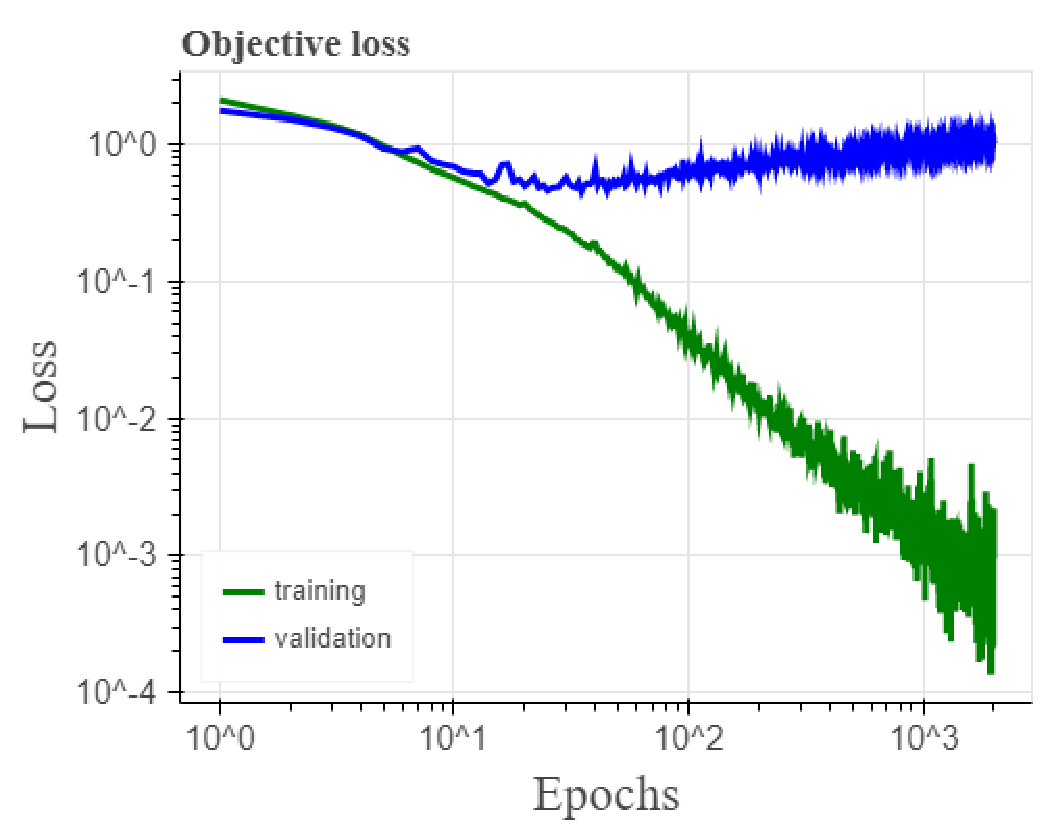}  & \includegraphics[width=0.31\columnwidth]{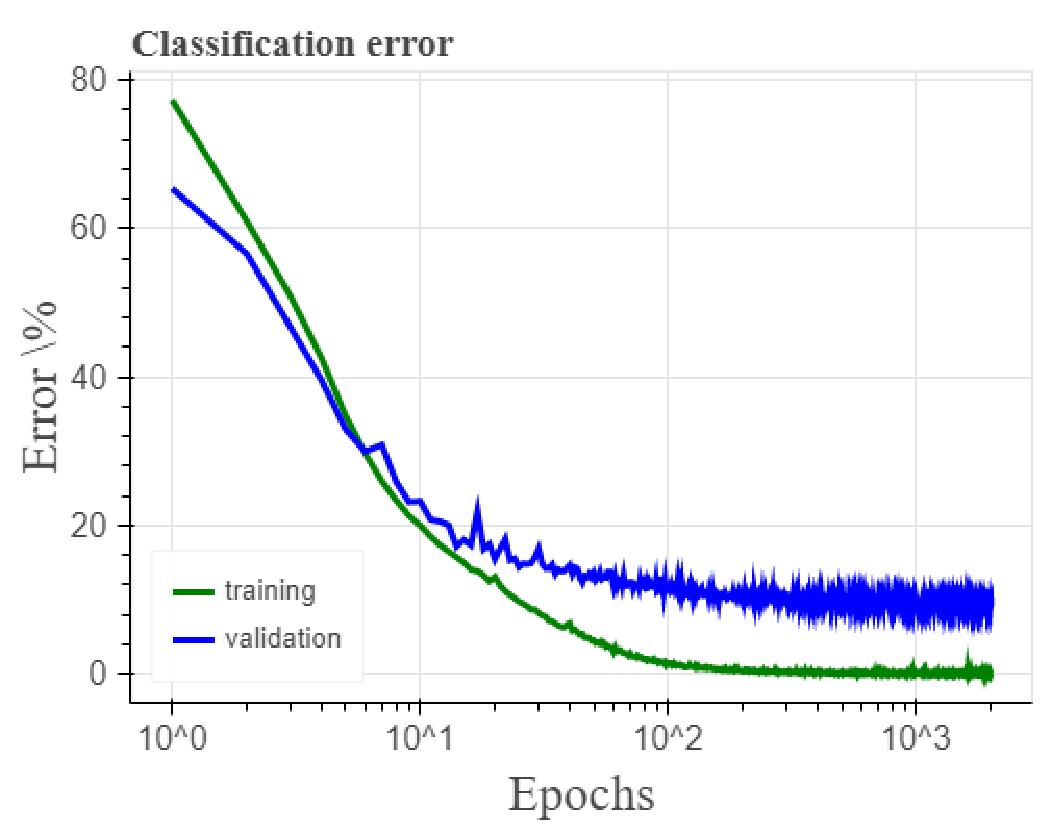}  & \includegraphics[width=0.31\columnwidth]{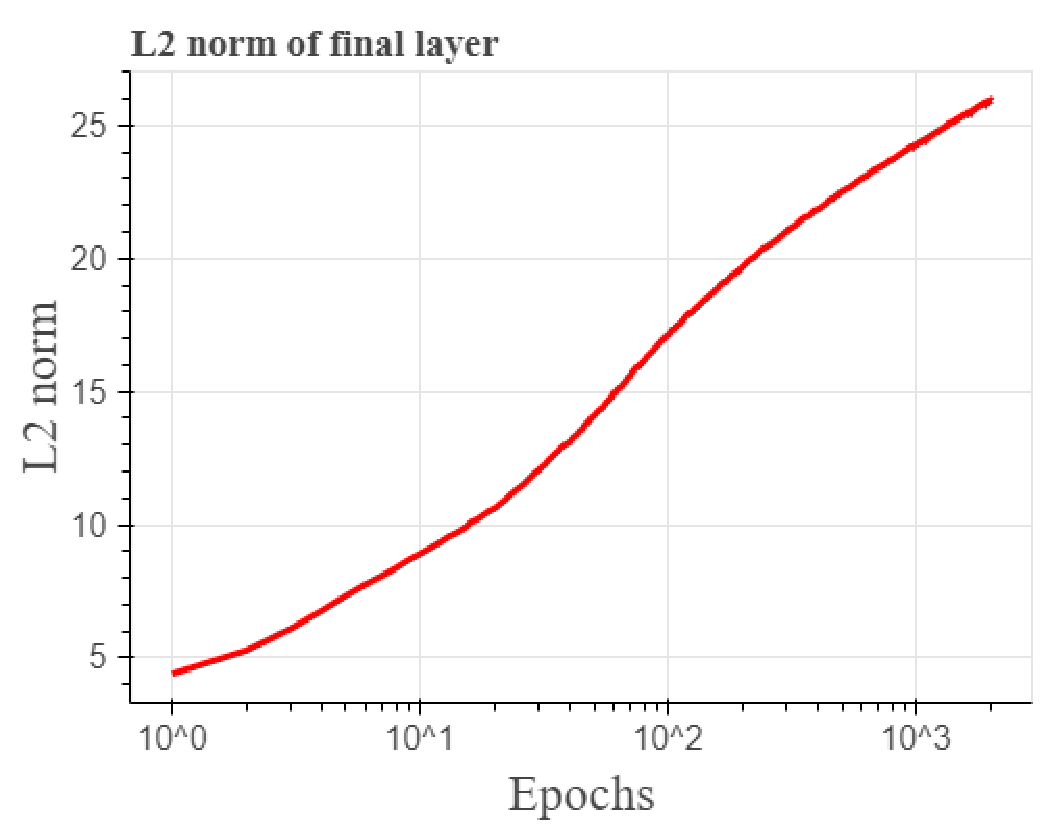}\tabularnewline
\end{tabular}
\par\end{centering}
\caption{Training of a convolutional neural network on CIFAR10 using stochastic
gradient descent with constant learning rate and momentum, softmax
output and a cross entropy loss, where we achieve $8.3\%$ final validation
error. We observe that, approximately: (1) The training loss decays
as a $t^{-1}$, (2) the $L_{2}$ norm of last weight layer increases
logarithmically, (3) after a while, the validation loss starts to
increase, and (4) in contrast, the validation (classification) error
slowly improves. \label{fig: DNN results}}
\end{figure}

So far we have only considered linear prediction. Naturally, it is
desirable to generalize our results also to non-linear models and
especially multi-layer neural networks.

Even without a formal extension and description of the precise bias,
our results already shed light on how minimizing the cross-entropy
loss with gradient descent can have a margin maximizing effect, how
the margin might improve only logarithmically slow, and why it might
continue to improve even as the validation loss increases. These effects
are demonstrated in Figure \ref{fig: DNN results} and Table \ref{tab:Sample-value-dnn}
which portray typical training of a convolutional neural network using
unregularized gradient descent\footnote{Code available here: \url{ https://github.com/paper-submissions/MaxMargin}}.
As can be seen, the norm of the weight increases, but the validation
error continues decreasing, albeit very slowly (as predicted by the
theory), even after the training error is zero and the training loss
is extremely small. We can now understand how even though the loss
is already extremely small, some sort of margin might be gradually
improving as we continue optimizing. We can also observe how the validation
loss increases despite the validation error decreasing, as discussed
in Section \ref{sec: convergence rates}.

As an initial advance toward tackling deep network, we can point out
that for several special cases, our results may be directly applied to
multi-layered networks. First, somewhat trivially, our results
may be applied directly to the last weight layer of a neural network if the last
hidden layer becomes fixed and linearly separable after a certain
number of iterations. This can become true, either approximately,
if the input to the last hidden layer is normalized (\emph{e.g.},
using batch norm), or exactly, if the last hidden layer is quantized
\citep{Hubara2016}. 

\begin{table}
\begin{centering}
\begin{tabular}{|c|c|c|c|c|c|c|}
\hline 
Epoch  & 50  & 100  & 200  & 400  & 2000  & 4000\tabularnewline
\hline 
\hline 
$L_{2}$ norm  & 13.6  & 16.5  & 19.6  & 20.3  & 25.9  & 27.54\tabularnewline
\hline 
Train loss  & 0.1  & 0.03  & 0.02  & 0.002  & $10^{-4}$  & $3\cdot10^{-5}$\tabularnewline
\hline 
Train error  & 4\%  & 1.2\%  & 0.6\%  & 0.07\%  & 0\%  & 0\%\tabularnewline
\hline 
Validation loss  & 0.52  & 0.55  & 0.77  & 0.77  & 1.01  & 1.18\tabularnewline
\hline 
Validation error  & 12.4\%  & 10.4\%  & 11.1\%  & 9.1\%  & 8.92\%  & 8.9\% \tabularnewline
\hline 
\end{tabular}
\par\end{centering}
\caption{Sample values from various epochs in the experiment depicted in Fig.
\ref{fig: DNN results}. \label{tab:Sample-value-dnn}}
\end{table}

Second, as we show next, our results may be applied exactly on deep networks if only a single weight layer is being optimized, and, furthermore, after a sufficient number of iterations, the activation units stop switching and the training error goes to zero. 

\begin{corR}
We examine a multilayer neural network with component-wise ReLU functions
$f\left(z\right)=\max\left[z,0\right]$, and weights $\left\{ \mathbf{W}_{l}\right\} _{l=1}^{L}$.
Given input $\mathbf{x}_{n}$ and target $y_{n}\in\left\{ -1,1\right\} $,
the DNN produces a scalar output 
\[
u_{n}=\mathbf{W}_{L}f\left(\mathbf{W}_{L-1}f\left(\cdots\mathbf{W}_{2}f\left(\mathbf{W}_{1}\mathbf{x}_{n}\right)\right)\right)
\]
 and has loss $\ell\left(y_{n}u_{n}\right)$, where $\ell$ obeys
assumptions \ref{assum: loss properties} and \ref{assum: exponential tail}. 

If we optimize a single weight layer $\mathbf{w}_{l}=\mathrm{vec}\left(\mathbf{W}_{l}^{\top}\right)$ using gradient
descent, so that $\mathcal{L}\left(\mathbf{w}_{l}\right)=\mathcal{\sum}_{n=1}^{N}\ell\left(y_{n}u_{n}(\mathbf{w}_{l})\right)$
converges to zero, and $\exists t_{0}$ such that $\forall t>t_{0}$
the ReLU inputs do not switch signs, then $\mathbf{w}_{l}(t)/\norm{\mathbf{w}_{l}(t)}$
converges to 
\[
\underset{\mathbf{w}_{l}}{\hat{\mathbf{w}}_{l}=\mathrm{argmin}}\left\lVert \mathbf{w}_{l}\right\rVert ^{2}\,\,\mathrm{s.t.}\,\,y_{n}u_{n}(\mathbf{w}_{l})\geq1.
\]
\end{corR}

\begin{proof}
We examine the output of the network given a single input \textbf{$\mathbf{x}_n$},
for\textbf{ $t>t_{0}$. }Since the ReLU inputs do not switch signs,
we can write $\mathbf{v}_{l}$, the output of layer $l$, as
\[
\mathbf{v}_{l,n}=\prod_{m=1}^{l}\mathbf{A}_{m,n}\mathbf{W}_{m}\mathbf{x}_n\,,
\]
where we defined $\mathbf{A}_{l,n}$ for $l<L$ as a diagonal 0-1 matrix, which diagonal is the ReLU slopes at layer $l$, sample $n$, and $\mathbf{A}_{L,n}=1$. Additionally,
we define
\[
\boldsymbol{\delta}_{l,n}=\mathbf{A}_{l,n}\prod_{m=L}^{l+1}\mathbf{W}_{m}^{\top}\mathbf{A}_{m,n}\,;\,\tilde{\mathbf{x}}_{l,n}=\boldsymbol{\delta}_{l,n}\otimes\mathbf{u}_{l-1,n}\,.
\]
Using this notation we can write 
\begin{equation} \label{eq: DNN_output}
u_n(\mathbf{w}_{l})=v_{L,n}=\prod_{m=1}^{L}\mathbf{A}_{m,n}\mathbf{W}_{m}\mathbf{x}_n=\boldsymbol{\delta}_{l,n}^{\top}\mathbf{W}_{l}\mathbf{u}_{l-1,n}=\tilde{\mathbf{x}}_{l,n}^{\top}\mathbf{w}_{l}\,.
\end{equation}
This implies that
\[
\mathcal{L}(\wvec_l)=\sumn \ell\left(y_n u_n(\mathbf{w}_{l})\right)=\sumn \ell\left(y_n\tilde{\mathbf{x}}_{l,n}^{\top}\mathbf{w}_{l}\right),
\]
which is the same as the original linear problem. Since the loss converges to zero, the dataset $\{\tilde{\mathbf{x}}_{l,n},y_n\}_{n=1}^N$ must be linearly separable. Applying Theorem \ref{thm: main theorem}, and recalling that $u(\mathbf{w}_{l})=\tilde{\mathbf{x}}_{l}^{\top}\mathbf{w}_{l}$ from eq. \ref{eq: DNN_output}, we prove this corollary.
\end{proof}

Importantly, this case is non-convex, unless we are optimizing the last layer. Note we assumed ReLU functions for simplicity, but this proof can be easily generalized for any other piecewise linear constant activation functions (\emph{e.g.},
leaky ReLU, max-pooling).

Lastly, in a follow-up work \citep{Gunasekar2018b}, given a few additional assumptions, extended our results to linear predictors which can be written as a homogeneous polynomial in the parameters. These results seem to indicate that, in many cases, GD operating on exp-tailed loss with positively homogeneous predictors aims to a specific direction. This is the direction of the max margin predictor minimizing the $L_2$ norm in the parameter space. It is not yet clear how to generally translate such an implicit bias in the parameter space to the implicit bias in the predictor space \textemdash{} except in special cases, such as deep linear neural nets, as we have shown in \citep{Gunasekar2018b}. Moreover, in non-linear neural nets, there are many equivalent max-margin solutions which minimize the $L_2$ norm of the parameters. Therefore, it is natural to expect that GD would have additional implicit biases, which select a specific subset of these solutions.

\subsection{Other optimization methods \label{sec:other_opt}}

\begin{figure}
\begin{centering}
\includegraphics[width=1\columnwidth]{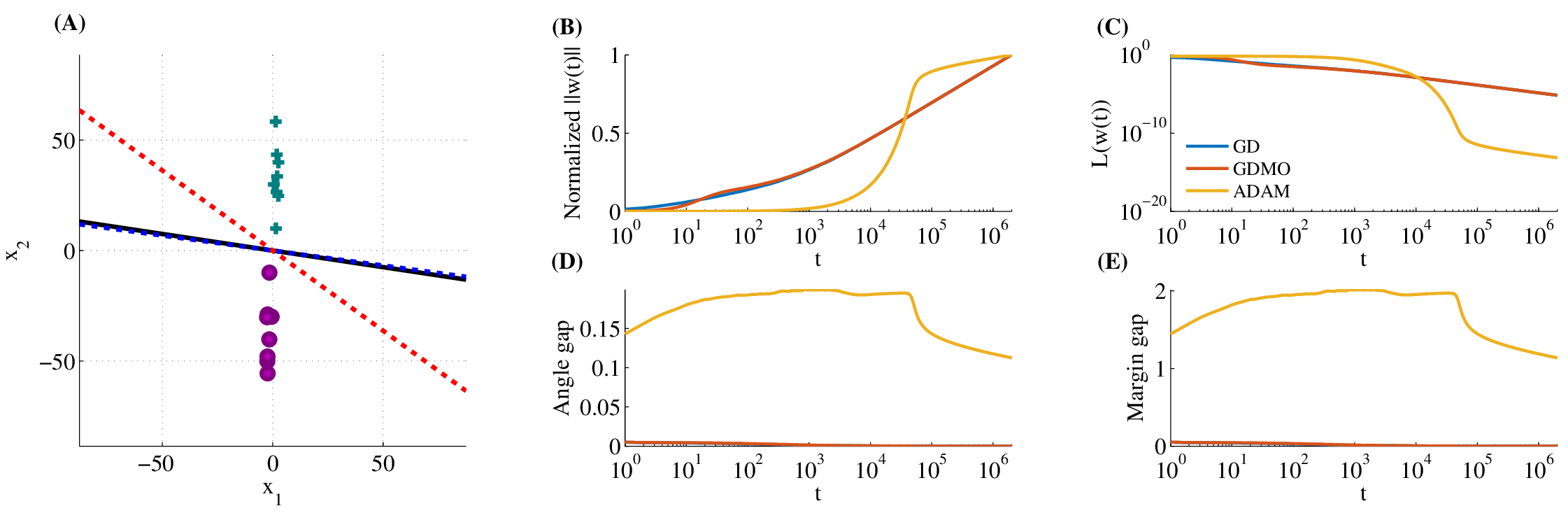} 
\par\end{centering}
\caption{Same as Fig. \ref{fig:Synthetic-dataset}, except we multiplied all
$x_{2}$ values in the dastaset by $20$, and also train using ADAM.
The final weight vector produced after $2\cdot10^{6}$ epochs of optimization
using ADAM (red dashed line) does not converge to L2 max margin solution
(black line), in contrast to GD (blue dashed line), or GDMO.\label{fig:Synthetic-dataset-adam}}
\end{figure}

In this paper we examined the implicit bias of gradient descent. Different
optimization algorithms exhibit different biases, and understanding
these biases and how they differ is crucial to understanding and constructing
learning methods attuned to the inductive biases we expect. Can we
characterize the implicit bias and convergence rate in other optimization
methods?

In Figure \ref{fig:Synthetic-dataset} we see that adding momentum
does not qualitatively affect the bias induced by gradient descent.
In Figure \ref{fig: SGD} in Appendix \ref{sec:Additional-Figures}
we also repeat the experiment using stochastic gradient descent, and
observe a similar asymptotic bias (this was later proved in \citet{Nacson2018b}). This is consistent with the fact that momentum,
acceleration and stochasticity do not change the bias when using gradient
descent to optimize an under determined least squares problem. It
would be beneficial, though, to rigorously understand how much we
can generalize our result to gradient descent variants, and how the
convergence rates might change in these cases.

On the other hand, as an example of how changing the optimization algorithm does change the bias, consider adaptive methods, such as AdaGrad \citep{duchi2011adaptive}
and ADAM \citep{Kingma2015}.
In Figure \ref{fig:Synthetic-dataset-adam} we show the predictors
obtained by ADAM and by gradient descent on a simple data set. Both
methods converge to zero training error solutions. But although gradient
descent converges to the $L_{2}$ max margin predictor, as predicted
by our theory, ADAM does not. The implicit bias of adaptive methods
has in fact been a recent topic of interest, with \citet{Hoffer2017a} and
\citet{Wilson2017} suggesting they lead to worse generalization, and \cite{Wilson2017} providing examples of the differences in the bias for linear regression problems with the squared loss.  Can we characterize the bias of adaptive methods for logistic regression problems?  Can we characterize the bias of other optimization methods, providing a general understanding linking optimization algorithms with their biases?

In a follow-up paper \citep{gunasekar2018characterizing}
provided initial answers to these questions. \citet{gunasekar2018characterizing} derived a precise characterization of the limit direction of steepest descent for general norms when optimizing the exp-loss, and show that for adaptive methods such as Adagrad the limit direction can depend on the initial point and step size and is thus not as predictable and robust as with non-adaptive methods.

\subsection{Other loss functions}
In this work we focused on loss functions with exponential tail and observed a very slow, logarithmic convergence of the normalized weight vector to the $L_2$ max margin direction. A natural question that follows is how does this behavior change with types of loss function tails. Specifically, does the normalized weight vector always converge to the $L_2$ max margin solution? How is the convergence rate affected? Can we improve the convergence rate beyond the logarithmic rate found in this work?

In a follow-up work \cite{Nacson2018} provided partial answers to these questions. They proved that the exponential tail has the optimal convergence rate, for tails for which $\ell^{\prime}(u)$ is of the form $\exp(-u^{\nu})$ with $\nu>0.25$. They then conjectured, based on heuristic analysis, that the exponential tail is optimal among all possible tails. Furthermore, they demonstrated that polynomial or heavier tails do not converge to the max margin solution. Lastly, for the exponential loss they proposed a normalized gradient scheme which can significantly improve convergence rate, achieving $O(\log(t)/\sqrt{t})$.

\subsection{Matrix Factorization}

With multi-layered neural networks in mind, \citet{Gunasekar2017}
recently embarked on a study of the implicit bias of under-determined
matrix factorization problems, where the {\em squared
loss} of the linear observation of a matrix is minimized by gradient descent on its
factorization. Since a matrix factorization can be viewed as a two-layer
network with linear activations, this is perhaps the simplest deep
model one can study in full, and can thus provide insight and direction
to studying more complex neural networks. \citeauthor{Gunasekar2017}
conjectured, and provided theoretical and empirical evidence, that
gradient descent on the factorization for an under-determined problem
converges to the minimum nuclear norm solution, but only if the initialization
is infinitesimally close to zero and the step-sizes are infinitesimally
small. With finite step-sizes or finite initialization, \citeauthor{Gunasekar2017}
could not characterize the bias. 

The follow-up paper \citep{gunasekar2018characterizing} studied this same problem with exponential loss instead of squared loss. Under additional assumptions on the asymptotic convergence of update directions and gradient directions, they were able to relate the direction of gradient descent iterates on the factorized parameterization asymptotically to the  maximum margin solution with unit nuclear norm. Unlike the case of squared loss, the result for exponential loss are independent of initialization and with only mild  conditions on the step size. Here again, we see the asymptotic nature of exponential loss on separable data nullifying the initialization effects thereby making the analysis simpler compared to squared loss. 
\remove{
Beyond
the practical relevance of the logistic loss, taking our approach
has the advantage that because of its asymptotic nature, it does not
depend on the initialization and step-size. It thus might prove easier
to analyze logistic regression on a matrix factorization instead of
the least square problem, providing significant insight into the implicit
biases of gradient descent on non-convex multi-layered optimization. \dnote{maybe add forward citation to ICML paper here?}
}

\section{Summary}

We characterized the implicit bias induced by gradient descent on homogeneous linear predictors when 
minimizing smooth monotone loss functions with an exponential tail.
This is the type of loss commonly being minimized in deep learning.
We can now rigorously understand: 
\begin{enumerate}
\item How gradient descent, without early stopping, induces implicit $L_{2}$
regularization and converges to the maximum $L_{2}$ margin solution,
when minimizing for binary classification with logistic loss, exp-loss, or other exponential tailed monotone
decreasing loss, as well as for multi-class classification with cross-entropy loss. Notably, even though the logistic loss and the exp-loss behave very different on non-separable problems, they exhibit the same behaviour for separable problems. This implies that the non-tail
part does not affect the bias. The bias is also independent of the step-size
used (as long as it is small enough to ensure convergence) and  is also independent on the initialization (unlike for least square problems). 
\item The convergence of the direction of gradient descent updates to the maximum $L_2$ margin solution, however is  very slow compared to the convergence of training loss, which explains why it is worthwhile
continuing to optimize long after we have zero training error, and
even when the loss itself is already extremely small. 
\item We should not rely on plateauing  of the training loss or on the loss (logistic or exp or cross-entropy)  evaluated on a validation data, as measures to decide when to stop. Instead, we should look at the $0$--$1$ error on the validation dataset. We might improve the validation and test errors even when  when the decrease in the training loss is tiny and even when the validation loss itself increases. 
\end{enumerate}
Perhaps that gradient descent leads to a max $L_{2}$ margin solution
is not a big surprise to those for whom the connection between $L_{2}$
regularization and gradient descent is natural. Nevertheless, we are
not familiar with any prior study or mention of this fact, let alone
a rigorous analysis and study of how this bias is exact and independent
of the initial point and the step-size. Furthermore, we also analyze
the rate at which this happens, leading to the novel observations
discussed above. Even more importantly, we hope that our analysis
can open the door to further analysis of different optimization methods
or in different models, including deep networks, where implicit regularization
is not well understood even for least square problems, or where we
do not have such a natural guess as for gradient descent on linear
problems. Analyzing gradient descent on logistic/cross-entropy loss
is not only arguably more relevant than the least square loss, but
might also be technically easier.

\section*{Acknowledgments}

The authors are grateful to J. Lee, and C. Zeno for
helpful comments on the manuscript. The research of DS was supported
by the Israel Science Foundation (grant No. 31/1031), by the Taub foundation and of NS by the National Science Foundation.

\appendix
%dummy comment inserted by tex2lyx to ensure that this paragraph is not empty%dummy comment inserted by tex2lyx to ensure that this paragraph is not empty

\part*{\newpage{}Appendix}

\section{Proof of Theorems \ref{thm: main theorem} and \ref{thm: refined Theorem} for almost every dataset\label{sec:proof}}

In the following sub-sections we first prove Theorem \ref{thm: main theorem almost everywhere} below, which is a version of Theorem \ref{thm: main theorem}, specialized for almost every dataset. We then prove Theorem \ref{thm: refined Theorem} (which is already stated for almost every dataset).

\begin{restatable}{thmR}{LRasymptotic_ae}

\label{thm: main theorem almost everywhere} For almost every dataset which is linearly separable (Assumption
\ref{assum: Linear sepereability}), any $\beta$-smooth decreasing
loss function (Assumption \ref{assum: loss properties}) with an exponential
tail (Assumption \ref{assum: exponential tail}), any stepsize $\eta<2\beta^{-1}\sigma_{\max}^{-2}\left(\text{\ensuremath{\mathbf{X}} }\right)$
and any starting point $\w(0)$, the gradient descent iterates (as in eq.~\ref{eq: gradient descent linear}) will behave as: 
\begin{equation}
\mathbf{w}\left(t\right)=\hat{\mathbf{w}}\log t+\boldsymbol{\rho}\left(t\right)\,,
\end{equation}
where $\hat{\mathbf{w}}$ is the $L_{2}$ max margin vector 
\[
\hat{\mathbf{w}}=\underset{\mathbf{\mathbf{w}}\in\mathbb{R}^{d}}{\mathrm{argmin}}\left\lVert \mathbf{w}\right\rVert ^{2}\,\,\mathrm{s.t.}\,\,\forall n:\,\mathbf{w}^{\top}\mathbf{x}_{n}\geq1,
\]
the residual $\rho(t)$ is bounded, and so 
\[
\lim_{t\rightarrow\infty}\frac{\mathbf{w}\left(t\right)}{\left\Vert \mathbf{w}\left(t\right)\right\Vert }=\frac{\hat{\mathbf{w}}}{\left\Vert \hat{\mathbf{w}}\right\Vert }.
\]
\end{restatable}

In the following proofs, for any solution $\mathbf{w}\left(t\right)$,
we define 
\[
\mathbf{r}\left(t\right)=\mathbf{w}\left(t\right)-\hat{\mathbf{w}}\log t-\tilde{\mathbf{w}},
\]
where $\hat{\mathbf{w}}$ and $\tilde{\mathbf{w}}$ follow the conditions
of Theorems \ref{thm: main theorem} and \ref{thm: refined Theorem},
\emph{i.e.} $\hat{\mathbf{w}}$ is the $L_{2}$ is the max margin vector defined above, and $\tilde{\mathbf{w}}$ is a vector which satisfies eq. \ref{eq: w tilde}:
\begin{equation}
\forall n\in\set:\,\eta\exp\left(-\mathbf{x}_{n}^{\top}\tilde{\mathbf{w}}\right)=\alpha_{n}\,,\label{eq:w_tilde2}
\end{equation}
where we recall that we denoted $\mathbf{X}_{\mathcal{\set}}\in\mathbb{R}^{d\times\left|\set\right|}$
as the matrix whose columns are the support vectors, a subset $\set\subset\left\{ 1,\dots,N\right\} $
of the columns of $\mathbf{X}=\left[\mathbf{x}_{1},\dots,\mathbf{x}_{N}\right]\in\mathbb{R}^{d\times N}$.

In Lemma \ref{lem: alpha} (Appendix \ref{sec:alpha}) we prove that
for almost every dataset $\boldsymbol{\alpha}$ is uniquely defined,
there are no more then $d$ support vectors and $\alpha_{n}\neq0$,
$\forall n\in\set$. Therefore, eq. \ref{eq:w_tilde2} is well-defined
in those cases. If the support vectors do not span the data, then
the solution $\tilde{\mathbf{w}}$ to eq. \ref{eq:w_tilde2} might
not be unique. In this case, we can use any such solution in the proof.

We furthermore denote the minimum margin to a non-support vector as:
\begin{equation}
\theta=\min_{n\notin\set}\mathbf{x}_{n}^{\top}\hat{\mathbf{w}}>1\,,\label{eq: v1 SVM}
\end{equation}
and by $C_{i}$,$\epsilon_{i}$,$t_{i}$ (\textbf{$i\in\mathbb{N}$})
various positive constants which are independent of $t$. Lastly,
we define $\mathbf{P}_{1}\in\mathbb{R}^{d\times d}$ as the orthogonal
projection matrix\footnote{This matrix can be written as $\mathbf{P}_{1}=\mathbf{X}_{\set}\mathbf{X}_{\set}^{+}$,
where $\mathbf{M}^{\pmpi}$ is the Moore-Penrose pseudoinverse of
$\mathbf{M}$. } to the subspace spanned by the support vectors (the columns of $\mathbf{X}_{\set}$),
and $\bar{\mathbf{P}}_{1}=\mathbf{I}-\mathbf{P}_{1}$ as the complementary
projection (to the left nullspace of $\mathbf{X}_{\set}$).

\subsection{Simple proof of Theorem \ref{thm: main theorem almost everywhere}}% in the special case $\eta\to0$ and $\ell(u)=\exp(-u)$}

In this section we first examine the special case that $\ell\left(u\right)=e^{-u}$
and take the continuous time limit of gradient descent: $\eta\rightarrow0$
, so 
\[
\dot{\mathbf{w}}\left(t\right)=-\nabla\mathcal{L}\left(\mathbf{w}\left(t\right)\right)\,.
\]
The proof in this case is rather short and self-contained (\emph{i.e.}, does
not rely on any previous results), and so it helps to clarify the
main ideas of the general (more complicated) proof which we will give
in the next sections.

Recall we defined 
\begin{equation}
\mathbf{r}\left(t\right)=\mathbf{w}\left(t\right)-\log\left(t\right)\hat{\mathbf{w}}-\tilde{\mathbf{w}}\,.\label{eq: r definition}
\end{equation}
Our goal is to show that $\left\Vert \mathbf{r}\left(t\right)\right\Vert $
is bounded, and therefore $\boldsymbol{\rho}\left(t\right)=\mathbf{r}\left(t\right)+\tilde{\mathbf{w}}$
is bounded. Eq. \ref{eq: r definition} implies that 
\begin{equation}
\dot{\mathbf{r}}\left(t\right)=\dot{\mathbf{w}}\left(t\right)-\frac{1}{t}\hat{\mathbf{w}}=-\nabla\mathcal{L}\left(\mathbf{w}\left(t\right)\right)-\frac{1}{t}\hat{\mathbf{w}}\label{eq: r dot}
\end{equation}
and therefore 
\begin{align}
 & \frac{1}{2}\frac{d}{dt}\left\Vert \mathbf{r}\left(t\right)\right\Vert ^{2}=\dot{\mathbf{r}}^{\top}\left(t\right)\mathbf{r}\left(t\right)\nonumber \\
 & =\sum_{n=1}^{N}\exp\left(-\mathbf{x}_{n}^{\top}\mathbf{w}\left(t\right)\right)\mathbf{x}_{n}^{\top}\mathbf{r}\left(t\right)-\frac{1}{t}\hat{\mathbf{w}}^{\top}\mathbf{r}\left(t\right)\nonumber \\
 & =\left[\sum_{n\in\set}\exp\left(-\log\left(t\right)\hat{\mathbf{w}}^{\top}\mathbf{x}_{n}-\tilde{\mathbf{w}}^{\top}\mathbf{x}_{n}-\mathbf{x}_{n}^{\top}\mathbf{r}\left(t\right)\right)\mathbf{x}_{n}^{\top}\mathbf{r}\left(t\right)-\frac{1}{t}\hat{\mathbf{w}}^{\top}\mathbf{r}\left(t\right)\right]\nonumber \\
 & +\left[\sum_{n\not\notin\set}\exp\left(-\log\left(t\right)\hat{\mathbf{w}}^{\top}\mathbf{x}_{n}-\tilde{\mathbf{w}}^{\top}\mathbf{x}_{n}-\mathbf{x}_{n}^{\top}\mathbf{r}\left(t\right)\right)\mathbf{x}_{n}^{\top}\mathbf{r}\left(t\right)\right]\text{,}\label{eq: dr^2/dt}
\end{align}
where in the last equality we used eq. \ref{eq: r definition} and
decomposed the sum over support vectors $\set$ and non-support vectors.
We examine both bracketed terms.
Recall that $\hat{\mathbf{w}}^{\top}\mathbf{x}_{n}=1$ for $n\in\set$,
and that we defined (in eq. \ref{eq:w_tilde2}) \textbf{$\tilde{\mathbf{w}}$
}so that $\sum_{n\in\set}\exp\left(-\tilde{\mathbf{w}}^{\top}\mathbf{x}_{n}\right)\mathbf{x}_{n}=\hat{\mathbf{w}}$.
Thus, the first bracketed term in eq. \ref{eq: dr^2/dt} can be written
as 
\begin{align}
 & \frac{1}{t}\sum_{n\in\set}\exp\left(-\tilde{\mathbf{w}}^{\top}\mathbf{x}_{n}-\mathbf{x}_{n}^{\top}\mathbf{r}\left(t\right)\right)\mathbf{x}_{n}^{\top}\mathbf{r}\left(t\right)-\frac{1}{t}\sum_{n\in\set}\exp\left(-\tilde{\mathbf{w}}^{\top}\mathbf{x}_{n}\right)\mathbf{x}_{n}^\top\mathbf{r}\left(t\right)\nonumber \\
= & \frac{1}{t}\sum_{n\in\set}\exp\left(-\tilde{\mathbf{w}}^{\top}\mathbf{x}_{n}\right)\left(\exp\left(-\mathbf{x}_{n}^{\top}\mathbf{r}\left(t\right)\right)-1\right)\mathbf{x}_{n}^{\top}\mathbf{r}\left(t\right)\leq0,\label{eq: dr/dt S}
\end{align}
since $\forall z,\;z\left(e^{-z}-1\right)\leq0$. Furthermore, since $\forall z\;e^{-z}z\leq1$
and $\theta=\mathrm{argmin}_{n\notin\set}\mathbf{x}_{n}^{\top}\hat{\mathbf{w}}>1$
(eq. \ref{eq: v1 SVM}), the second bracketed term in eq. \ref{eq: dr^2/dt}
can be upper bounded by 
\begin{align}
\sum_{n\not\notin\set}\exp\left(-\log\left(t\right)\hat{\mathbf{w}}^{\top}\mathbf{x}_{n}-\tilde{\mathbf{w}}^{\top}\mathbf{x}_{n}\right)\exp\left(-\mathbf{x}_{n}^\top \mathbf{r}(t)\right)\mathbf{x}_{n}^\top \mathbf{r}(t) & \leq\frac{1}{t^{\theta}}\sum_{n\not\notin\set}\exp\left(-\tilde{\mathbf{w}}^{\top}\mathbf{x}_{n}\right)\,.\label{eq: dr/dt non-S}
\end{align}
Substituting eq. \ref{eq: dr/dt S} and \ref{eq: dr/dt non-S} into
eq. \ref{eq: dr^2/dt} and integrating, we obtain, that $\exists C,C^{\prime}$
such that %\snote{why not just define based n $t_1=0$} \dnote{I wanted to avoid terms like $\log(t)$ to explode. Also, to be consistent with proof below.}
\[
\forall t_{1},\forall t>t_{1}:\left\Vert \mathbf{r}\left(t\right)\right\Vert ^{2}-||\mathbf{r}(t_{1})||^{2}\leq C\int_{t_{1}}^{t}\frac{dt}{t^{\theta}}\leq C^{\prime}<\infty\,,
\]

since $\theta>1$ (eq. \ref{eq: v1 SVM}). Thus, we showed that $\mathbf{r}(t)$
is bounded, which completes the proof for the special case. $\blacksquare$

\subsection{Complete proof of Theorem \ref{thm: main theorem almost everywhere} \label{sec:Proof-of-Theorem}}

Next, we give the proof for the general case (non-infinitesimal step size, and exponentially-tailed
functions). Though it is based on a similar analysis as in the special
case we examined in the previous section, it is somewhat more involved
since we have to bound additional terms.

First, we state two auxiliary lemmata, that are proven below in appendix
sections \ref{sec:Proof of GD convergence} and \ref{sec:Proof-of-Lemma correlation}:

\begin{restatable}{lemR}{GDconvergence}

\label{lem: GD convergence} Let $\mathcal{L}\left(\mathbf{w}\right)$
be a $\beta$-smooth non-negative objective. If $\eta<2\beta^{-1}$,
then, for any $\w(0)$, with the GD sequence 
\begin{equation}
\mathbf{w}\left(t+1\right)=\mathbf{w}\left(t\right)-\eta\nabla\mathcal{L}\left(\mathbf{w}(t)\right)\,\label{eq: gradient descent}
\end{equation}
we have that $\sum_{u=0}^{\infty}\left\Vert \nabla\mathcal{L}\left(\mathbf{w}\left(u\right)\right)\right\Vert ^{2}<\infty$
and therefore $\lim_{t\rightarrow\infty}\left\Vert \nabla\mathcal{L}\left(\mathbf{w}\left(t\right)\right)\right\Vert ^{2}=0.$

\end{restatable}

\begin{restatable}{lemR}{correlation}

\label{lem: correlation bound} We have 
\begin{equation}
\exists C_{1},t_{1}:\,\forall t>t_{1}:\,\left(\mathbf{r}\left(t+1\right)-\mathbf{r}\left(t\right)\right)^{\top}\mathbf{r}\left(t\right)\leq C_{1}t^{-\min\left(\theta,1+1.5\mu_{+},1+0.5\mu_{-}\right)}\,.\label{eq: general case}
\end{equation}
Additionally, $\forall\epsilon_{1}>0\,$, $\exists C_{2},t_{2}$,
such that $\forall t>t_{2}$, if 
\begin{equation}
\left\Vert \mathbf{P}_{1}\mathbf{r}\left(t\right)\right\Vert \geq\epsilon_{1},\label{eq: bounded r(t)}
\end{equation}
then the following improved bound holds 
\begin{equation}
\left(\mathbf{r}\left(t+1\right)-\mathbf{r}\left(t\right)\right)^{\top}\mathbf{r}\left(t\right)\leq-C_{2}t^{-1}<0\,.\label{eq: bounded cases}
\end{equation}

\end{restatable}

Our goal is to show that $\left\Vert \mathbf{r}\left(t\right)\right\Vert $
is bounded, and therefore $\boldsymbol{\rho}\left(t\right)=\mathbf{r}\left(t\right)+\tilde{\mathbf{w}}$
is bounded. To show this, we will upper bound the following equation
\begin{align}
\left\Vert \mathbf{r}\left(t+1\right)\right\Vert ^{2} & =\left\Vert \mathbf{r}\left(t+1\right)-\mathbf{r}\left(t\right)\right\Vert ^{2}+2\left(\mathbf{r}\left(t+1\right)-\mathbf{r}\left(t\right)\right)^{\top}\mathbf{r}\left(t\right)+\left\Vert \mathbf{r}\left(t\right)\right\Vert ^{2}\label{eq: r recursion}
\end{align}
First, we note that first term in this equation can be upper-bounded
by 
\begin{align}
 & \left\Vert \mathbf{r}\left(t+1\right)-\mathbf{r}\left(t\right)\right\Vert ^{2}\nonumber \\
 & \overset{\left(1\right)}{=}\left\Vert \mathbf{w}\left(t+1\right)-\hat{\mathbf{w}}\log\left(t+1\right)-\tilde{\mathbf{w}}-\mathbf{w}\left(t\right)+\hat{\mathbf{w}}\log\left(t\right)+\tilde{\mathbf{w}}\right\Vert ^{2}\nonumber \\
 & \overset{\left(2\right)}{=}\left\Vert -\eta\nabla\mathcal{L}\left(\mathbf{w}\left(t\right)\right)-\hat{\mathbf{w}}\left[\log\left(t+1\right)-\log\left(t\right)\right]\right\Vert ^{2}\nonumber \\
 & =\eta^{2}\left\Vert \nabla\mathcal{L}\left(\mathbf{w}\left(t\right)\right)\right\Vert ^{2}+\left\Vert \hat{\mathbf{w}}\right\Vert ^{2}\log^{2}\left(1+t^{-1}\right)+2\eta\hat{\mathbf{w}}^{\top}\nabla\mathcal{L}\left(\mathbf{w}\left(t\right)\right)\log\left(1+t^{-1}\right)\nonumber \\
 & \overset{\left(3\right)}{\leq}\eta^{2}\left\Vert \nabla\mathcal{L}\left(\mathbf{w}\left(t\right)\right)\right\Vert ^{2}+\left\Vert \hat{\mathbf{w}}\right\Vert ^{2}t^{-2}\label{eq: square r difference}
\end{align}
where in $\left(1\right)$ we used eq. \ref{eq: r definition},
in $\left(2\right)$ we used eq. \ref{eq: gradient descent linear},
and in $\left(3\right)$ we used $\forall x>0:\,x\geq\log\left(1+x\right)>0$,
and also that 
\begin{equation}
\hat{\mathbf{w}}^{\top}\nabla\mathcal{L}\left(\mathbf{w}\left(t\right)\right)=\sum_{n=1}^{N}\ell^{\prime}\left(\mathbf{w}\left(t\right)^{\top}\mathbf{x}_{n}\right)\hat{\mathbf{w}}^{\top}\mathbf{x}_{n}\leq0\,,
\end{equation}
since $\hat{\mathbf{w}}^{\top}\mathbf{x}_{n}\geq1$ (from the definition
of $\what$) and $\ell^{\prime}(u)\leq0$.

Also, from Lemma \ref{lem: GD convergence} we know that 
\begin{equation}
\left\Vert \nabla\mathcal{L}\left(\mathbf{w}\left(t\right)\right)\right\Vert ^{2}=o\left(1\right)\,\mathrm{and}\,\sum_{t=0}^{\infty}\left\Vert \nabla\mathcal{L}\left(\mathbf{w}\left(t\right)\right)\right\Vert ^{2}<\infty\,.\label{eq: norm grad squared 1/t}
\end{equation}
Substituting eq. \ref{eq: norm grad squared 1/t} into eq. \ref{eq: square r difference},
and recalling that a $t^{-\nu}$ power series converges for any $\nu>1$,
we can find $C_{0}$ such that 
\begin{equation}
\left\Vert \mathbf{r}\left(t+1\right)-\mathbf{r}\left(t\right)\right\Vert ^{2}=o\left(1\right)\,\mathrm{and}\,\sum_{t=0}^{\infty}\left\Vert \mathbf{r}\left(t+1\right)-\mathbf{r}\left(t\right)\right\Vert ^{2}=C_{0}<\infty\,.\label{eq: square norm of r difference-1}
\end{equation}
Note that this equation also implies that $\forall\epsilon_{0}$ 
\begin{equation}
\exists t_{0}:\forall t>t_{0}:\left|\left\Vert \mathbf{r}\left(t+1\right)\right\Vert -\left\Vert \mathbf{r}\left(t\right)\right\Vert \right|<\epsilon_{0}\,.\label{eq: norm difference convergence}
\end{equation}

Next, we would like to bound the second term in eq. \ref{eq: r recursion}.
From eq. \ref{eq: general case} in Lemma \ref{lem: correlation bound},
we can find $t_{1},C_{1}$ such that $\forall t>t_{1}$: 
\begin{equation}
\left(\mathbf{r}\left(t+1\right)-\mathbf{r}\left(t\right)\right)^{\top}\mathbf{r}\left(t\right)\leq C_{1}t^{-\min\left(\theta,1+1.5\mu_{+},1+0.5\mu_{-}\right)}\,.\label{eq: correlation general case}
\end{equation}
Thus, by combining eqs. \ref{eq: correlation general case} and \ref{eq: square norm of r difference-1}
into eq. \ref{eq: r recursion}, we find 
\begin{align*}
\, & \left\Vert \mathbf{r}\left(t\right)\right\Vert ^{2}-\left\Vert \mathbf{r}\left(t_{1}\right)\right\Vert ^{2}\\
 & =\sum_{u=t_{1}}^{t-1}\left[\left\Vert \mathbf{r}\left(u+1\right)\right\Vert ^{2}-\left\Vert \mathbf{r}\left(u\right)\right\Vert ^{2}\right]\\
 & \leq C_{0}+2\sum_{u=t_{1}}^{t-1}C_{1}u^{-\min\left(\theta,1+1.5\mu_{+},1+0.5\mu_{-}\right)}
\end{align*}
which is a bounded, since $\theta>1$ (eq. \ref{eq: v1 SVM}) and $\mu_-,\mu_+>0$ (Definition \ref{def: exponential tail}). Therefore,
$\left\Vert \mathbf{r}\left(t\right)\right\Vert $ is bounded. $\blacksquare$

\subsection{Proof of Theorem \ref{thm: refined Theorem} \label{subsec:Proof-of-refined-Theorem}}

All that remains now is to show that $\left\Vert \mathbf{r}\left(t\right)\right\Vert \rightarrow0$
if $\mathrm{rank}\left(\mathbf{X}_{\set}\right)=\mathrm{rank}\left(\mathbf{X}\right)$,
and that $\tilde{\mathbf{w}}$ is unique given $\mathbf{w}\left(0\right)$.
To do so, this proof will continue where the proof of Theorem \ref{thm: main theorem}
stopped, using notations and equations from that proof.

Since $\mathbf{r}\left(t\right)$ has a bounded norm, its two orthogonal
components $\mathbf{r}\left(t\right)=\mathbf{P}_{1}\mathbf{r}\left(t\right)+\bar{\mathbf{P}}_{1}\mathbf{r}\left(t\right)$
also have bounded norms (recall that $\mathbf{P}_{1},\bar{\mathbf{P}}_{1}$
were defined in the beginning of appendix section \ref{sec:proof}).
From eq. \ref{eq: gradient descent linear}, $\nabla\mathcal{L}\left(\mathbf{w}\right)$
is spanned by the columns of $\mathbf{X}$. If $\mathrm{rank}\left(\mathbf{X}_{\set}\right)=\mathrm{rank}\left(\mathbf{X}\right)$,
then it is also spanned by the columns of $\mathbf{X}_{\set}$, and
so $\bar{\mathbf{P}}_{1}\nabla\mathcal{L}\left(\mathbf{w}\right)=0$.
Therefore, $\bar{\mathbf{P}}_{1}\mathbf{r}\left(t\right)$ is not
updated during GD, and remains constant. Since $\tilde{\mathbf{w}}$
in eq. \ref{eq: r definition} is also bounded, we can absorb this
constant $\bar{\mathbf{P}}_{1}\mathbf{r}\left(t\right)$ into \textbf{$\tilde{\mathbf{w}}$}
without affecting eq. \ref{eq: w tilde} (since $\forall n\in\set:\,\mathbf{x}_{n}^{\top}\bar{\mathbf{P}}_{1}\mathbf{r}\left(t\right)=0$).
Thus, without loss of generality, we can assume that $\mathbf{r}\left(t\right)=\mathbf{P}_{1}\mathbf{r}\left(t\right)$.

We define the set 
\[
\mathcal{T}=\left\{ t>\max\left[t_{2},t_{0}\right]:\ensuremath{\left\Vert \mathbf{r}\left(t\right)\right\Vert <\epsilon_{1}}\right\} \,.
\]
By contradiction, we assume that the complementary set is not finite,
\[
\bar{\mathcal{T}}=\left\{ t>\max\left[t_{2},t_{0}\right]:\ensuremath{\left\Vert \mathbf{r}\left(t\right)\right\Vert \geq\epsilon_{1}}\right\} \,.
\]
Additionally, the set $\mathcal{T}$ is not finite: if it were finite, it would
have had a finite maximal point $t_{\max}\in\mathcal{T}$, and then, combining eqs. \ref{eq: bounded cases}, \ref{eq: r recursion}, and \ref{eq: square norm of r difference-1},
we would find that $\forall t>t_{\max}$
\begin{align*}
\left\Vert \mathbf{r}\left(t\right)\right\Vert ^{2}-\left\Vert \mathbf{r}\left(t_{\max}\right)\right\Vert ^{2} & =\sum_{u=t_{\max}}^{t-1}\left[\left\Vert \mathbf{r}\left(u+1\right)\right\Vert ^{2}-\left\Vert \mathbf{r}\left(u\right)\right\Vert ^{2}\right]\leq C_{0}-2C_{2}\sum_{u=t_{\max}}^{t-1}u^{-1}\rightarrow-\infty\,,
\end{align*}
which is impossible since $\left\Vert \mathbf{r}\left(t\right)\right\Vert ^{2}\geq0$.
Furthermore, eq. \ref{eq: square norm of r difference-1}
implies that 
\[
\sum_{u=0}^{t}\left\Vert \mathbf{r}\left(u+1\right)-\mathbf{r}\left(t\right)\right\Vert ^{2}=C_{0}-h\left(t\right)
\]
 where $h\left(t\right)$ is a positive monotone function decreasing
to zero. Let $t_{3},t$ be any two points such that $t_{3}<t$,
$\left\{ t_{3}, t_{3}+1,\dots t\right\} \subset\bar{\mathcal{T}}$, and $\left(t_{3}-1\right)\in\mathcal{T}$. For all such  $t_{3}$ and $t$, we have 
\begin{align}
\left\Vert \mathbf{r}\left(t\right)\right\Vert ^{2}\nonumber & \leq\left\Vert \mathbf{r}\left(t_{3}\right)\right\Vert ^{2}+\sum_{u=t_{3}}^{t-1}\left[\left\Vert \mathbf{r}\left(u+1\right)\right\Vert ^{2}-\left\Vert \mathbf{r}\left(u\right)\right\Vert ^{2}\right]\\ 
\nonumber &= \left\Vert \mathbf{r}\left(t_{3}\right)\right\Vert ^{2}+\sum_{u=t_{3}}^{t-1}\left[\left\Vert \mathbf{r}\left(u+1\right)-\mathbf{r}\left(u\right)\right\Vert ^{2}+2\left(\mathbf{r}\left(u+1\right)-\mathbf{r}\left(u\right)\right)^{\top}\mathbf{r}\left(u\right)\right]\\
\nonumber
 & \leq\left\Vert \mathbf{r}\left(t_{3}\right)\right\Vert ^{2}+h\left(t_{3}\right)-h\left(t-1\right)-2C_{2}\sum_{u=t_{3}}^{t-1}u^{-1}\\ 
 & \leq\left\Vert \mathbf{r}\left(t_{3}\right)\right\Vert ^{2}+h\left(t_{3}\right)\,. \label{eq: r(t) bound}
\end{align}
Also, recall that $t_{3}>t_{0}$, so from eq. \ref{eq: norm difference convergence},
we have that $\left|\left\Vert \mathbf{r}\left(t_{3}\right)\right\Vert -\left\Vert \mathbf{r}\left(t_{3}-1\right)\right\Vert \right|<\epsilon_{0}$.
Since $\left\Vert \mathbf{r}\left(t_{3}-1\right)\right\Vert <\epsilon_{1}$ (from $\mathcal{T}$ definition),
we conclude that $\left\Vert \mathbf{r}\left(t_{3}\right)\right\Vert \leq\epsilon_{1}+\epsilon_{0}$.
Moreover, since $\mathcal{\bar{\mathcal{T}}}$ is an infinite set,
we can choose $t_{3}$ as large as we want. This implies that $\forall\epsilon_{2}>0$
we can find $t_{3}$ such that $\epsilon_{2}>h\left(t_{3}\right)$,
since $h\left(t\right)$ is a monotonically decreasing function. Therefore, from eq. \ref{eq: r(t) bound}, 
$\forall\epsilon_{1},\epsilon_{0},\epsilon_{2}$, $\exists t_{3}\in \bar{\mathcal{T}}$
such that 
\[
\forall t>t_{3}:\,\left\Vert \mathbf{r}\left(t\right)\right\Vert ^{2}\leq\epsilon_{1}+\epsilon_{0}+\epsilon_{2}\,.
\]
This implies that $\left\Vert \mathbf{r}\left(t\right)\right\Vert \rightarrow0$.

Lastly, we note that since $\bar{\mathbf{P}}_{1}\mathbf{r}\left(t\right)$
is not updated during GD, we have that $\bar{\mathbf{P}}_{1}\left(\tilde{\mathbf{w}}-\mathbf{w}\left(0\right)\right)=0$.
This sets $\tilde{\mathbf{w}}$ uniquely, together with eq. \ref{eq: w tilde}.
$\blacksquare$

\subsection{Proof of Lemma \ref{lem: GD convergence}\label{sec:Proof of GD convergence}}

\GDconvergence*

This proof is a slightly modified version of the proof of Theorem
2 in \citep{Ganti2015}. Recall a well-known property of $\beta$-smooth
functions: %Nati: cite something? 
\begin{equation}
\left|f\left(\mathbf{x}\right)-f\left(\mathbf{y}\right)-\nabla f\left(\mathbf{y}\right)^{\top}\left(\mathbf{x-y}\right)\right|\leq\frac{\beta}{2}\left\Vert \mathbf{x}-\mathbf{y}\right\Vert ^{2}\,.\label{eq: propery of beta smoothness}
\end{equation}
From the $\beta$-smoothness of $\mathcal{L}\left(\mathbf{w}\right)$
\begin{align*}
\mathcal{L}\left(\mathbf{w}\left(t+1\right)\right) & \leq\mathcal{L}\left(\mathbf{w}\left(t\right)\right)+\nabla\mathcal{L}\left(\mathbf{w}\left(t\right)\right)^{\top}\left(\mathbf{w}\left(t+1\right)-\mathbf{w}\left(t\right)\right)+\frac{\beta}{2}\left\Vert \mathbf{w}\left(t+1\right)-\mathbf{w}\left(t\right)\right\Vert ^{2}\\
 & =\mathcal{L}\left(\mathbf{w}\left(t\right)\right)-\eta\left\Vert \nabla\mathcal{L}\left(\mathbf{w}\left(t\right)\right)\right\Vert ^{2}+\frac{\beta\eta^{2}}{2}\left\Vert \nabla\mathcal{L}\left(\mathbf{w}\left(t\right)\right)\right\Vert ^{2}\\
 & =\mathcal{L}\left(\mathbf{w}\left(t\right)\right)-\eta\left(1-\frac{\beta\eta}{2}\right)\left\Vert \nabla\mathcal{L}\left(\mathbf{w}\left(t\right)\right)\right\Vert ^{2}
\end{align*}
Thus, we have 
\[
\frac{\mathcal{L}\left(\mathbf{w}\left(t\right)\right)-\mathcal{L}\left(\mathbf{w}\left(t+1\right)\right)}{\eta\left(1-\frac{\beta\eta}{2}\right)}\geq\left\Vert \nabla\mathcal{L}\left(\mathbf{w}\left(t\right)\right)\right\Vert ^{2}
\]
which implies 
\[
\sum_{u=0}^{t}\left\Vert \nabla\mathcal{L}\left(\mathbf{w}\left(u\right)\right)\right\Vert ^{2}\leq\sum_{u=0}^{t}\frac{\mathcal{L}\left(\mathbf{w}\left(u\right)\right)-\mathcal{L}\left(\mathbf{w}\left(u+1\right)\right)}{\eta\left(1-\frac{\beta\eta}{2}\right)}=\frac{\mathcal{L}\left(\mathbf{w}\left(0\right)\right)-\mathcal{L}\left(\mathbf{w}\left(t+1\right)\right)}{\eta\left(1-\frac{\beta\eta}{2}\right)}\,.
\]
The right hand side is upper bounded by a finite constant, since $L\left(\mathbf{w}\left(0\right)\right)<\infty$
and $0\leq\mathcal{L}\left(\mathbf{w}\left(t+1\right)\right)$. This
implies 
\[
\sum_{u=0}^{\infty}\left\Vert \nabla\mathcal{L}\left(\mathbf{w}\left(u\right)\right)\right\Vert ^{2}<\infty\,,
\]
and therefore $\left\Vert \nabla\mathcal{L}\left(\mathbf{w}\left(t\right)\right)\right\Vert ^{2}\rightarrow0$.  $\blacksquare$

\subsection{Proof of Lemma \ref{lem: correlation bound}\label{sec:Proof-of-Lemma correlation}}

Recall that we defined $\mathbf{r}\left(t\right)=\mathbf{w}\left(t\right)-\hat{\mathbf{w}}\log t-\tilde{\mathbf{w}}$,
with $\hat{\mathbf{w}}$ and $\tilde{\mathbf{w}}$ follow the conditions
of the Theorems \ref{thm: main theorem} and \ref{thm: refined Theorem},
\emph{i.e}, $\hat{\mathbf{w}}$ is the $L_{2}$ max margin vector
and (eq. \ref{eq: max margin vector}), and eq. \ref{eq: w tilde}
holds 
\[
\forall n\in\set:\,\eta\exp\left(-\mathbf{x}_{n}^{\top}\tilde{\mathbf{w}}\right)=\alpha_{n}\,.
\]

\correlation*

From Lemma \ref{lem: convergence of linear classifiers}, $\forall n:\,\lim_{t\rightarrow\infty}\mathbf{w}\left(t\right)^{\top}\mathbf{x}_{n}=\infty$.
In addition, from assumption \ref{assum: exponential tail} the negative
loss derivative $-\ell^{\prime}\left(u\right)$ has an exponential
tail $e^{-u}$ (recall we assume $a=c=1$ without loss of generality).
Combining both facts, we have positive constants $\mu_{-},\mu_{+}$,
$t_{-}$ and $t_{+}$ such that $\forall n$ 
\begin{align}
 & \!\!\forall t>t_{+}:-\ell^{\prime}\left(\mathbf{w}\left(t\right)^{\top}\mathbf{x}_{n}\right)\leq\left(1+\exp\left(-\mu_{+}\mathbf{w}\left(t\right)^{\top}\mathbf{x}_{n}\right)\right)\exp\left(-\mathbf{w}\left(t\right)^{\top}\mathbf{x}_{n}\right)\label{eq: exp bound top}\\
 & \!\!\forall t>t_{-}:-\ell^{\prime}\left(\mathbf{w}\left(t\right)^{\top}\mathbf{x}_{n}\right)\geq\left(1-\exp\left(-\mu_{-}\mathbf{w}\left(t\right)^{\top}\mathbf{x}_{n}\right)\right)\exp\left(-\mathbf{w}\left(t\right)^{\top}\mathbf{x}_{n}\right)\label{eq: exp bound bottom}
\end{align}
Next, we examine the expression we wish to bound, recalling that $\mathbf{r}\left(t\right)=\mathbf{w}\left(t\right)-\hat{\mathbf{w}}\log t-\tilde{\mathbf{w}}$:
\begin{align}
 & \left(\mathbf{r}\left(t+1\right)-\mathbf{r}\left(t\right)\right)^{\top}\mathbf{r}\left(t\right)\nonumber \\
 & =\left(-\eta\nabla\mathcal{L}\left(\mathbf{w}\left(t\right)\right)-\hat{\mathbf{w}}\left[\log\left(t+1\right)-\log\left(t\right)\right]\right)^{\top}\mathbf{r}\left(t\right)\nonumber \\
 & =-\eta\sum_{n=1}^{N}\ell^{\prime}\left(\mathbf{w}\left(t\right)^{\top}\mathbf{x}_{n}\right)\mathbf{x}_{n}^{\top}\mathbf{r}\left(t\right)-\hat{\mathbf{w}}^{\top}\mathbf{r}\left(t\right)\log\left(1+t^{-1}\right)\nonumber \\
 & =\hat{\mathbf{w}}^{\top}\mathbf{r}\left(t\right)\left[t^{-1}-\log\left(1+t^{-1}\right)\right]-\eta\sum_{n\notin\set}\ell^{\prime}\left(\mathbf{w}\left(t\right)^{\top}\mathbf{x}_{n}\right)\mathbf{x}_{n}^{\top}\mathbf{r}\left(t\right)\label{eq: norm r dot}\\
 & -\eta\sum_{n\in\set}\left[t^{-1}\exp\left(-\tilde{\mathbf{w}}^{\top}\mathbf{x}_{n}\right)+\ell^{\prime}\left(\mathbf{w}\left(t\right)^{\top}\mathbf{x}_{n}\right)\right]\mathbf{x}_{n}^{\top}\mathbf{r}\left(t\right)\nonumber 
\end{align}
where in last line we used eqs. \ref{eq:kkt} and \ref{eq: w tilde}
to obtain 
\[
\hat{\mathbf{w}}=\sum_{n\in\set}\alpha_{n}\mathbf{x}_{n}=\eta\sum_{n\in\set}\exp\left(-\tilde{\mathbf{w}}^{\top}\mathbf{x}_{n}\right)\mathbf{x}_{n}\,.
\]
We examine the three terms in eq. \ref{eq: norm r dot}. The first
term can be upper bounded by 
\begin{align}
 & \hat{\mathbf{w}}^{\top}\mathbf{r}\left(t\right)\left[t^{-1}-\log\left(1+t^{-1}\right)\right]\nonumber \\
\leq & \max\left[\hat{\mathbf{w}}^{\top}\mathbf{r}\left(t\right),0\right]\left[t^{-1}-\log\left(1+t^{-1}\right)\right]\nonumber \\
\overset{\left(1\right)}{\leq} & \max\left[\hat{\mathbf{w}}^{\top}\mathbf{P}_{1}\mathbf{r}\left(t\right),0\right]t^{-2}\nonumber \\
\overset{\left(2\right)}{\leq} & \begin{cases}
\mathbf{\left\Vert \hat{\mathbf{w}}\right\Vert }\epsilon_{1}t^{-2} & ,\,\mathrm{if}\,\left\Vert \mathbf{P}_{1}\mathbf{r}\left(t\right)\right\Vert \leq\epsilon_{1}\\
o\left(t^{-1}\right) & ,\,\mathrm{if}\,\left\Vert \mathbf{P}_{1}\mathbf{r}\left(t\right)\right\Vert >\epsilon_{1}
\end{cases}\label{eq: w_hat r bound 1}
\end{align}
where in $\left(1\right)$ we used that $\bar{\mathbf{P}}_{1}\hat{\mathbf{w}}=\bar{\mathbf{P}}_{1}\mathbf{X}_{\set}\boldsymbol{\alpha}=0$
from eq. \ref{eq:kkt}, and in $\left(2\right)$ we used that $\hat{\mathbf{w}}^{\top}\mathbf{r}\left(t\right)=o\left(t\right)$,
since 
\begin{align*}\hat{\mathbf{w}}^{\top}\mathbf{r}\left(t\right) & =\hat{\mathbf{w}}^{\top}\left(\mathbf{w}\left(0\right)-\eta\sum_{u=0}^{t}\nabla\mathcal{L}\left(\mathbf{w}\left(u\right)\right)-\hat{\mathbf{w}}\log\left(t\right)-\tilde{\mathbf{w}}\right)\\
 & \leq\hat{\mathbf{w}}^{\top}\left(\mathbf{w}\left(0\right)-\tilde{\mathbf{w}}-\hat{\mathbf{w}}\log\left(t\right)\right)+\eta\left\Vert \hat{\mathbf{w}}\right\Vert \sum_{u=0}^{t}\left\Vert \nabla\mathcal{L}\left(\mathbf{w}\left(u\right)\right)\right\Vert \\
 & \leq O\left(\log\left(t\right)\right)+\eta\left\Vert \hat{\mathbf{w}}\right\Vert \sum_{u=0}^{\left\lceil \sqrt{t}\right\rceil }\left\Vert \nabla\mathcal{L}\left(\mathbf{w}\left(u\right)\right)\right\Vert +\eta\left\Vert \hat{\mathbf{w}}\right\Vert \sum_{u=\left\lceil \sqrt{t}\right\rceil }^{t}\left\Vert \nabla\mathcal{L}\left(\mathbf{w}\left(u\right)\right)\right\Vert \\
 & \leq O\left(\log\left(t\right)\right)+\eta\left\lceil \sqrt{t}\right\rceil \left\Vert \hat{\mathbf{w}}\right\Vert \max_{0\leq u\leq\left\lceil \sqrt{t}\right\rceil }\left\Vert \nabla\mathcal{L}\left(\mathbf{w}\left(u\right)\right)\right\Vert +\eta t\left\Vert \hat{\mathbf{w}}\right\Vert \max_{\left\lceil \sqrt{t}\right\rceil \leq u\leq t}\left\Vert \nabla\mathcal{L}\left(\mathbf{w}\left(u\right)\right)\right\Vert \\
 & =O\left(\log\left(t\right)\right)+\left\lceil \sqrt{t}\right\rceil O\left(1\right)+o\left(1\right)t=o\left(t\right) \, ,
\end{align*}
where in the last line we used that $\nabla\mathcal{L}\left(\mathbf{w}\left(t\right)\right)=o\left(1\right)$,
from Lemma \ref{lem: GD convergence}.

Next, we upper bound the second term in eq. \ref{eq: norm r dot}. From eq. \ref{eq: exp bound top} $\exists t_+^{\prime}$, such that $\forall >t_0>t_+^{\prime}$, 
\begin{align}
&\ell'(\wvec(t)^\top \xn)\le 2\exp(-\wvec(t)^\top \xn).\label{eq: two factor inequality}
\end{align}

Therefore, $\forall t>t_{+}^{\prime}$: 
\begin{align}
 & -\eta\sum_{n\notin\set}\ell^{\prime}\left(\mathbf{w}\left(t\right)^{\top}\mathbf{x}_{n}\right)\mathbf{x}_{n}^{\top}\mathbf{r}\left(t\right)\nonumber \\
\leq & -\eta\sum_{n\notin\set:\,\mathbf{x}_{n}^{\top}\mathbf{r}\left(t\right)\geq0}\!\!\!\!\!\!\!\ell^{\prime}\left(\mathbf{w}\left(t\right)^{\top}\mathbf{x}_{n}\right)\mathbf{x}_{n}^{\top}\mathbf{r}\left(t\right)\nonumber \\
\overset{\left(1\right)}{\leq} & \eta\sum_{n\notin\set:\,\mathbf{x}_{n}^{\top}\mathbf{r}\left(t\right)\geq0}\!\!\!\!\!\!\!2\exp\left(-\mathbf{w}\left(t\right)^{\top}\mathbf{x}_{n}\right)\mathbf{x}_{n}^{\top}\mathbf{r}\left(t\right)\nonumber \\
\overset{\left(2\right)}{\leq} & \eta\sum_{n\notin\set:\,\mathbf{x}_{n}^{\top}\mathbf{r}\left(t\right)\geq0}\!\!\!\!\!\!\!2t^{-\mathbf{x}_{n}^{\top}\hat{\mathbf{w}}}\exp\left(-\tilde{\mathbf{w}}^{\top}\mathbf{x}_{n}-\mathbf{x}_{n}^{\top}\mathbf{r}\left(t\right)\right)\mathbf{x}_{n}^{\top}\mathbf{r}\left(t\right)\nonumber \\
\overset{\left(3\right)}{\leq} & \eta\sum_{n\notin\set:\,\mathbf{x}_{n}^{\top}\mathbf{r}\left(t\right)\geq0}\!\!\!\!\!\!\!2t^{-\mathbf{x}_{n}^{\top}\hat{\mathbf{w}}}\exp\left(-\tilde{\mathbf{w}}^{\top}\mathbf{x}_{n}\right)\nonumber \\
\overset{\left(4\right)}{\leq} & \eta N\exp\left(-\min_{n}\mathbf{\tilde{w}}{}^{\top}\mathbf{x}_{n}\right)t^{-\theta} \label{eq: bound on non-SV gradient}
\end{align}

where in $\left(1\right)$ we used eq. \ref{eq: two factor inequality}, in
$\left(2\right)$ we used $\mathbf{w}\left(t\right)=\hat{\mathbf{w}}\log t+\tilde{\mathbf{w}}+\mathbf{r}\left(t\right)$,
in $\left(3\right)$ we used $xe^{-x}\leq1$ and $\mathbf{x}_{n}^{\top}\mathbf{r}\left(t\right)\geq0$,and
in $\left(4\right)$ we used $\theta>1$, from eq. \ref{eq: v1 SVM}.

Lastly, we will  bound the sum in the third term in eq. \ref{eq: norm r dot}
\begin{equation}
-\eta\sum_{n\in\set}\left[t^{-1}\exp\left(-\tilde{\mathbf{w}}^{\top}\mathbf{x}_{n}\right)+\ell^{\prime}\left(\mathbf{w}\left(t\right)^{\top}\mathbf{x}_{n}\right)\right]\mathbf{x}_{n}^{\top}\mathbf{r}\left(t\right)\,.\label{eq: minimum over S1 tilde}
\end{equation}
We examine each term $n$ in this sum, and divide into two cases,
depending on the sign of $\mathbf{x}_{n}^{\top}\mathbf{r}\left(t\right)$.

First, if $\mathbf{x}_{n}^{\top}\mathbf{r}\left(t\right)\geq0$, then
term $n$ in eq. \ref{eq: minimum over S1 tilde} can be upper bounded
$\forall t>t_{+}$, using eq. \ref{eq: exp bound top}, by 
\begin{equation}
\eta t^{-1}\exp\left(-\tilde{\mathbf{w}}^{\top}\mathbf{x}_{n}\right)\left[\left(1+t^{-\mu_{+}}\exp\left(-\mu_{+}\tilde{\mathbf{w}}^{\top}\mathbf{x}_{n}\right)\right)\exp\left(-\mathbf{x}_{n}^{\top}\mathbf{r}\left(t\right)\right)-1\right]\mathbf{x}_{n}^{\top}\mathbf{r}\left(t\right)\label{eq: positive case}
\end{equation}
We further divide into cases: 
\begin{enumerate}
\item If $\left|\mathbf{x}_{n}^{\top}\mathbf{r}(t)\right|\leq C_{0}t^{-0.5\mu_{+}}$,
then we can upper bound eq. \ref{eq: positive case} with 
\begin{equation}
\eta\exp\left(-\left(1+\mu_{+}\right)\min_{n}\tilde{\mathbf{w}}^{\top}\mathbf{x}_{n}\right)C_{0}t^{-1-1.5\mu_{+}}\,.\label{eq: positive case a}
\end{equation}
\item If $\left|\mathbf{x}_{n}^{\top}\mathbf{r}(t)\right|>C_{0}t^{-0.5\mu_{+}}$,
then we can find $t_{+}^{\prime\prime}>t_{+}^{\prime}$ to upper bound
eq. \ref{eq: positive case} $\forall t>t_{+}^{\prime\prime}$: 
\begin{align}
 & \eta t^{-1}e^{-\tilde{\mathbf{w}}^{\top}\mathbf{x}_{n}}\left[\left(1+t^{-\mu_{+}}e^{-\mu_{+}\tilde{\mathbf{w}}^{\top}\mathbf{x}_{n}}\right)\exp\left(-C_{0}t^{-0.5\mu_{+}}\right)-1\right]\mathbf{x}_{n}^{\top}\mathbf{r}\left(t\right)\,\nonumber \\
\overset{\left(1\right)}{\leq} & \eta t^{-1}e^{-\tilde{\mathbf{w}}^{\top}\mathbf{x}_{n}}\left[\left(1+t^{-\mu_{+}}e^{-\mu_{+}\tilde{\mathbf{w}}^{\top}\mathbf{x}_{n}}\right)\left(1-C_{0}t^{-0.5\mu_{+}}+C_{0}^{2}t^{-\mu_{+}}\right)-1\right]\mathbf{x}_{n}^{\top}\mathbf{r}\left(t\right)\nonumber \\
\leq & \eta t^{-1}e^{-\tilde{\mathbf{w}}^{\top}\mathbf{x}_{n}}\left[\left(1-C_{0}t^{-0.5\mu_{+}}+C_{0}^{2}t^{-\mu_{+}}\right)e^{-\mu_{+}\min\limits_{n}\tilde{\mathbf{w}}^{\top}\mathbf{x}_{n}}t^{-\mu_{+}}-C_{0}t^{-0.5\mu_{+}}+C_{0}^{2}t^{-\mu_{+}}\right]\mathbf{x}_{n}^{\top}\mathbf{r}\left(t\right)\nonumber \\
\overset{\left(2\right)}{\leq} & 0,\,\forall t>t_{+}^{\prime\prime}\label{eq: positive case b}
\end{align}
where in $\left(1\right)$ we used the fact that $e^{-x}\leq1-x+x^{2}$
for $x\geq0$ and in $\left(2\right)$ we defined $t_{+}^{\prime\prime}$
so that the previous expression is negative \textemdash{} since $t^{-0.5\mu_{+}}$ decreases slower than $t^{-\mu_{+}}$. 
\item If $\left|\mathbf{x}_{n}^{\top}\mathbf{r}(t)\right|\geq\epsilon_{2}$,
then we define $t_{+}^{\prime\prime\prime}>t_{+}^{\prime\prime}$
such that $t_{+}^{\prime\prime\prime}>\exp\left(\min_{n}\tilde{\mathbf{w}}^{\top}\mathbf{x}_{n}\right)\left[e^{0.5\epsilon_{2}}-1\right]^{-1/\mu_{+}}$,
and therefore $\forall t>t_{+}^{\prime\prime\prime}$, we have $\left(1+t^{-\mu_{+}}\exp\left(-\mu_{+}\tilde{\mathbf{w}}^{\top}\mathbf{x}_{n}\right)\right)e^{-\epsilon_{2}}<e^{-0.5\epsilon_{2}}$~.

This implies that $\forall t>t_{+}^{\prime\prime\prime}$ we can upper
bound eq. \ref{eq: positive case} by 
\begin{align}
-\eta\exp\left(-\max_{n}\tilde{\mathbf{w}}^{\top}\mathbf{x}_{n}\right)\left(1-e^{-0.5\epsilon_{2}}\right)\epsilon_{2}t^{-1}.\label{eq: positive case c}
\end{align}
\end{enumerate}
Second, if $\mathbf{x}_{n}^{\top}\mathbf{r}(t)<0$, we again further
divide into cases: 
\begin{enumerate}
\item If $\left|\mathbf{x}_{n}^{\top}\mathbf{r}(t)\right|\leq C_{0}t^{-0.5\mu_{-}}$,
then, since $-\ell^{\prime}\left(\mathbf{w}\left(t\right)^{\top}\mathbf{x}_{n}\right)>0$,
we can upper bound term $n$ in eq. \ref{eq: minimum over S1 tilde}
with 
\begin{equation}
\eta t^{-1}\exp\left(-\tilde{\mathbf{w}}^{\top}\mathbf{x}_{n}\right)\left|\mathbf{x}_{n}^{\top}\mathbf{r}\left(t\right)\right|\leq\eta\exp\left(-\min_{n}\mathbf{\tilde{w}}{}^{\top}\mathbf{x}_{n}\right)C_{0}t^{-1-0.5\mu_{-}}\label{eq: negative case a}
\end{equation}
\item If $\left|\mathbf{x}_{n}^{\top}\mathbf{r}\left(t\right)\right|>C_{0}t^{-0.5\mu_{-}}$
, then, using eq. \ref{eq: exp bound bottom} we upper bound term
$n$ in eq. \ref{eq: minimum over S1 tilde} with 
\begin{align}
 & \eta\left[-t^{-1}e^{-\tilde{\mathbf{w}}^{\top}\mathbf{x}_{n}}-\ell^{\prime}\left(\mathbf{w}\left(t\right)^{\top}\mathbf{x}_{n}\right)\right]\mathbf{x}_{n}^{\top}\mathbf{r}\left(t\right)\nonumber \\
\leq & \eta\left[-t^{-1}e^{-\tilde{\mathbf{w}}^{\top}\mathbf{x}_{n}}+\left(1-\exp\left(-\mu_{-}\mathbf{w}\left(t\right)^{\top}\mathbf{x}_{n}\right)\right)\exp\left(-\mathbf{w}\left(t\right)^{\top}\mathbf{x}_{n}\right)\right]\mathbf{x}_{n}^{\top}\mathbf{r}\left(t\right)\nonumber \\
= & \eta t^{-1}e^{-\tilde{\mathbf{w}}^{\top}\mathbf{x}_{n}}\left[1-\exp\left(-\mathbf{r}\left(t\right)^{\top}\mathbf{x}_{n}\right)\left(1-\left[t^{-1}e^{-\tilde{\mathbf{w}}^{\top}\mathbf{x}_{n}}\exp\left(-\mathbf{r}\left(t\right)^{\top}\mathbf{x}_{n}\right)\right]^{\mu_{-}}\right)\right]\left|\mathbf{x}_{n}^{\top}\mathbf{r}\left(t\right)\right|\label{eq: case 2 2}
\end{align}
Next, we will show that $\exists t_{-}^{\prime}>t_{-}$ such that
the last expression is strictly negative $\forall t>t_{-}^{\prime}$.
Let $M>1$ be some arbitrary constant. Then, since $\left[t^{-1}e^{-\tilde{\mathbf{w}}^{\top}\mathbf{x}_{n}}\exp\left(-\mathbf{r}\left(t\right)^{\top}\mathbf{x}_{n}\right)\right]^{\mu_{-}}=\exp\left(-\mu_{-}\mathbf{w}\left(t\right)^{\top}\mathbf{x}_{n}\right)\rightarrow0$
from Lemma \ref{lem: convergence of linear classifiers}, $\exists t_{M}>\max (t_{-}, Me^{-\tilde{\mathbf{w}}^{\top}\mathbf{x}_{n}})$
such that $\forall t>t_{M}$, if $\exp\left(-\mathbf{r}\left(t\right)^{\top}\mathbf{x}_{n}\right)\geq M>1$
then 
\begin{equation}
\exp\left(-\mathbf{r}\left(t\right)^{\top}\mathbf{x}_{n}\right)\left(1-\left[t^{-1}e^{-\tilde{\mathbf{w}}^{\top}\mathbf{x}_{n}}\exp\left(-\mathbf{r}\left(t\right)^{\top}\mathbf{x}_{n}\right)\right]^{\mu_{-}}\right)\geq M^{\prime}>1\,.\label{eq: M bound 1}
\end{equation}
Furthermore, if $\exists t>t_{M}$ such that $\exp\left(\mathbf{r}\left(t\right)^{\top}\mathbf{x}_{n}\right)<M$,
then 
\begin{align}
 & \exp\left(-\mathbf{r}\left(t\right)^{\top}\mathbf{x}_{n}\right)\left(1-\left[t^{-1}e^{-\tilde{\mathbf{w}}^{\top}\mathbf{x}_{n}}\exp\left(-\mathbf{r}\left(t\right)^{\top}\mathbf{x}_{n}\right)\right]^{\mu_{-}}\right)\nonumber \\
> & \exp\left(-\mathbf{r}\left(t\right)^{\top}\mathbf{x}_{n}\right)\left(1-\left[t^{-1}e^{-\tilde{\mathbf{w}}^{\top}\mathbf{x}_{n}}M\right]^{\mu_{-}}\right).\label{eq: M bound 2}
\end{align}
which is lower bounded by 
\begin{align*}
 & \left(1+C_{0}t^{-0.5\mu_{-}}\right)\left(1-t^{-\mu_{-}}\left[e^{-\tilde{\mathbf{w}}^{\top}\mathbf{x}_{n}}M\right]^{\mu_{-}}\right)\\
\geq & 1+C_{0}t^{-0.5\mu_{-}}-t^{-\mu_{-}}\left[e^{-\tilde{\mathbf{w}}^{\top}\mathbf{x}_{n}}M\right]^{\mu_{-}}-t^{-1.5\mu_{-}}\left[e^{-\tilde{\mathbf{w}}^{\top}\mathbf{x}_{n}}M\right]^{\mu_{-}}C_{0}
\end{align*}
since $\left|\mathbf{x}_{n}^{\top}\mathbf{r}\left(t\right)\right|>C_{0}t^{-0.5\mu_{-}}$,
$\mathbf{x}_{n}^{\top}\mathbf{r}\left(t\right)<0$ and $e^{x}\geq1+x$.
In this case last line is strictly larger than $1$ for sufficiently
large $t$. Therefore, after we substitute eqs. \ref{eq: M bound 1}
and \ref{eq: M bound 2} into \ref{eq: case 2 2}, we find that $\exists t_{-}^{\prime}>t_{M}>t_{-}$
such that $\forall t>t_{-}^{\prime}$, term $k$ in eq. \ref{eq: minimum over S1 tilde}
is strictly negative 
\begin{equation}
\eta\left[-t^{-1}e^{-\tilde{\mathbf{w}}^{\top}\mathbf{x}_{k}}-\ell^{\prime}\left(\mathbf{w}\left(t\right)^{\top}\mathbf{x}_{k}\right)\right]\mathbf{x}_{k}^{\top}\mathbf{r}\left(t\right)<0\label{eq: negative case b}
\end{equation}
\item If $\left|\mathbf{x}_{k}^{\top}\mathbf{r}(t)\right|\geq\epsilon_{2}$
, which is a special case of the previous case ($\left|\mathbf{x}_{k}^{\top}\mathbf{r}\left(t\right)\right|>C_{0}t^{-0.5\mu_{-}}$)
then $\forall t>t_{-}^{\prime}$, either eq. \ref{eq: M bound 1}
or \ref{eq: M bound 2} holds. Furthermore, in this case, $\exists t_{-}^{\prime\prime}>t_{-}^{\prime}$
and $M^{\prime\prime}>1$ such that $\forall t>t_{-}^{\prime\prime}$
eq. \ref{eq: M bound 2} can be lower bounded by 
\[
\exp\left(\epsilon_{2}\right)\left(1-\left[t^{-1}e^{-\tilde{\mathbf{w}}^{\top}\mathbf{x}_{k}}M\right]^{\mu_{-}}\right)>M^{\prime\prime}>1\,.
\]
Substituting this, together with eq. \ref{eq: M bound 1}, into eq.
\ref{eq: case 2 2}, we can find $C_{0}^{\prime}>0$ such we can upper
bound term $k$ in eq. \ref{eq: minimum over S1 tilde} with 
\begin{equation}
-C_{0}^{\prime}t^{-1}\,,\,\forall t>t_{-}^{\prime\prime}\,.\label{eq: negative case c}
\end{equation}
\end{enumerate}
To conclude, we choose $t_{0}=\max\left[t_{+}^{\prime\prime\prime},t_{-}^{\prime\prime}\right]$: 
\begin{enumerate}
\item If $\left\Vert \mathbf{P}_{1}\mathbf{r}\left(t\right)\right\Vert \geq\epsilon_{1}$
(as in Eq. \ref{eq: bounded r(t)}), we have that 
\begin{equation}
\max_{n\in\set}\left|\mathbf{x}_{n}^{\top}\mathbf{r}\left(t\right)\right|^{2}\overset{\left(1\right)}{\geq}\frac{1}{\left|\set\right|}\sum_{n\in\set}\left|\mathbf{x}_{n}^{\top}\mathbf{P}_{1}\mathbf{r}\left(t\right)\right|^{2}=\frac{1}{\left|\set\right|}\left\Vert \mathbf{X}_{\set}^{\top}\mathbf{P}_{1}\mathbf{r}\left(t\right)\right\Vert ^{2}\overset{\left(2\right)}{\geq}\frac{1}{\left|\set\right|}\sigma_{\min}^{2}\left(\mathbf{X}_{\set}\right)\epsilon_{1}^{2}\label{eq: epsilon 2}
\end{equation}
where in $\left(1\right)$ we used $\mathbf{P}_{1}^{\top}\mathbf{x}_{n}=\mathbf{x}_{n}$
$\forall n\in\set$, in $\left(2\right)$ we denoted by $\sigma_{\min}\left(\mathbf{X}_{\set}\right)$,
the minimal non-zero singular value of $\mathbf{X}_{\set}$ and used
eq. \ref{eq: bounded r(t)}. Therefore, for some $k$, $\left|\mathbf{x}_{k}^{\top}\mathbf{r}\right|\geq\epsilon_{2}\triangleq\sqrt{\left|\set\right|^{-1}\sigma_{\min}^{2}\left(\mathbf{X}_{\set}\right)\epsilon_{1}^{2}}$.
In this case, we denote $C_{0}^{\prime\prime}$ as the minimum between
$C_{0}^{\prime}$ (eq. \ref{eq: negative case c}) and $\eta\exp\left(-\max_{n}\tilde{\mathbf{w}}^{\top}\mathbf{x}_{n}\right)\left(1-e^{-0.5\epsilon_{2}}\right)\epsilon_{2}$
(eq. \ref{eq: positive case c}). Then we find that eq. \ref{eq: minimum over S1 tilde}
can be upper bounded by $-C_{0}^{\prime\prime}t^{-1}+o\left(t^{-1}\right)$,
$\forall t>t_{0}$, given eq. \ref{eq: bounded r(t)}. Substituting
this result, together with eqs. \ref{eq: w_hat r bound 1} and \ref{eq: bound on non-SV gradient}
into eq. \ref{eq: norm r dot}, we obtain $\forall t>t_{0}$ 
\[
\left(\mathbf{r}\left(t+1\right)-\mathbf{r}\left(t\right)\right)^{\top}\mathbf{r}\left(t\right)\leq-C_{0}^{\prime\prime}t^{-1}+o\left(t^{-1}\right)\,.
\]
This implies that $\exists C_{2}<C_{0}^{\prime\prime}$ and $\exists t_{2}>t_{0}$
such that eq. \ref{eq: bounded cases} holds. This implies also that
eq. \ref{eq: general case} holds for $\left\Vert \mathbf{P}_{1}\mathbf{r}\left(t\right)\right\Vert \geq\epsilon_{1}$. 
\item Otherwise, if $\left\Vert \mathbf{P}_{1}\mathbf{r}\left(t\right)\right\Vert <\epsilon_{1}$,
we find that $\forall t>t_{0}$ , each term in eq. \ref{eq: minimum over S1 tilde}
can be upper bounded by either zero (eqs. \ref{eq: positive case b}
and \ref{eq: negative case b}), or terms proportional to $t^{-1-1.5\mu_{+}}$
(eq. \ref{eq: positive case a}) or $t^{-1-0.5\mu_{-}}$, (eq. \ref{eq: negative case a}).
Combining this together with eqs. \ref{eq: w_hat r bound 1}, \ref{eq: bound on non-SV gradient}
into eq. \ref{eq: norm r dot} we obtain (for some positive constants
$C_{3}$, $C_{4}$, $C_{5}$, and $C_{6}$) 
\[
\left(\mathbf{r}\left(t+1\right)-\mathbf{r}\left(t\right)\right)^{\top}\mathbf{r}\left(t\right)\leq C_{3}t^{-1-1.5\mu_{+}}+C_{4}t^{-1-0.5\mu_{-}}+C_{5}t^{-2}+C_{6}t^{-\theta}\,.
\]
Therefore, $\exists t_{1}>t_{0}$ and $C_{1}$ such that eq. \ref{eq: general case}
holds. $\blacksquare$
\end{enumerate}

\section{Generic solutions of the KKT conditions in eq. \ref{eq:kkt} \label{sec:alpha}}
\begin{lem}
\label{lem: alpha} For almost all datasets there is a unique $\boldsymbol{\alpha}$
which satisfies the KKT conditions (eq. \ref{eq:kkt}): 
\[
\what=\sum_{n=1}^{N}\alpha_{n}\x_{n}\quad\quad\forall n\;\left(\alpha_{n}\geq0\;\textrm{and}\;\what^{\top}\x_{n}=1\right)\;\;\textrm{OR}\;\;\left(\alpha_{n}=0\;\textrm{and}\;\what^{\top}\x_{n}>1\right)
\]
Furthermore, in this solution $\alpha_{n}\neq0$ if $\what^{\top}\mathbf{x}_{n}=1$,
{\em{i.e.}}, $\x_{n}$ is a support vector ($n\in\set$), and
there are at most $d$ such support vectors. 
\end{lem}

For almost every set $\mathbf{X}$, no more than $d$ points $\mathbf{x}_{n}$
can be on the same hyperplane. Therefore, since all support vectors
must lie on the same hyperplane, there can be at most $d$ support
vectors, for almost every $\mathbf{X}$.

Given the set of support vectors, $\set$, the KKT conditions of eq.
\ref{eq:kkt} entail that $\alpha_{n}=0$ if $n\notin\set$ and 
\begin{equation}
\mathbf{1}=\mathbf{X}_{\set}^{\top}\what=\mathbf{X}_{\set}^{\top}\mathbf{X}_{\set}\boldsymbol{\alpha}_{\set}\,,\label{eq: dual KKT}
\end{equation}
where we denoted \textbf{$\boldsymbol{\alpha}_{\set}$ }as $\boldsymbol{\alpha}$
restricted to the support vector components. For almost every set
$\mathbf{X}$, since $d\geq\left|\set\right|$, $\mathbf{X}_{\set}^{\top}\mathbf{X}_{\set}\in\mathbb{R}^{\left|\set\right|\times\left|\set\right|}$
is invertible. Therefore, $\boldsymbol{\alpha}_{\set}$ has the unique
solution 
\begin{equation}
\left(\mathbf{X}_{\set}^{\top}\mathbf{X}_{\set}\right)^{-1}\mathbf{1}=\boldsymbol{\alpha}_{\set}\,.\label{eq: inverted dual KKT}
\end{equation}
This implies that $\forall n\in\set$, $\alpha_{n}$ is equal to a
rational function in the components of $\mathbf{X}_{S}$, \emph{i.e.},
$\alpha_{n}=p_{n}\left(\text{\textbf{X}}_{\set}\right)/q_{n}\left(\text{\textbf{X}}_{\set}\right)$,
where $p_{n}$ and $q_{n}$ are polynomials in the components of $\mathbf{X}_{S}$.
Therefore, if $\alpha_{n}=0$, then $p_{n}\left(\text{\textbf{X}}_{\set}\right)=0$,
so the components of $\mathbf{X}_{\set}$ must be at a root of the
polynomial $p_{n}$. The roots of the polynomial $p_{n}$ have measure
zero, unless $\forall\mathbf{X}_{\set}:\,\,p_{n}\left(\text{\textbf{X}}_{\set}\right)=0$.
However, $p_{n}$ cannot be identically equal to zero, since, for
example, if $\mathbf{X}_{\set}^{\top}=\left[\mathbf{I}_{\left|\set\right|\times\left|\set\right|},\boldsymbol{0}_{\left|\set\right|\times\left(d-\left|\set\right|\right)}\right]$,
then $\mathbf{X}_{\set}^{\top}\mathbf{X}_{\set}=\mathbf{I}_{\left|\set\right|\times\left|\set\right|}$,
and so in this case $\forall n\in\set$, $\alpha_{n}=1\neq0$, from
eq. \ref{eq: inverted dual KKT}.

Therefore, for a given $\set$, the event that ``eq.
\ref{eq: dual KKT} has a solution with a zero component''
has a zero measure. Moreover, the union of these events, for all possible
$\set$, also has zero measure, as a finite union of zero measures
sets (there are only finitely many possible sets $\set\subset\left\{ 1,\dots,N\right\} $
). This implies that, for almost all datasets $\mathbf{X}$, $\alpha_{n}=0$
only if $n\notin\set$. Furthermore, for almost all datasets the solution
$\boldsymbol{\alpha}$ is unique: for each dataset, $\set$ is uniquely
determined, and given $\set$ , the solution eq. \ref{eq: dual KKT}
is uniquely given by eq. \ref{eq: inverted dual KKT}. $\blacksquare$

\section{Completing the proof of Theorem  \ref{thm: main theorem} for zero measure cases \label{sec: proof of degenerate case}}

In the preceding Appendices, we established Theorem \ref{thm: refined Theorem}, which only applied when all support vectors are associated with non-zero coefficients.  This characterizes almost all data sets, \emph{i.e.}~all except for measure zero.  We now turn to presenting and proving a more complete characterization of the limit behaviour of gradient descent, which covers all data sets, including those degenerate data sets not covered by Theorem \ref{thm: refined Theorem}, thus establishing Theorem \ref{thm: main theorem}.

\newcommand{\barP}{\bar{\mathbf{P}}}
\newcommand{\barS}{\bar{\set}}

In order to do so, we first have to introduce additional notation and a recursive treatment of the data set.  We will define a sequence of data sets $\barP_m \mathbf{X}_{\barS_m}$ obtained by considering only a subset $\barS_m$ of the points, and projecting them using the projection matrix $\barP_m$.  We start, for $m=0$, with the full original data set, \emph{i.e.}~$\barS_0=\{1,\ldots,N\}$ and $\barP_0=\mathbf{I}_{d \times d}$.  We then define $\hat{\mathbf{w}}_m$ as the max margin predictor for $\barP_{m-1} \mathbf{X}_{\barS_{m-1}}$, \emph{i.e.}:
\begin{equation}\label{eq:hatwm}
\hat{\mathbf{w}}_{m}=\underset{\mathbf{\mathbf{w}}\in\mathbb{R}^{d}}{\mathrm{argmin}}\left\lVert \mathbf{w}\right\rVert ^{2}\,\,\mathrm{s.t.}\,\,\mathbf{w}^{\top}\bar{\mathcal{\mathbf{P}}}_{m-1}\mathbf{x}_{n}\geq1\,\forall n\in\bar{\set}_{m-1}\,.
\end{equation}
In particular, $\hat{\mathbf{w}}_1$ is the max margin predictor for the original data set.  We then denote $\set_m^+$ the indices of non-support vectors for \ref{eq:hatwm}, $\set_m$ the indices of support vector of \ref{eq:hatwm} with non-zero coefficients for the dual variables corresponding to the margin constraints (for some dual solution), and $\barS_m$ the set of support vector with zero coefficients.  That is:
\begin{align}\label{eq:sm}
\set_{m}^{+}= & \left\{ n\in\bar{\set}_{m-1}|\what_{m}^{\top}\bar{\mathcal{\mathbf{P}}}_{m-1}\x_{n}>1\right\} \nonumber\\
\set_{m}^{=}= & \left\{ n\in\bar{\set}_{m-1}|\what_{m}^{\top}\bar{\mathcal{\mathbf{P}}}_{m-1}\x_{n}=1\right\} =\bar{\mathcal{S}}_{m}\setminus\set_{m}^{+}\nonumber\\
\set_{m}= & \left\{ n\in\set_{m}^{=}|\exists\boldsymbol{\alpha}\in\mathbb{R}_{\geq0}^{N}:\mathbf{\hat{w}}_{m}=\sum_{k=1}^{N}\alpha_{k}\bar{\mathcal{\mathbf{P}}}_{m-1}\x_{k},\alpha_{n}>0,\forall i\notin\set_{m}^{=}:\,\alpha_{i}=0\right\} \nonumber\\
\bar{\mathcal{S}}_{m}= &\, \, \set_{m}^{=} \setminus\set_{m}\,.
\end{align}
The problematic degenerate case, not covered by the analysis of Theorem \ref{thm: refined Theorem}, is when there are support vectors with zero coefficients, \emph{i.e.}, when $\barS_m\neq\emptyset$.  In this case we recurse on these zero-coefficient support vectors (\emph{i.e.}, on $\barS_m$), but only consider their components orthogonal to the non-zero-coefficient support vectors (\emph{i.e.}, not spanned by points in $\set_m$).  That is, we project using:
\begin{equation}
    \bar{\mathbf{P}}_{m}=\bar{\mathbf{P}}_{m-1}\left(\mathbf{I}_{d}-\mathbf{X}_{\set_{m}}\mathbf{X}_{\set_{m}}^{\pmpi}\right)
\end{equation}
where we denoted $\mathbf{A}^{\pmpi}$ as the Moore-Penrose pseudo-inverse 
of $\mathbf{A}$. We also denote $\mathbf{P}_{m}  =\mathbf{I}_{d}-\bar{\mathbf{P}}_{m}$.  

 This recursive treatment continues as long as $\barS_m\neq\emptyset$, defining a sequence $\hat{\mathbf{w}}_m$ of max margin predictors, for smaller and lower dimensional data sets $\barP_{m-1} \mathbf{X}_{\barS_{m-1}}$.  We stop when $\barS_m=\emptyset$ and denote the stopping stage $M$---that is, $M$ is the minimal $m$ such that $\barS_m=\emptyset$.  Our characterization will be in terms of the sequence $\hat{\mathbf{w}}_1,\ldots,\hat{\mathbf{w}}_M$.  As established in Lemma \ref{lem: alpha} of Appendix \ref{sec:alpha}, for almost all data sets we will not have support vectors with non-zero coefficients, and so we will have $M=1$, and so the characterization only depends on the max margin predictor $\hat{\mathbf{w}}_1$ of the original data set.  But, even for the measure zero of data sets in which $M>1$, we provide the following more complete characterization:
\begin{thm}
\label{theorem: main2}For all datasets which are linearly
separable (Assumption \ref{assum: Linear sepereability}) and given
a $\beta$-smooth loss function (Assumption \ref{assum: loss properties})
with an exponential tail (Assumption \ref{assum: exponential tail}), gradient descent (as in eq. \ref{eq: gradient descent linear})
with step size $\eta<2\beta^{-1}\sigma_{\max}^{-2}\left(\text{\ensuremath{\mathbf{X}} }\right)$
and any starting point $\w(0)$, the iterates of gradient descent can be written as:
\begin{equation}
\mathbf{w}\left(t\right)=\sum_{m=1}^{M}\hat{\mathbf{w}}_{m}\log^{\circ m}\left(t\right)+\boldsymbol{\rho}\left(t\right)\,,\label{eq: asymptotic form-1}
\end{equation}
where $\log^{\circ m}\left(t\right)=\overbrace{\log\log\cdots\log}^{m\,\mathrm{times}}\left(t\right)$,
$\hat{\mathbf{w}}_{m}$ is the $L_2$ max margin vector defined in eq.
\ref{eq:hatwm}, and the residual $\boldsymbol{\rho}\left(t\right)$
is bounded. 
\end{thm}

\subsection{Auxiliary notation}

We say that a function $f:\mathbb{N}\rightarrow\mathbb{R}$ is absolutely
summable if $\sum_{t=1}^{\infty}\left|f\left(t\right)\right|<\infty$,
and then we denote $f\left(t\right)\in L_{1}$. Furthermore, we define

\[
\mathbf{r}\left(t\right)=\mathbf{w}\left(t\right)-\sum_{m=1}^{M}\left[\hat{\mathbf{w}}_{m}\log^{\circ m}\left(t\right)+\tilde{\mathbf{w}}_{m}+\sum_{k=1}^{m-1}\frac{\check{\mathbf{w}}_{k,m}}{\prod_{r=k}^{m-1}\log^{\circ r}\left(t\right)}\right]
\]
where $\tilde{\mathbf{w}}_{m}$ and $\check{\mathbf{w}}_{k,m}$ are
defined next, and additionally, we denote

\[
\tilde{\mathbf{w}}=\sum_{m=1}^{M}\tilde{\mathbf{w}}_{m}\,.
\]

We define, $\forall m\geq1$, $\tilde{\mathbf{w}}_{m}$ as the solution
of 
\begin{equation}
\forall m\geq1:\forall n\in\set_{m}:\, \eta \sum_{n\in\set_{m}}\exp\left(-\sum_{k=1}^{m}\tilde{\mathbf{w}}_{k}^{\top}\mathbf{x}_{n}\right)\bar{\mathbf{P}}_{m-1}\mathbf{x}_{n}=\hat{\mathbf{w}}_{m}\,,\label{eq: w tilde-1}
\end{equation}
such that 
\begin{equation}
\mathbf{P}_{m-1}\tilde{\mathbf{w}}_{m}=0\,\mathrm{and}\,\bar{\mathbf{P}}_{m}\tilde{\mathbf{w}}_{m}=0.\label{eq: w tilde constraints}
\end{equation}
The existence and uniqueness of the solution, $\tilde{\mathbf{w}}_{m}$
are proved in appendix section \ref{subsec: existence 1}.

Lastly, we define, $\forall m>k\geq1$, $\check{\mathbf{w}}_{k,m}$
as the solution of 
\begin{equation}
\sum_{n\in\set_{m}}\exp\left(-\tilde{\mathbf{w}}^{\top}\mathbf{x}_{n}\right)\mathbf{P}_{m-1}\mathbf{x}_{n}=\sum_{k=1}^{m-1}\left[\sum_{n\in\set_{k}}\exp\left(-\tilde{\mathbf{w}}^{\top}\mathbf{x}_{n}\right)\mathbf{x}_{n}\mathbf{x}_{n}^{\top}\right]\check{\mathbf{w}}_{k,m}\label{eq: w check}
\end{equation}
such that
\begin{equation}
\mathbf{P}_{k-1}\check{\mathbf{w}}_{k,m}=0\,\mathrm{and}\,\mathbf{\bar{P}}_{k}\check{\mathbf{w}}_{k,m}=0\,.\label{eq: w check constraints}
\end{equation}
The existence and uniqueness of the solution $\check{\mathbf{w}}_{k,m}$
are proved in appendix section \ref{subsec: existence 2}. 

Together, eqs. \ref{eq: w tilde-1}-\ref{eq: w check constraints}
entail the existence of a unique decomposition, $\forall m\geq1:$
\begin{equation}
\hat{\mathbf{w}}_{m}=\eta \sum_{n\in\set_{m}}\exp\left(-\tilde{\mathbf{w}}^{\top}\mathbf{x}_{n}\right)\mathbf{x}_{n}-\eta \sum_{k=1}^{m-1}\left[\sum_{n\in\set_{k}}\exp\left(-\tilde{\mathbf{w}}^{\top}\mathbf{x}_{n}\right)\mathbf{x}_{n}\mathbf{x}_{n}^{\top}\right]\check{\mathbf{w}}_{k,m}\label{eq: w_hat full decomposition}
\end{equation}
given the constraints in eqs. \ref{eq: w tilde constraints} and \ref{eq: w check constraints}
hold. 

\subsection{Proof of Theorem \ref{theorem: main2}}

In the following proofs, for any solution $\wvec(t)$, we define 
\[
\tvec\left(t\right)=\sum_{m=2}^{M}\hat{\mathbf{w}}_{m}\log^{\circ m}\left(t\right)+\sum_{m=1}^{M}\sum_{k=1}^{m-1}\frac{\check{\mathbf{w}}_{k,m}}{\prod_{r=k}^{m-1}\log^{\circ r}\left(t\right)}
\]
noting that 
\[
\left\Vert \tvec\left(t+1\right)-\tvec\left(t\right)\right\Vert \leq\frac{C_{\tau}}{t\log\left(t\right)}
\]
and 
\begin{equation}
\rvec(t)=\wvec(t)-\hat{\mathbf{w}}_{1}\log\left(t\right)-\tilde{\mathbf{w}}-\tvec\left(t\right)\label{eq: define r(t) degenrate}
\end{equation}
where $\tilde{\mathbf{w}}$ follow the conditions of Theorem \ref{theorem: main2}.
Our goal is to show that $\Vert\rvec(t)\Vert$ is bounded. To show
this, we will upper bound the following equation 
\begin{equation}
\Vert\rvec(t+1)\Vert^{2}=\Vert\rvec(t+1)-\rvec(t)\Vert^{2}+2\left(\rvec(t+1)-\rvec(t)\right)^{\top}\rvec(t)+\Vert\rvec(t)\Vert^{2}\label{eq: norm r(t+1) degenerate}
\end{equation}
First, we note that $\exists t_{0}$ such that $\forall t>t_{0}$
the first term in this equation can be upper bounded by 
\begin{flalign}
 & ||\rvec(t+1)-\rvec(t)||^{2}\nonumber \\
 & \overset{(1)}{=}||\wvec(t+1)-\hat{\mathbf{w}}_{1}\log\left(t+1\right)-\tvec\left(t+1\right)-\wvec(t)+\hat{\mathbf{w}}_{1}\log\left(t\right)+\tvec\left(t\right)||^{2}\nonumber \\
 & \overset{(2)}{=}||-\eta\nabla L(\wvec(t))-\hat{\mathbf{w}}_{1}(\log\left(t+1\right)-\log\left(t\right))-(\tvec\left(t+1\right)-\tvec\left(t\right))||^{2}\nonumber \\
 & =\eta^{2}||\nabla L(\wvec(t))||^{2}+\left\Vert \hat{\mathbf{w}}_{1}\right\Vert ^{2}\log^{2}\left(1+t^{-1}\right)+\left\Vert \tvec\left(t+1\right)-\tvec\left(t\right)\right\Vert ^{2}\nonumber \\
 & +2\eta\nabla L(\wvec(t))^{\top}\left(\hat{\mathbf{w}}_{1}\log\left(1+t^{-1}\right)+\tvec\left(t+1\right)-\tvec\left(t\right)\right)\nonumber \\
 & +2\hat{\mathbf{w}}_{1}^{\top}(\tvec\left(t+1\right)-\tvec\left(t\right))\log\left(1+t^{-1}\right)\nonumber \\
 & \overset{(3)}{\le}\eta^{2}||\nabla L(\wvec(t))||^{2}+\left\Vert \hat{\mathbf{w}}_{1}\right\Vert ^{2}t^{-2}+C_{\tau}^2t^{-2}\log^{-2}\left(t\right)+2C_\tau\left\Vert \hat{\mathbf{w}}_{1}\right\Vert t^{-2}\log^{-1}(t)\,\,\,,\forall t>t_{0}\label{eq: norm(r(t+1)-r(t) degenerate}
\end{flalign}
where in (1) we used eq. \ref{eq: define r(t) degenrate}, in (2)
we used eq. \ref{eq: gradient descent linear} and in (3) we used
$\forall x>0:\,x\geq\log\left(1+x\right)>0$, and also using $\ell^{\prime}(\wvec(t)^{\top}\xn)<0$ for large enough $t$, we have that 
\begin{equation}\small
\left(\hat{\mathbf{w}}_{1}\log\left(1+t^{-1}\right)+\tvec\left(t+1\right)-\tvec\left(t\right)\right)^\top\derL\le\sumn\ell'(\wvec(t)^{\top}\xn)\left(\what_{1}^{\top}\xn\log\left(1+t^{-1}\right)-\frac{\left\Vert \xn\right\Vert C_{\tau}^{\prime}}{t\log\left(t\right)}\right)
\end{equation}
which is negative for sufficiently large $t_{0}$ (since $\log\left(1+t^{-1}\right)$
decreases as $t^{-1}$, which is slower then $1/\left(t\log\left(t\right)\right)$),
$\forall n:\,\what_{1}^{\top}\xn\ge1$ and $\ell^{\prime}(u)\leq0$.

Also, from Lemma \ref{lem: GD convergence} we know that: 
\begin{equation}
\Vert\nabla\mathcal{L}\left(\mathbf{w}\left(t\right)\right)\Vert^{2}=o(1)\text{ and }\sum_{u=0}^{\infty}\Vert\nabla\mathcal{L}(\wvec(u))\Vert^{2}<\infty\label{eq: derL converge degenerate}
\end{equation}
Substituting eq. \ref{eq: derL converge degenerate} into eq. \ref{eq: norm(r(t+1)-r(t) degenerate},
and recalling that $t^{-\nu_{1}}\log^{-\nu_{2}}\left(t\right)$ converges
for any $\nu_{1}>1$ and any $\nu_{2}$, and so
\begin{equation}
\kappa_{0}\left(t\right)\triangleq||\rvec(t+1)-\rvec(t)||^{2}\in L_{1}\,.\label{eq: sum norm r bound}
\end{equation}
Also, in the next subsection we will prove that

\begin{restatable}{lemR}{Correlation}

\label{lem: r correlation}Let $\kappa_{1}\left(t\right)$ and $\kappa_{2}\left(t\right)$
be functions in $L_{1}$, then 
\begin{equation}
\left(\mathbf{r}\left(t+1\right)-\mathbf{r}\left(t\right)\right)^{\top}\mathbf{r}\left(t\right)\leq\kappa_{1}\left(t\right)\left\Vert \mathbf{r}\left(t\right)\right\Vert +\kappa_{2}\left(t\right)\label{eq: general case-1}
\end{equation}

\end{restatable}
Thus, by combining eqs. \ref{eq: general case-1} and \ref{eq: sum norm r bound}
into eq. \ref{eq: norm r(t+1) degenerate}, we find 
\begin{align*}
\Vert\rvec(t+1)\Vert^{2} & \leq\kappa_{0}\left(t\right)+2\kappa_{1}\left(t\right)\left\Vert \mathbf{r}\left(t\right)\right\Vert +2\kappa_{2}\left(t\right)+\Vert\rvec(t)\Vert^{2}
\end{align*}
On this result we apply the following lemma (with $\phi\left(t\right)=\Vert\rvec(t)\Vert$,
$h\left(t\right)=2\kappa_{1}\left(t\right)$, and $z\left(t\right)=\kappa_{0}\left(t\right)+2\kappa_{2}\left(t\right)$), which we prove in appendix \ref{subsec: PhiSumProof}: 

\begin{restatable}{lemR}{PhiSum}

\label{lem: PhiSum}Let $\phi\left(t\right),h\left(t\right),z\left(t\right)$
be three functions from $\mathbb{N}$ to $\mathbb{R}_{\geq0}$, and
$C_{1},C_{2},C_{3}$ be three positive constants. Then, if $\sum_{t=1}^{\infty}h\left(t\right)\leq C_{1}<\infty$,
and 
\begin{equation}
\phi^{2}\left(t+1\right)\leq z\left(t\right)+h\left(t\right)\phi\left(t\right)+\phi^{2}\left(t\right)\,\label{eq: rho bound}
\end{equation}
we have
\begin{equation}
\phi^{2}\left(t+1\right)\leq C_{2}+C_{3}\sum_{u=1}^{t}z\left(u\right)\,\label{eq: rho sum bound}
\end{equation}

\end{restatable}
and obtain that 
\[
\Vert\rvec(t+1)\Vert^{2}\leq C_{2}+C_{3}\sum_{u=1}^{t}\left(\kappa_{0}\left(u\right)+2\kappa_{2}\left(u\right)\right)\leq C_{4}<\infty\,,
\]
since we assumed that $\forall i=0,1,2:\,\kappa_{i}\left(t\right)\in L_{1}$.
This completes our proof. $\blacksquare$

\subsection{Proof of Lemma \ref{lem: r correlation}}
Before we prove Lemma \ref{lem: r correlation}, we prove the following auxilary Lemma:
\begin{restatable}{lemR}{logtsum}\label{lem:logtsum}
Consider the function $f(t)=t^{-\nu_1}(\log(t))^{-\nu_2}(\log\log(t))^{-\nu_3}\ldots(\log^{\circ M}(t))^{-\nu_{M+1}}$. If $\exists m_0\le M+1$ such that $\nu_{m_0}>1$ and for all $m'<m_0$,$\nu_{m'}=1$, then $f(t)\in L_1$. \label{lem:int}
\end{restatable}
\begin{proof}To prove Lemma~\ref{lem:int}, we will show that the improper integeral $\int_{t_1}^\infty f(t)dt$ for any $t_1>0$ is bounded, \emph{i.e.}, $\forall t_1>0,\int_{t_1}^\infty f(t)dt<C$. Using the integeral test for convergence (or Maclaurin--Cauchy test) this in turn implies that $\forall t_1>0,\sum_{t_1}^\infty f(t)<C$, and thus $f(t)\in L_1$.

First, if $m_0>1$, then $\nu_1=\nu_2\ldots=\nu_{m_0-1}=1$ and $\nu_{m_0}=1+\epsilon$ for some $\epsilon>0$.  Using change of variables $y=\log^{\circ (m_0-1)}(t)$, we have 
\[
\mathrm{d}y=\left(t\prod_{r=1}^{m_0-2}\log^{\circ r}(t)\right)^{-1}\mathrm{d}t=t^{-\nu_1}\prod_{r=1}^{m_0-2}\left(\log^{\circ r}(t)\right)^{-\nu_{r+1}}\mathrm{d}t
\]
and for all $m>m_0$,  $\left(\log^{\circ (m-1)}(t)\right)^{-\nu_{m}}=\left(\log^{\circ( m-m_0)}(y)\right)^{-\nu_{m}}\le\left(\log(y)\right)^{|\nu_{m}|}$. Thus, denoting $\tilde{\nu}=\sum_{m=m_0+1}^{M+1}|\nu_{m}|$ and $\log^{\circ (m_0-1)}(t_1)=y_1$, we have 
\begin{equation}\label{eq:logsum}
    \int_{t_1}^\infty f(t)\mathrm{d}t=\int_{y_1}^\infty y^{-\nu_{m_0}}\prod_{m=m_0+1}^{M+1} \left(\log^{\circ m-m_0}(y)\right)^{-\nu_{m}}\mathrm{d}(y)\le \int_{y_1}^{\infty}\frac{\left(\log(y)\right)^{\tilde{\nu}}}{y^{1+\epsilon}} \mathrm{d}y.
\end{equation}

For $m_0=1$, we  have $\nu_1=1+\epsilon$ for some $\epsilon>0$, and for $m>1$, $\left(\log^{\circ (m-1)}(t)\right)^{-\nu_{m}}\le\left(\log(t)\right)^{|\nu_{m}|}$. Thus, denoting, $\tilde{\nu}=\sum_{m=2}^{M+1}|\nu_{m}|$, we have  $\int_{t_1}^\infty f(t)\mathrm{d}t\le \int_{t_1}^\infty\frac{\left(\log(t)\right)^{\tilde{\nu}}}{t^{1+\epsilon}} \mathrm{d}t$.

Thus, for any $m_0$, we only need to show that for all $t_1>0,\epsilon>0$ and $\tilde{\nu}>0$, $\int_{t_1}^\infty\frac{\left(\log(t)\right)^{\tilde{\nu}}}{t^{1+\epsilon}} \mathrm{d}t<\infty.$

Let us now look at $\int_{t_1}^\infty\frac{\left(\log(t)\right)^{\tilde{\nu}}}{t^{1+\epsilon}} \mathrm{d}t$. using $u=\left(\log(t)\right)^{\tilde{\nu}}$ and $\mathrm{d}v=\frac{1}{t^{1+\epsilon}}$, we have $\mathrm{d}u=\tilde{\nu}t^{-1}\left(\log(t)\right)^{\tilde{\nu}-1}$ and $v=-\frac{1}{\epsilon t^{\epsilon}}$. Using integration by parts, $\int u\mathrm{d}v=uv-\int v\mathrm{d}u$, we have  
\[\int\frac{\left(\log(t)\right)^{\tilde{\nu}}}{t^{1+\epsilon}} \mathrm{d}t=-\frac{\left(\log(t)\right)^{\tilde{\nu}}}{\epsilon t^{\epsilon}}+\frac{\bar{\nu}}{\epsilon}\int \frac{\left(\log(t)\right)^{\tilde{\nu}-1}}{t^{1+\epsilon}} \mathrm{d}t
\]

Recursing the above equation $K$ times such that $\tilde{\nu}-K<0$, we have positive constants $c_0,c_1,\ldots c_K>0$ independent of $t$, such that
\begin{align}\label{eq:int-byparts}
\int_{t_1}^\infty\frac{\left(\log(t)\right)^{\tilde{\nu}}}{t^{1+\epsilon}} \mathrm{d}t&=\left[-\sum_{k=0}^{K-1}\frac{c_k\left(\log(t)\right)^{\tilde{\nu}-k}}{\epsilon t^{\epsilon}}\right]_{t=t_1}^\infty+c_K\int_{t=t_1}^\infty \frac{\left(\log(t)\right)^{\tilde{\nu}-K}}{t^{1+\epsilon}} \mathrm{d}t\nonumber\\
&\overset{(1)}=\sum_{k=0}^{K-1}\frac{c_k\left(\log(t_1)\right)^{\tilde{\nu}-k}}{\epsilon t_1^{\epsilon}}+c_K\int_{t=t_1}^\infty \frac{\left(\log(t)\right)^{\tilde{\nu}-K}}{t^{1+\epsilon}} \mathrm{d}t\nonumber\\
&\overset{(2)}\le\sum_{k=0}^{K-1}\frac{c_k\left(\log(t_1)\right)^{\tilde{\nu}-k}}{\epsilon t_1^{\epsilon}}+c_K\int_{t=t_1}^\infty \frac{1}{t^{1+\epsilon}}\overset{(3)}\nonumber\\
&=\sum_{k=0}^{K-1}\frac{c_k\left(\log(t_1)\right)^{\tilde{\nu}-k}}{\epsilon t_1^{\epsilon}}y+\frac{c_K}{\epsilon t_1^\epsilon}<\infty
\end{align}
where $(1)$ follows as $\sum_{k=0}^{K-1}\frac{c_k\left(\log(t)\right)^{\tilde{\nu}-k}}{\epsilon t^{\epsilon}}\overset{t\to\infty}\rightarrow 0$, $(2)$ follows as $K$ is chosen such that $\tilde{\nu}-K<0$ and hence for all $t>0$, $\left(\log(t)\right)^{\tilde{\nu}-K}<1$. This completes the proof of the lemma. 
\end{proof}
\Correlation*
\begin{proof}
Recall that we defined
\begin{align}
\rvec(t) & =\wvec(t)-\mathbf{q}\left(t\right)\label{eq: define r(t) degenrate 2}
\end{align}
where
\begin{align}
\mathbf{q}\left(t\right) & =\sum_{m=1}^{M}\left[\hat{\mathbf{w}}_{m}\log^{\circ m}\left(t\right)+\mathbf{h}_{m}\left(t\right)\right]\,.\label{eq: q(t)}\\
\mathbf{h}_{m}\left(t\right) & =\wtilde_{m}+\sum_{k=1}^{m-1}\frac{\check{\mathbf{w}}_{k,m}}{\prod_{r=k}^{m-1}\log^{\circ r}\left(t\right)}\label{eq: h_m}
\end{align}
with $\hat{\mathbf{w}}_{m}$, $\tilde{\mathbf{w}}_{m}$ and $\wcheck_{k,m}$
defined in eqs. \ref{eq:hatwm}, \ref{eq: w tilde-1} and \ref{eq: w check},
respectively. We note that
\begin{equation}
\left\Vert \mathbf{q}\left(t+1\right)-\mathbf{q}\left(t\right)-\dot{\mathbf{q}}\left(t\right)\right\Vert \leq C_{q}t^{-2}\in L_1\label{eq: q second derivative}
\end{equation}
where
\begin{equation}
\dot{\mathbf{q}}\left(t\right)=\sum_{m=1}^{M}\hat{\mathbf{w}}_{m}\frac{1}{t\prod_{r=1}^{m-1}\log^{\circ r}\left(t\right)}+\dot{\mathbf{h}}_{m}\left(t\right)\,.\label{eq: q dot}
\end{equation}
Additionally, we define $C_{h},C_{h}^{\prime}$ so that 
\begin{equation}
\left\Vert \mathbf{h}_{m}\left(t\right)\right\Vert \leq\left\Vert \wtilde_{m}\right\Vert +\sum_{k=1}^{m}\left\Vert \check{\mathbf{w}}_{k,m}\right\Vert \leq C_{h}\,\label{eq: h bound}
\end{equation}
and 
\begin{equation}
\left\Vert \dot{\mathbf{h}}_{m}\left(t\right)\right\Vert \leq\frac{C_{h}^{\prime}}{t\left(\prod_{r=1}^{m-2}\log^{\circ r}\left(t\right)\right)\left(\log^{\circ (m-1)}\left(t\right)\right)^{2}}\in L_1\,.\label{eq: h dot bound}
\end{equation}
We wish to calculate 
\begin{align}
 & \left(\rvec(t+1)-\rvec(t)\right)^{\top}\rvec(t)\nonumber \\
\overset{\left(1\right)}{=} & \left[\wvec(t+1)-\mathbf{w}\left(t\right)-\left[\mathbf{q}\left(t+1\right)-\mathbf{q}\left(t\right)\right]\right]^{\top}\mathbf{r}\left(t\right)\nonumber \\
\overset{\left(2\right)}{=} & \left[-\eta\derL-\dot{\mathbf{q}}\left(t\right)\right]^{\top}\mathbf{r}\left(t\right)-\left[\mathbf{q}\left(t+1\right)-\mathbf{q}\left(t\right)-\dot{\mathbf{q}}\left(t\right)\right]^{\top}\mathbf{r}\left(t\right)\label{eq: r correlation}
\end{align}
where in $\left(1\right)$ we used eq. \ref{eq: define r(t) degenrate 2}
and in $\left(2\right)$ we used the definition of GD in eq. \ref{eq: gradient descent linear}.
We can bound the second term using Cauchy-Shwartz inequality and eq.
\ref{eq: q second derivative}:
\[
\left[\mathbf{q}\left(t+1\right)-\mathbf{q}\left(t\right)-\dot{\mathbf{q}}\left(t\right)\right]^{\top}\mathbf{r}\left(t\right)\leq\left\Vert \mathbf{q}\left(t+1\right)-\mathbf{q}\left(t\right)-\dot{\mathbf{q}}\left(t\right)\right\Vert \left\Vert \mathbf{r}\left(t\right)\right\Vert \leq C_{q}t^{-2}\left\Vert \mathbf{r}\left(t\right)\right\Vert \,.
\]
Next, we examine the second term in eq. \ref{eq: r correlation}
\begin{align}
 & \left[-\eta\derL-\dot{\mathbf{q}}\left(t\right)\right]^{\top}\mathbf{r}\left(t\right)\nonumber \\
= & \left[-\eta\sum_{n=1}^{N}\ell'(\wvec(t)^{\top}\xn)\,\mathbf{x}_{n}-\dot{\mathbf{q}}\left(t\right)\right]^{\top}\mathbf{r}\left(t\right)\nonumber \\
\overset{(1)}= & -\sum_{m=1}^{M}\dot{\mathbf{h}}_{m}\left(t\right)^{\top}\mathbf{r}\left(t\right)-\eta\sum_{m=1}^{M}\sum_{n\in\set_{m}^{+}}\ell'(\wvec(t)^{\top}\xn)\,\mathbf{x}_{n}^{\top}\mathbf{r}\left(t\right)\nonumber \\
+ & \left[\eta\sum_{m=1}^{M}\sum_{n\in\set_{m}}-\ell'(\wvec(t)^{\top}\xn)\,\mathbf{x}_{n}-\sum_{m=1}^{M}\hat{\mathbf{w}}_{m}\frac{1}{t\prod_{r=1}^{m-1}\log^{\circ r}\left(t\right)}\right]^{\top}\mathbf{r}\left(t\right),\label{eq: r correlation 2}
\end{align}
where in $(1)$ recall from eq. \ref{eq:sm} that $\set_m,\set_m^+$ are mutually exclusive and $\cup_{m=1}^M \set_m \cup \set_m^+=[N]$.

Next we upper bound the three terms in eq. \ref{eq: r correlation 2}.

To bound the first term in eq. \ref{eq: r correlation 2} we use Cauchy-Shartz,
and eq. $\ref{eq: h dot bound}$.
\[
\sum_{m=1}^{M}\dot{\mathbf{h}}_{m}\left(t\right)^{\top}\mathbf{r}\left(t\right)\leq\sum_{m=1}^{M}\left\Vert \dot{\mathbf{h}}_{m}\left(t\right)\right\Vert \left\Vert \mathbf{r}\left(t\right)\right\Vert \leq\frac{MC_{h}^{\prime}}{t\left(\prod_{r=1}^{m-2}\log^{\circ r}\left(t\right)\right)\left(\log^{\circ (m-1)}\left(t\right)\right)^{2}}\left\Vert \mathbf{r}\left(t\right)\right\Vert 
\]

In bounding the second term in eq. \ref{eq: r correlation 2}, note that for tight exponential tail loss, since $\w\left(t\right)^{\top}\mathbf{x}_{n}\to\infty$, for large enough $t_0$, we have $-\ell'(\w\left(t\right)^{\top}\mathbf{x}_{n})\le (1+\exp(-\mu_+\w\left(t\right)^{\top}\mathbf{x}_{n}))\exp(-\w\left(t\right)^{\top}\mathbf{x}_{n})\le 2\exp(-\w\left(t\right)^{\top}\mathbf{x}_{n})$ for all $t>t_0$. 
The first term in eq. \ref{eq: r correlation 2} can be bounded by the following set of inequalities, for $t>t_0$, 
\begin{align}
 & \eta\sum_{m=1}^{M}\sum_{n\in\set_{m}^{+}}-\ell'(\w\left(t\right)^{\top}\mathbf{x}_{n})\,\mathbf{x}_{n}^{\top}\mathbf{r}\left(t\right)\le \eta\sum_{m=1}^{M}\sum_{n\in\set_{m}^{+}:\,\mathbf{x}_{n}^{\top}\mathbf{r}\left(t\right)\geq0}-\ell'(\w\left(t\right)^{\top}\mathbf{x}_{n})\,\mathbf{x}_{n}^{\top}\mathbf{r}\left(t\right)\nonumber \\
\overset{\left(1\right)}{\leq} & 2\eta\sum_{m=1}^{M}\sum_{n\in\set_{m}^{+}:\,\mathbf{x}_{n}^{\top}\mathbf{r}\left(t\right)\geq0}\exp\left(-\sum_{l=1}^{M}\left[\hat{\mathbf{w}}_{l}^{\top}\x_{n}\log^{\circ l}\left(t\right)+\mathbf{x}_{n}^{\top}\mathbf{h}_{l}(t)\right]-\mathbf{x}_{n}^{\top}\mathbf{r}\left(t\right)\right)\mathbf{x}_{n}^{\top}\mathbf{r}\left(t\right)\nonumber \\
\overset{\left(2\right)}{\leq} & 2\eta\sum_{m=1}^{M}\sum_{n\in\set_{m}^{+}:\,\mathbf{x}_{n}^{\top}\mathbf{r}\left(t\right)\geq0}\exp\left(-\sum_{l=1}^{M}\left[\hat{\mathbf{w}}_{l}^{\top}\x_{n}\log^{\circ l}\left(t\right)+\mathbf{x}_{n}^{\top}\mathbf{h}_{l}\left(t\right)\right]\right)\nonumber \\
\overset{\left(3\right)}{\leq} & 2\eta\sum_{m=1}^M\left|\set_{m}^{+}\max_{n\in \set_m^+}\right|\exp\left(M\left\Vert \mathbf{x}_{n}\right\Vert C_{h}\right)\exp\left(-\sum_{l=1}^{M}\hat{\mathbf{w}}_{l}^{\top}\x_{n}\log^{\circ l}\left(t\right)\right)\nonumber \\
\overset{\left(4\right)}{\leq} &\left\{\begin{array}{cc} \sum_{m=1}^M\frac{2\eta\left|\set_{m}^{+}\right|\exp\left(M\max_{n\in \set_m^+}\left\Vert \mathbf{x}_{n}\right\Vert C_{h}\right)}{t\left(\prod_{k=1}^{m-1}\log^{\circ k}\left(t\right)\right)\left(\log^{\circ m-1}\left(t\right)\right)^{\theta_{m}}\left(\prod_{k=m}^{M-1}\left(\log^{\circ m}\left(t\right)\right)^{\mathbf{\hat{w}}_{k}^{\top}\x_{n}}\right)} & \text{if }M>1\\
\frac{2\eta\left|\set_{1}^{+}\right|\exp\left(\max_{n}\left\Vert \mathbf{x}_{n}\right\Vert C_{h}\right)}{t^{\theta_{1}}}& \text{if }M=1
\end{array}\right.\in L_1.\label{eq: bound on non-SV gradient-1}
\end{align}
where in $\left(1\right)$ we used eqs. \ref{eq: define r(t) degenrate 2}
and \ref{eq: q(t)}, in $\left(2\right)$ we used that $\forall x: xe^{-x}\leq1$
and $\mathbf{x}_{n}^{\top}\mathbf{r}\left(t\right)\geq0$, $\left(3\right)$ we
used eq. \ref{eq: h bound} and in $\left(4\right)$ we denoted $\theta_{m}=\min_{n\in\set_{m}^{+}}\what_{m}^{\top}\x_{n}>1$ and the last line is integrable based on Lemma~\ref{lem:int}. 

Next, we bound the last term in eq. \ref{eq: r correlation 2}. For exponential tailed losses (Assumption \ref{assum: exponential tail}), since $\wvec(t)^\top \xn\to\infty$, we have positive constants $\mu_{-},\mu_{+}>0$,
$t_{-}$ and $t_{+}$ such that $\forall n$ 
\begin{align*}
 & \!\!\forall t>t_{+}:-\ell^{\prime}\left(\mathbf{w}\left(t\right)^{\top}\mathbf{x}_{n}\right)\leq\left(1+\exp\left(-\mu_{+}\mathbf{w}\left(t\right)^{\top}\mathbf{x}_{n}\right)\right)\exp\left(-\mathbf{w}\left(t\right)^{\top}\mathbf{x}_{n}\right)\\
 & \!\!\forall t>t_{-}:-\ell^{\prime}\left(\mathbf{w}\left(t\right)^{\top}\mathbf{x}_{n}\right)\geq\left(1-\exp\left(-\mu_{-}\mathbf{w}\left(t\right)^{\top}\mathbf{x}_{n}\right)\right)\exp\left(-\mathbf{w}\left(t\right)^{\top}\mathbf{x}_{n}\right)
\end{align*}

We define $\gamma_n(t)$ as \begin{equation}
\gamma_n(t)=\left\{\begin{array}{ll}(1+\exp(-\mu_+\wvec(t)^\top\xn) & \text{if }\mathbf{r}\left(t\right)^\top\xn\ge0\\ (1-\exp(-\mu_-\wvec(t)^\top\xn) & \text{if }\mathbf{r}\left(t\right)^\top\xn<0\end{array}\right..
\label{eq:gamma}
\end{equation}

This implies $t>\max{(t_+,t_-)}$,  $-\ell'(\wvec(t)^\top\xn)\;\xn^\top\mathbf{r}\left(t\right)\le \gamma_n(t)\exp\left(-\w^{\top}\left(t\right)\mathbf{x}_{n}\right)\mathbf{x}_{n}$. \newpage
From this result, we have the following set of inequalities:
\begin{align}
&\eta\sum_{m=1}^{M}\sum_{n\in\set_{m}}-\ell'(\wvec(t)^{\top}\xn)\,\mathbf{x}_{n}^{\top}\mathbf{r}\left(t\right)
\le  \eta\sum_{m=1}^{M}\sum_{n\in\set_{m}}\gamma_n(t)\exp\left(-\wvec\left(t\right)^\top\mathbf{x}_{n}\right)\mathbf{x}_{n}^{\top}\mathbf{r}\left(t\right)\nonumber \\
\overset{\left(1\right)}{=} & \eta\sum_{m=1}^{M}\sum_{n\in\set_{m}}\gamma_n(t)\exp\left(-\sum_{l=1}^{M}\left[\hat{\mathbf{w}}_{l}^{\top}\x_{n}\log^{\circ l}\left(t\right)+\mathbf{x}_{n}^{\top}\wtilde_{l}+\sum_{k=1}^{l-1}\frac{\mathbf{x}_{n}^{\top}\check{\mathbf{w}}_{k,l}}{\prod_{r=k}^{l-1}\log^{\circ r}\left(t\right)}\right]-\mathbf{x}_{n}^{\top}\mathbf{r}\left(t\right)\right)\mathbf{x}_{n}^{\top}\mathbf{r}\left(t\right)\nonumber \\
\overset{\left(2\right)}{=} & \sum_{m=1}^{M}\sum_{n\in\set_{m}}\frac{\eta\gamma_n(t)\exp\left(-\mathbf{x}_{n}^{\top}\wtilde\right)\exp\left(-\mathbf{x}_{n}^{\top}\mathbf{r}\left(t\right)\right)\mathbf{x}_{n}^{\top}\mathbf{r}\left(t\right)}{t\prod_{r=1}^{m-1}\log^{\circ r}\left(t\right)}\exp\left(-\sum_{k=1}^{m}\sum_{l=k+1}^{M}\frac{\mathbf{x}_{n}^{\top}\check{\mathbf{w}}_{k,l}}{\prod_{r=k}^{l-1}\log^{\circ r}\left(t\right)}\right)\nonumber \\
\overset{\left(3\right)}{=} & \sum_{m=1}^{M}\sum_{n\in\set_{m}}\frac{\eta\gamma_n(t)\exp\left(-\mathbf{x}_{n}^{\top}\wtilde\right)\exp\left(-\mathbf{x}_{n}^{\top}\mathbf{r}\left(t\right)\right)\mathbf{x}_{n}^{\top}\mathbf{r}\left(t\right)}{t\prod_{r=1}^{m-1}\log^{\circ r}\left(t\right)}\exp\left(-\sum_{l=m+1}^{M}\frac{\mathbf{x}_{n}^{\top}\check{\mathbf{w}}_{m,l}}{\prod_{r=m}^{l-1}\log^{\circ r}\left(t\right)}\right)\psi_{m}\left(t\right)\nonumber \\
\leq & \sum_{m=1}^{M}\sum_{n\in\set_{m}}\frac{\eta\gamma_n(t)\exp\left(-\mathbf{x}_{n}^{\top}\wtilde\right)\exp\left(-\mathbf{x}_{n}^{\top}\mathbf{r}\left(t\right)\right)\mathbf{x}_{n}^{\top}\mathbf{r}\left(t\right)}{t\prod_{r=1}^{m-1}\log^{\circ r}\left(t\right)}\psi_{m}\left(t\right)\left[\left(1-\sum_{l=m}^{M-1}\frac{\mathbf{x}_{n}^{\top}\check{\mathbf{w}}_{m,l+1}}{\prod_{r=m}^{l}\log^{\circ r}\left(t\right)}\right)\right.\nonumber\\
&\quad+\left.\exp\left(-\sum_{l=m}^{M-1}\frac{\mathbf{x}_{n}^{\top}\check{\mathbf{w}}_{m,l+1}}{\prod_{r=m}^{l}\log^{\circ r}\left(t\right)}\right)-\left(1-\sum_{l=m}^{M-1}\frac{\mathbf{x}_{n}^{\top}\check{\mathbf{w}}_{m,l+1}}{\prod_{r=m}^{l}\log^{\circ r}\left(t\right)}\right)\right]\label{eq: temp  main term, degenerate}
\end{align}
where in $\left(1\right)$ we used eqs. \ref{eq: define r(t) degenrate 2}
and \ref{eq: q(t)}, and in $\left(2\right)$ we used $\mathbf{P}_{k-1}\check{\mathbf{w}}_{k,m}=0$
from eq. \ref{eq: w check constraints} (so $\mathbf{x}_{n}^{\top}\check{\mathbf{w}}_{k,l}=0$
if $m<k$) and in $\left(3\right)$ defined 
\begin{equation}
\psi_{m}\left(t\right)=\exp\left(-\sum_{k=1}^{m-1}\sum_{l=k+1}^{M}\frac{\mathbf{x}_{n}^{\top}\check{\mathbf{w}}_{k,l}}{\prod_{r=k}^{l-1}\log^{\circ r}\left(t\right)}\right)\,.
\label{eq:psi}
\end{equation}
Note $\exists t_{\psi}$ such that $\forall t>t_{\psi}$, we can bound
$\psi_{m}\left(t\right)$ by
\begin{equation}
\exp\left(\frac{-M\max_{n}\left\Vert \mathbf{x}_{n}\right\Vert C_{h}}{\log^{\circ\left(m-1\right)}\left(t\right)}\right)\leq\psi_{m}\left(t\right)\leq1\,.\label{eq: psi bound}
\end{equation}

Thus, the third term in \ref{eq: r correlation 2} is given by {\small
\begin{flalign}
&\eta\sum_{m=1}^{M}\sum_{n\in\set_{m}}-\ell'(\wvec(t)^{\top}\xn)\,\mathbf{x}_{n}^{\top}\mathbf{r}\left(t\right)
-\sum_{m=1}^{M}\frac{\hat{\mathbf{w}}_{m}^{\top}\mathbf{r}\left(t\right)}{t\prod_{r=1}^{m-1}\log^{\circ r}\left(t\right)}\nonumber \\
\overset{(1)}\leq  &\sum_{m=1}^{M}\sum_{n\in\set_{m}}\frac{\eta\gamma_n(t)\exp\left(-\mathbf{x}_{n}^{\top}\wtilde\right)\exp\left(-\mathbf{x}_{n}^{\top}\mathbf{r}\left(t\right)\right)\mathbf{x}_{n}^{\top}\mathbf{r}\left(t\right)}{t\prod_{r=1}^{m-1}\log^{\circ r}\left(t\right)}\psi_{m}\left(t\right)\left[
\exp\left(-\sum_{l=m}^{M-1}\frac{\mathbf{x}_{n}^{\top}\check{\mathbf{w}}_{m,l+1}}{\prod_{r=m}^{l}\log^{\circ r}\left(t\right)}\right)\right.\nonumber\\
&\quad\quad\quad\left.-\left(1-\sum_{l=m}^{M-1}\frac{\mathbf{x}_{n}^{\top}\check{\mathbf{w}}_{m,l+1}}{\prod_{r=m}^{l}\log^{\circ r}\left(t\right)}\right)\right]\nonumber \\
&+\sum_{m=1}^{M}\Bigg[\sum_{n\in\set_{m}}\frac{\eta\gamma_n(t)\exp\left(-\mathbf{x}_{n}^{\top}\wtilde\right)\exp\left(-\mathbf{x}_{n}^{\top}\mathbf{r}\left(t\right)\right)\mathbf{x}_{n}^{\top}\mathbf{r}\left(t\right)}{t\prod_{r=1}^{m-1}\log^{\circ r}\left(t\right)}\psi_{m}\left(t\right)\left(1-\sum_{l=m}^{M-1}\frac{\mathbf{x}_{n}^{\top}\check{\mathbf{w}}_{m,l+1}}{\prod_{r=m}^{l}\log^{\circ r}\left(t\right)}\right)\nonumber\\
&-\frac{\mathbf{r}(t)^{\top}\hat{\mathbf{w}}_{m}}{t\prod_{r=1}^{m-1}\log^{\circ r}\left(t\right)}\Bigg],%\nonumber\\
%&\quad\quad\left.-\left(\sum_{n\in\set_{m}}\frac{\eta\exp\left(-\mathbf{x}_{n}^{\top}\wtilde\right)\mathbf{x}_{n}^{\top}\mathbf{r}\left(t\right)}{t\prod_{r=1}^{m-1}\log^{\circ r}\left(t\right)}-\sum_{k=1}^{m-1}\sum_{n\in\set_k}\frac{\eta\exp\left(-\mathbf{x}_{n}^{\top}\wtilde\right)\mathbf{x}_{n}^{\top}\mathbf{r}\left(t\right)\mathbf{x}_{n}^{\top}\check{\mathbf{w}}_{k,m}}{\prod_{r=1}^{m-1}t\log^{\circ r}\left(t\right)}\right)\right]
\label{eq: main term, degenerate}
\end{flalign}}
where $(1)$ follows from the bound in eq.~\ref{eq: temp  main term, degenerate}. 
\newpage
We examine the first term in eq. \ref{eq: main term, degenerate} {\small
\begin{align*}
\sum_{m=1}^{M}\sum_{n\in\set_{m}}&\frac{\eta\gamma_n(t)\exp\left(-\mathbf{x}_{n}^{\top}\wtilde\right)\exp\left(-\mathbf{x}_{n}^{\top}\mathbf{r}\left(t\right)\right)\mathbf{x}_{n}^{\top}\mathbf{r}\left(t\right)}{t\prod_{r=1}^{m-1}\log^{\circ r}\left(t\right)}\psi_{m}\left(t\right)\\
\cdot&\Bigg[
\exp\left(-\sum_{l=m}^{M-1}\frac{\mathbf{x}_{n}^{\top}\check{\mathbf{w}}_{m,l+1}}{\prod_{r=m}^{l}\log^{\circ r}\left(t\right)}\right)-\left(1-\sum_{l=m}^{M-1}\frac{\mathbf{x}_{n}^{\top}\check{\mathbf{w}}_{m,l+1}}{\prod_{r=m}^{l}\log^{\circ r}\left(t\right)}\right)\Bigg]\end{align*}}
$\forall t>t_{1}>t_{\psi}$, where we will determine $t_{1}$ later. We have the following for all $m\in[M]$
{
\begin{align}
&\sum_{n\in\set_{m}}\frac{\eta\gamma_n(t)\exp\left(-\mathbf{x}_{n}^{\top}\wtilde\right)\exp\left(-\mathbf{x}_{n}^{\top}\mathbf{r}\left(t\right)\right)\mathbf{x}_{n}^{\top}\mathbf{r}\left(t\right)}{t\prod_{r=1}^{m-1}\log^{\circ r}\left(t\right)}\psi_{m}\left(t\right)\nonumber\\
&\cdot\Bigg[\exp\left(-\sum_{l=m}^{M-1}\frac{\mathbf{x}_{n}^{\top}\check{\mathbf{w}}_{m,l+1}}{\prod_{r=m}^{l}\log^{\circ r}\left(t\right)}\right)\nonumber-\left(1-\sum_{l=m}^{M-1}\frac{\mathbf{x}_{n}^{\top}\check{\mathbf{w}}_{m,l+1}}{\prod_{r=m}^{l}\log^{\circ r}\left(t\right)}\right)\Bigg]\nonumber \\
  &\overset{\left(1\right)}{\leq}  \sum_{n\in\set_{m}: \atop \mathbf{x}_{n}^{\top}\mathbf{r}\left(t\right)\geq0}\frac{\eta\gamma_n(t)\exp\left(-\mathbf{x}_{n}^{\top}\wtilde\right)\psi_{m}\left(t\right)}{t\prod_{r=1}^{m-1}\log^{\circ r}\left(t\right)}\left[\exp\left(-\sum_{l=m}^{M-1}\frac{\mathbf{x}_{n}^{\top}\check{\mathbf{w}}_{m,l+1}}{\prod_{r=m}^{l}\log^{\circ r}\left(t\right)}\right)-\left(1-\sum_{l=m}^{M-1}\frac{\mathbf{x}_{n}^{\top}\check{\mathbf{w}}_{m,l+1}}{\prod_{r=m}^{l}\log^{\circ r}\left(t\right)}\right)\right]\nonumber \\
%\overset{\left(2\right)}{\leq} & \frac{\eta}{2}\sum_{m=1}^{M}\frac{1}{t\prod_{r=1}^{m-1}\log^{\circ r}\left(t\right)}\sum_{n\in\set_{m}:\,\mathbf{x}_{n}^{\top}\mathbf{r}\left(t\right)\geq0}\exp\left(-\mathbf{x}_{n}^{\top}\wtilde\right)\psi_{m}\left(t\right)\left[\exp\left(-\sum_{l=m}^{M}\frac{\mathbf{x}_{n}^{\top}\check{\mathbf{w}}_{m,l}}{\prod_{r=m}^{l}\log^{\circ r}\left(t\right)}\right)-\left(1-\sum_{l=m}^{M}\frac{\mathbf{x}_{n}^{\top}\check{\mathbf{w}}_{m,l}}{\prod_{r=m}^{l}\log^{\circ r}\left(t\right)}\right)\right]\\
&\overset{\left(2\right)}{\leq}  \sum_{n\in\set_{m}: \atop \mathbf{x}_{n}^{\top}\mathbf{r}\left(t\right)\geq0}\frac{\eta\gamma_n(t)\exp\left(-\mathbf{x}_{n}^{\top}\wtilde\right)\left(t\right)}{t\prod_{r=1}^{m-1}\log^{\circ r}\psi_m\left(t\right)}\left(\sum_{l=m}^{M-1}\frac{\mathbf{x}_{n}^{\top}\check{\mathbf{w}}_{m,l+1}}{\prod_{r=m}^{l}\log^{\circ r}\left(t\right)}\right)^{2}\in L_1\,,
\end{align}}
where we set $t_{1}>0$
such that $\forall t>t_{1}$ the term in the square bracket is positive and \[ \sum_{l=m}^{M-1}\frac{\mathbf{x}_{n}^{\top}\check{\mathbf{w}}_{m,l+1}}{\prod_{r=m}^{l}\log^{\circ r}\left(t\right)}>-1 \, ,\] in $\left(1\right)$ we used that since $e^{-x}\geq1-x$, and also from using $e^{-x}x\leq1$ and in $\left(2\right)$
we use that $\forall x\geq -1$ we have that $e^{-x}\le 1-x+x^{2}$ and $\psi_{m}\left(t\right)\le 1$ from eq. \ref{eq: psi bound}.
\newpage
We examine the second term in eq. \ref{eq: main term, degenerate} using the decomposition of $\hat{\mathbf{w}}_{m}$ from eq. \ref{eq: w_hat full decomposition}  {\small
\begin{align}
&\sum_{m=1}^{M}\left[\sum_{n\in\set_{m}}\frac{\eta\gamma_n(t)\exp\left(-\mathbf{x}_{n}^{\top}\wtilde\right)\exp\left(-\mathbf{x}_{n}^{\top}\mathbf{r}\left(t\right)\right)\mathbf{x}_{n}^{\top}\mathbf{r}\left(t\right)}{t\prod_{r=1}^{m-1}\log^{\circ r}\left(t\right)}\psi_{m}\left(t\right)\left(1-\sum_{l=m}^{M-1}\frac{\mathbf{x}_{n}^{\top}\check{\mathbf{w}}_{m,l+1}}{\prod_{r=m}^{l}\log^{\circ r}\left(t\right)}\right)-\frac{\mathbf{x}_{n}^{\top}\hat{\mathbf{w}}_{m}}{t\prod_{r=1}^{m-1}\log^{\circ r}\left(t\right)}\right]\nonumber \\
&\overset{(1)}= %\sum_{m=1}^{M}\left[\sum_{n\in\set_{m}}\frac{\eta\gamma_n(t)\exp\left(-\mathbf{x}_{n}^{\top}\wtilde\right)\exp\left(-\mathbf{x}_{n}^{\top}\mathbf{r}\left(t\right)\right)\mathbf{x}_{n}^{\top}\mathbf{r}\left(t\right)}{t\prod_{r=1}^{m-1}\log^{\circ r}\left(t\right)}\psi_{m}\left(t\right)\left(1-\sum_{l=m}^{M-1}\frac{\mathbf{x}_{n}^{\top}\check{\mathbf{w}}_{m,l+1}}{\prod_{r=m}^{l}\log^{\circ r}\left(t\right)}\right)\right.\nonumber\\
%&\quad\quad\left.-\left(\sum_{n\in\set_{m}}\frac{\eta\exp\left(-\mathbf{x}_{n}^{\top}\wtilde\right)\mathbf{x}_{n}^{\top}\mathbf{r}\left(t\right)}{t\prod_{r=1}^{m-1}\log^{\circ r}\left(t\right)}-\sum_{k=1}^{m-1}\sum_{n\in\set_k}\frac{\eta\exp\left(-\mathbf{x}_{n}^{\top}\wtilde\right)\mathbf{x}_{n}^{\top}\mathbf{r}\left(t\right)\mathbf{x}_{n}^{\top}\check{\mathbf{w}}_{k,m}}{\prod_{r=1}^{m-1}t\log^{\circ r}\left(t\right)}\right)\right]\nonumber\\
%&=
\sum_{m=1}^{M}\sum_{n\in\set_{m}}\frac{\eta\exp\left(-\mathbf{x}_{n}^{\top}\wtilde\right)\mathbf{x}_{n}^{\top}\mathbf{r}\left(t\right)}{t\prod_{r=1}^{m-1}\log^{\circ r}\left(t\right)}\left(\gamma_n(t)\exp\left(-\mathbf{x}_{n}^{\top}\mathbf{r}\left(t\right)\right)
\psi_{m}\left(t\right)-1\right)\nonumber\\
&\quad\quad-\sum_{m=1}^{M}\sum_{n\in\set_{m}}\frac{\eta\gamma_n(t)\exp\left(-\mathbf{x}_{n}^{\top}\wtilde\right)\exp\left(-\mathbf{x}_{n}^{\top}\mathbf{r}\left(t\right)\right)\mathbf{x}_{n}^{\top}\mathbf{r}\left(t\right)\psi_{m}(t)}{t\prod_{r=1}^{m-1}\log^{\circ r}\left(t\right)}\sum_{l=m}^{M-1}\frac{\mathbf{x}_{n}^{\top}\check{\mathbf{w}}_{m,l+1}}{\prod_{r=m}^{l}\log^{\circ r}\left(t\right)}\nonumber\\
&\quad\quad+\sum_{m=1}^{M}\sum_{k=1}^{m-1}\sum_{n\in\set_k}\frac{\eta\exp\left(-\mathbf{x}_{n}^{\top}\wtilde\right)\mathbf{x}_{n}^{\top}\mathbf{r}\left(t\right)\mathbf{x}_{n}^{\top}\check{\mathbf{w}}_{k,m}}{\prod_{r=1}^{m-1}t\log^{\circ r}\left(t\right)}\nonumber\\
&\overset{\left(2\right)}{=}\sum_{m=1}^{M}\sum_{n\in\set_{m}}\frac{\eta\exp\left(-\mathbf{x}_{n}^{\top}\wtilde\right)\mathbf{x}_{n}^{\top}\mathbf{r}\left(t\right)}{t\prod_{r=1}^{m-1}\log^{\circ r}\left(t\right)}\left(\gamma_n(t)\exp\left(-\mathbf{x}_{n}^{\top}\mathbf{r}\left(t\right)\right)
\psi_{m}\left(t\right)-1\right)\nonumber\\
&\quad\quad-\sum_{m=1}^{M}\sum_{n\in\set_{m}}\sum_{l=m}^{M-1}\frac{\eta\gamma_n(t)\exp\left(-\mathbf{x}_{n}^{\top}\wtilde\right)\exp\left(-\mathbf{x}_{n}^{\top}\mathbf{r}\left(t\right)\right)\mathbf{x}_{n}^{\top}\mathbf{r}\left(t\right)\psi_{m}(t)\mathbf{x}_{n}^{\top}\check{\mathbf{w}}_{m,l+1}}{t\prod_{r=1}^{l}\log^{\circ r}\left(t\right)}\nonumber\\
&\quad\quad+\sum_{k=1}^{M}\sum_{n\in\set_k}\sum_{m=k}^{M-1}\frac{\eta\exp\left(-\mathbf{x}_{n}^{\top}\wtilde\right)\mathbf{x}_{n}^{\top}\mathbf{r}\left(t\right)\mathbf{x}_{n}^{\top}\check{\mathbf{w}}_{k,m+1}}{\prod_{r=1}^{m}t\log^{\circ r}\left(t\right)}\nonumber\\
&\overset{\left(3\right)}{=} \sum_{m=1}^{M}\sum_{n\in\set_{m}}\left[\frac{1}{t\prod_{r=1}^{m-1}\log^{\circ r}\left(t\right)}-\sum_{k=m}^{M-1}\frac{\mathbf{x}_{n}^{\top}\check{\mathbf{w}}_{m,k+1}}{t\prod_{r=1}^{k}\log^{\circ r}\left(t\right)}\right]\eta\exp\left(-\mathbf{x}_{n}^{\top}\wtilde\right)\left(\gamma_n(t)\psi_{m}\left(t\right)\exp\left(-\mathbf{x}_{n}^{\top}\mathbf{r}\left(t\right)\right)-1\right)\mathbf{x}_{n}^{\top}\mathbf{r}\left(t\right)\nonumber\\
&:= \sum_{m=1}^{M}\sum_{n\in\set_{m}} \Gamma_{m,n}(t),\label{eq: the last term}
\end{align}}
where in $\left(1\right)$ we used eq. \ref{eq: w_hat full decomposition}, in $\left(2\right)$
we re-arranged the order of summation in the last term, and in $(3)$ we just use a change of variables.  

%We can easily verify that for all $m$ and $n\in\set_m$, $\exists t_{2}>t_{\psi}$ such that $\forall t>t_{2}$ the term in
%the square bracket in eq. \ref{eq: the last term} satisfies $\left|\sum_{k=m}^{M-1}\frac{\mathbf{x}_{n}^{\top}\check{\mathbf{w}}_{m,k+1}}{t\prod_{r=1}^{k}\log^{\circ r}\left(t\right)}\right|\le \frac{0.5}{t\prod_{r=1}^{m-1}\log^{\circ r}\left(t\right)}$. 

Next, we examine $\Gamma_{m,n}(t)$ for each $m$ and $n\in\set_{m}$ in eq.
\ref{eq: the last term}. Note that, $\exists t_{2}>t_{\psi}$ such that $\forall t>t_{2}$ we have \[\left|\sum_{k=m}^{M-1}\frac{\mathbf{x}_{n}^{\top}\check{\mathbf{w}}_{m,k+1}}{t\prod_{r=1}^{k}\log^{\circ r}\left(t\right)}\right|\le \frac{0.5}{t\prod_{r=1}^{m-1}\log^{\circ r}\left(t\right)}\, . \]
In this case, $\forall t>t_2$
\begin{align}
 \Gamma_{m,n}(t)
%\overset{\left(1\right)}{\leq}  \eta\left[\frac{1}{t\prod_{r=1}^{m-1}\log^{\circ r}\left(t\right)}-\sum_{k=m+1}^{M}\frac{\mathbf{x}_{n}^{\top}\check{\mathbf{w}}_{m,k-1}}{t\prod_{r=1}^{k-1}\log^{\circ r}\left(t\right)}\right]\exp\left(-\mathbf{x}_{n}^{\top}\wtilde\right)\left(\gamma_n(t)\exp\left(-\mathbf{x}_{n}^{\top}\mathbf{r}\left(t\right)\right)-1\right)\mathbf{x}_{n}^{\top}\mathbf{r}\left(t\right)\nonumber\\
\overset{\left(1\right)}{\leq} \eta\left[\frac{\kappa(n,t)}{t\prod_{r=1}^{m-1}\log^{\circ r}\left(t\right)}\right]\exp\left(-\mathbf{x}_{n}^{\top}\wtilde\right)\left(\gamma_n(t)\psi_m(t)\exp\left(-\mathbf{x}_{n}^{\top}\mathbf{r}\left(t\right)\right)-1\right)\mathbf{x}_{n}^{\top}\mathbf{r}\left(t\right),
%\leq & 0\,\forall t>t_{2}
\label{eq:case1}
\end{align}
where in $\left(1\right)$ follows from the definition of 
$t_{2}$, wherein \[\kappa_n(t)=\left\{\begin{array}{ll}
1.5&\text{if }\left(\gamma_n(t)\psi_m(t)\exp\left(-\mathbf{x}_{n}^{\top}\mathbf{r}\left(t\right)\right)-1\right)\mathbf{x}_{n}^{\top}\mathbf{r}\left(t\right)>0\\
0.5 & \text{if }\left(\gamma_n(t)\psi_m(t)\exp\left(-\mathbf{x}_{n}^{\top}\mathbf{r}\left(t\right)\right)-1\right)\mathbf{x}_{n}^{\top}\mathbf{r}\left(t\right)<0
\end{array}\right..\]

\begin{asparaenum}
\item First, if $\mathbf{x}_{n}^{\top}\mathbf{r}\left(t\right)>0$, then $\gamma_n(t)=(1+\exp(-\mu_+\wvec(t)^\top\xn))>0$. 

We further divide into two cases. In the following $C_0,C_1$ are some constants independent of $t$.
\begin{compactenum}
\item If $\left|\mathbf{x}_{n}^{\top}\mathbf{r}\left(t\right)\right|>C_0t^{-0.5\mu_+}$, then we have the following
\begin{align}
&\gamma_n(t)\psi_m(t)\exp\left(-\mathbf{x}_{n}^{\top}\mathbf{r}\left(t\right)\right) \nonumber \\
& \overset{(1)}\le\left(1+\exp\left(-\mu_+\sum_{l=1}^M\left[\hat{\mathbf{w}}_l^\top \xn\log^{\circ l}(t)+\mathbf{h}_{l}^\top\xn\right]\right)\right)\exp\left(-\mathbf{x}_{n}^{\top}\mathbf{r}\left(t\right)\right)\nonumber\\
&\overset{(2)}\le\left(1+\frac{\exp(\mu_+C_h\norm{\xn})}{\left(t\prod_{r=1}^{m-1}\log^{\circ r}\left(t\right)\right)^{\mu_+}}\right)\exp(-C_0t^{-0.5\mu_+})\nonumber\\
&\overset{(3)}\le \left(1+C_1t^{-\mu_+}\right)\left(1-C_0t^{-0.5\mu_+}+0.5C_0^2t^{-\mu_+}\right),\forall t>t_+'\nonumber\\
&\le 1-C_0t^{-0.5\mu_+}\left(1+C_1t^{-\mu_+}\right)+0.5C_0^2t^{-\mu_+}\left(1+C_1t^{-\mu_+}\right)\overset{(4)}\le 1,\forall t>t_+^{\prime\prime},
\label{eq.case1a}
\end{align}
where in $(1)$, we use $\psi_m(t)\le 1$ from eq. \ref{eq: psi bound} and using eq. \ref{eq: define r(t) degenrate 2}, in $(2)$ we used bound on $\mathbf{h}_{m}$ from eq. \ref{eq: h bound}, in $(3)$ for some large enough $t_+'>t_+$, we have $\frac{\exp(\mu_+C_h\norm{\xn})}{\left(\prod_{r=1}^{m-1}\log^{\circ r}\left(t\right)\right)^{\mu_+}}\le C_1$,  and for the second term we used the inequality $e^{-x}\le 1-x+0.5 x^2$ for $x>0$, and $(4)$ holds asymptotically for $t>t_+^{\prime\prime}$ for large enough $t_+^{\prime\prime}>t_+'$  as $C_0t^{-0.5\mu_+}$ converges slower than $0.5C_0^2t^{-\mu_+}$ to 0. 

Thus, using eq. \ref{eq.case1a}  in eq. \ref{eq:case1}, $\forall t>\max{(t_2,t_+^{\prime\prime})}$, we have
\begin{align*}
 \Gamma_{m,n}(t)
{\leq} \left[\frac{\eta\kappa(n,t)\exp\left(-\mathbf{x}_{n}^{\top}\wtilde\right)}{t\prod_{r=1}^{m-1}\log^{\circ r}\left(t\right)}\right]\left(\gamma_n(t)\psi_m(t)\exp\left(-\mathbf{x}_{n}^{\top}\mathbf{r}\left(t\right)\right)-1\right)\mathbf{x}_{n}^{\top}\mathbf{r}\left(t\right)\le0
\end{align*}

\item If $0<\mathbf{x}_{n}^{\top}\mathbf{r}\left(t\right)<C_0t^{-0.5\mu_+}$, then we have the following: $\psi_m(t)\le 1$ from eq. \ref{eq: psi bound}, 
 $\exp\left(-\mathbf{x}_{n}^{\top}\mathbf{r}\left(t\right)\right)\le1$ as $\mathbf{x}_{n}^{\top}\mathbf{r}\left(t\right)>0$,  and since $\wvec(t)^\top\xn\to\infty$, for large enough $t>t_+^{\prime\prime\prime}$, $\gamma_n(t)=\left(1+\exp\left(-\mu_+\wvec(t)^\top\xn\right)\right)\le2$
 
 This gives us, $\left(\gamma_n(t)\psi_m(t)\exp\left(-\mathbf{x}_{n}^{\top}\mathbf{r}\left(t\right)\right)-1\right)\mathbf{x}_{n}^{\top}\mathbf{r}\left(t\right)\le\mathbf{x}_{n}^{\top}\mathbf{r}\left(t\right)\le C_0t^{-0.5\mu_+}$, and using this in eq. \ref{eq:case1}, $\forall t>\max{(t_2,t_+')}$
\begin{align*}
 \Gamma_{m,n}(t)
{\leq} \left[\frac{\eta\kappa(n,t)\exp\left(-\mathbf{x}_{n}^{\top}\wtilde\right)}{t\prod_{r=1}^{m-1}\log^{\circ r}\left(t\right)}\right]C_0t^{-0.5\mu_+}\in L_1.
\end{align*}
\end{compactenum}
\item Second, if $\mathbf{x}_{n}^{\top}\mathbf{r}\left(t\right)\leq0$, then $\gamma_n(t)=(1-\exp(-\mu_-\wvec(t)^\top\xn))\in(0,1)$. We again  divide into following special cases. 
\begin{compactenum}
\item If $\left|\mathbf{x}_{n}^{\top}\mathbf{r}\left(t\right)\right|\le C_0\left(\log^{\circ (m-1)}(t)\right)^{-0.5\tilde{\mu}_-}$, where $\tilde\mu_-=\min{(\mu_-,1)}$, then we have 
\begin{align*}
 &\Gamma_{m,n}(t)\le \left[\frac{1.5 \eta\exp\left(-\mathbf{x}_{n}^{\top}\wtilde\right)}{t\prod_{r=1}^{m-1}\log^{\circ r}\left(t\right)}\right]\left(1-\gamma_n(t)\psi_{m}\left(t\right)\exp\left(-\mathbf{x}_{n}^{\top}\mathbf{r}\left(t\right)\right)\right)\left|\mathbf{x}_{n}^{\top}\mathbf{r}\left(t\right)\right|\\
\overset{(1)}\leq & \left[\frac{1.5 \eta\exp\left(-\mathbf{x}_{n}^{\top}\wtilde\right)}{t\prod_{r=1}^{m-2}\log^{\circ r}\left(t\right)}\right]C_0\left(\log^{\circ (m-1)}(t)\right)^{-1-0.5\tilde{\mu}_-}\in L_{1}.
\end{align*}
where in $(1)$ we used that $\left(1-\gamma_n(t)\psi_{m}\left(t\right)\exp\left(-\mathbf{x}_{n}^{\top}\mathbf{r}\left(t\right)\right)\right)<1$ and\\ $\left|\mathbf{x}_{n}^{\top}\mathbf{r}\left(t\right)\right|\le C_0\left(\log^{\circ (m-1)}(t)\right)^{-0.5\tilde{\mu}_-}$.
\item If  $\psi_{m}\left(t\right)\exp\left(-\mathbf{x}_{n}^{\top}\mathbf{r}\left(t\right)\right)<1,$
then, from eq. \ref{eq: psi bound} 
\begin{equation}\label{eq:tmp}
\frac{-M\max_{n}\left\Vert \mathbf{x}_{n}\right\Vert C_{h}}{\log^{\circ\left(m-1\right)}\left(t\right)}\leq\log\psi_{m}\left(t\right)<\mathbf{x}_{n}^{\top}\mathbf{r}\left(t\right).
\end{equation}
In this case, since $\gamma_n(t)=1-\exp(-\wvec(t)^\top\xn)<1$, we also have $\gamma_n(t)\psi_{m}\left(t\right)\exp\left(-\mathbf{x}_{n}^{\top}\mathbf{r}\left(t\right)\right)<1$, and hence $\left(\gamma_n(t)\psi_{m}\left(t\right)\exp\left(-\mathbf{x}_{n}^{\top}\mathbf{r}\left(t\right)\right)-1\right)\mathbf{x}_{n}^{\top}\mathbf{r}(t)>0$. Thus, $\forall t>t_{2}$, in \ref{eq:case1}, $\kappa_n(t)=1.5$, and we have 
\begin{align*}
 &\Gamma_{m,n}(t)\le \left[\frac{1.5 \eta\exp\left(-\mathbf{x}_{n}^{\top}\wtilde\right)}{t\prod_{r=1}^{m-1}\log^{\circ r}\left(t\right)}\right]\left(1-\gamma_n(t)\psi_{m}\left(t\right)\exp\left(-\mathbf{x}_{n}^{\top}\mathbf{r}\left(t\right)\right)\right)\left|\mathbf{x}_{n}^{\top}\mathbf{r}\left(t\right)\right|\\
\overset{(1)}\leq & \left[\frac{1.5 \eta\exp\left(-\mathbf{x}_{n}^{\top}\wtilde\right)}{t\prod_{r=1}^{m-1}\log^{\circ r}\left(t\right)}\right]\frac{M\max_{n}\left\Vert \mathbf{x}_{n}\right\Vert C_{h}}{\log^{\circ\left(m-1\right)}\left(t\right)}\leq \frac{C_2}{t\prod_{r=1}^{m-2}\log^{\circ r}\left(t\right)\left(\log^{\circ (m-1)}\left(t\right)\right)^{2}}\in L_{1}, 
\end{align*}
where $(1)$ follows from $\left(1-\gamma_n(t)\psi_{m}\left(t\right)\exp\left(-\mathbf{x}_{n}^{\top}\mathbf{r}\left(t\right)\right)\right)<1$ and the bound on $|\xn^\top\mathbf{r}(t)|=-\xn^\top\mathbf{r}(t)$ from eq. \ref{eq:tmp}. 
\item If $\psi_{m}\left(t\right)\exp\left(-\mathbf{x}_{n}^{\top}\mathbf{r}\left(t\right)\right)>1,$ and $\left|\mathbf{x}_{n}^{\top}\mathbf{r}\left(t\right)\right|>C_0\left(\log^{\circ (m-1)}(t)\right)^{-0.5\tilde{\mu}_-}$, where $\tilde{\mu}_=\min{(1,\mu_-)}$. 

Since, $\mathbf{x}_{n}^{\top}\wvec(t)\to\infty$ and $\psi_m(t)\to 1$ from eq.~\ref{eq:psi}, for large enough $t_-'>t_-$, we have  $\forall t>t_-'$, $\psi_m(t)>0.5$ and  $\gamma_n(t)=(1-\exp(-\mu_-\mathbf{x}_{n}^{\top}\wvec(t)))>0.5$. Let $\tau>\max{(4,t_-')}$ be an arbitrarily large constant. For all $t>\tau$, if $\exp\left(-\mathbf{x}_{n}^{\top}\mathbf{r}\left(t\right)\right)>\tau\ge4$, then $\gamma_n(t)\psi_m(t)\exp\left(-\mathbf{x}_{n}^{\top}\mathbf{r}\left(t\right)\right)> 0.25\tau\ge1$.

On the other hand, if there exists $t>\tau\ge4$, such that $\exp\left(-\mathbf{x}_{n}^{\top}\mathbf{r}\left(t\right)\right)<\tau$, then for some constants $C_1,C_2$ we have the following 
\begin{compactenum}[(i)]
\item  $\exp(-\mathbf{x}_{n}^{\top}\mathbf{r}\left(t\right))=\exp(|\mathbf{x}_{n}^{\top}\mathbf{r}\left(t\right)|)\ge \left(1+C_0\left(\log^{\circ (m-1)}(t)\right)^{-0.5\tilde{\mu}_-}\right)$,  since $e^x>1+x$ for all $x$,
\item  $\psi_m(t)\ge\exp\left(-C_1\left(\log^{\circ (m-1)}(t)\right)^{-1}\right)\ge \left(1-C_1\left(\log^{\circ (m-1)}(t)\right)^{-1}\right)$  from eq.~\ref{eq: psi bound}   and again using  $e^x>1+x$ for all $x$,
\item 
\begin{align*}
    &\gamma_n(t)=\left(1-\left[\frac{\exp(-\mathbf{h}_l(t)^\top \xn)\exp\left(-\mathbf{x}_{n}^{\top}\mathbf{r}\left(t\right)\right)}{t\prod_{r=1}^{m-1}\log^{\circ r}(t)}\right]^{\mu_-}\right)& \\
    &\ge \left(1-\left[\frac{\exp(-C_h\|x_n\|)\tau}{t\prod_{r=1}^{m-1}\log^{\circ r}(t)}\right]^{\mu_-}\right)\ge\left(1-\left(C_2\log^{\circ (m-1)}(t)\right)^{-\mu_-}\right), \forall t>t_-^{\prime\prime}
\end{align*} 
where the last inequality follows as for large enough $t_-^{\prime\prime}>t_-'$,  we have $\frac{\exp(-C_h\|x_n\|)\tau}{t\prod_{r=1}^{m-2}\log^{\circ r}(t)}\le C_2$.
\end{compactenum}

Using the above inequalities, we have {\small
\begin{align}
&\gamma_n(t)\psi_m(t)\exp\left(-\mathbf{x}_{n}^{\top}\mathbf{r}\left(t\right)\right)\nonumber\\
&\ge\left(1+C_0\left(\log^{\circ (m-1)}(t)\right)^{-0.5\tilde{\mu}_-}\right)\left(1-C_1\left(\log^{\circ (m-1)}(t)\right)^{-1}\right)\left(1-C_2\left(\log^{\circ (m-1)}(t)\right)^{-\mu_-}\right)\nonumber\\
&\overset{(1)}\ge 1+C_0\left(\log^{\circ (m-1)}(t)\right)^{-0.5\tilde{\mu}_-}-C_1\left(\log^{\circ (m-1)}(t)\right)^{-1}-C_2\left(\log^{\circ (m-1)}(t)\right)^{-{\mu}_-}\nonumber\\
&\quad\quad -C_0C_2\left(\log^{\circ (m-1)}(t)\right)^{-\mu_-0.5\tilde{\mu}_-}-C_0C_1\left(\log^{\circ (m-1)}(t)\right)^{-1-0.5\tilde{\mu}_-}\overset{(2)}\ge 1, \forall t>t_-^{\prime\prime\prime},
\label{eq:case2b}
\end{align}}
where in $(1)$ we dropped the other positive terms, and $(2)$ follows for large enough $t_-^{\prime\prime\prime}>t_-^{\prime\prime}$ as the $C_0\log\left(\log^{\circ (m-1)}(t)\right)^{-0.5\tilde{\mu}_-}$   converges to $0$ more slowly than the other negative terms.

Finally, using eq.~\ref{eq:case2b} in eq. \ref{eq:case1}, we have for all $t>\max{(t_2,\tau,t_\psi,t_-^{\prime\prime\prime})}$
\begin{align}
 \Gamma_{m,n}(t)
\le \left[\frac{\eta\kappa(n,t)\exp\left(-\mathbf{x}_{n}^{\top}\wtilde\right)}{t\prod_{r=1}^{m-1}\log^{\circ r}\left(t\right)}\right]\left(1-\gamma_n(t)\psi_m(t)\exp\left(-\mathbf{x}_{n}^{\top}\mathbf{r}\left(t\right)\right)\right)\left|\mathbf{x}_{n}^{\top}\mathbf{r}\left(t\right)\right|\leq 0
\end{align}
    \end{compactenum}
\end{asparaenum}

Collecting all the terms from the above special cases, and substituting
back into eq. \ref{eq: r correlation}, we note that all terms are
either negative, in $L_{1}$, or of the form $f\left(t\right)\left\Vert \mathbf{r}\left(t\right)\right\Vert $,
where $f\left(t\right)\in L_{1}$, thus proving the lemma. 
\end{proof}

\subsection{Proof of the existence and uniqueness of the solution to eqs. \ref{eq: w tilde-1}-\ref{eq: w tilde constraints} \label{subsec: existence 1}}

We wish to prove that $\forall m\geq1:$

\begin{equation}
\sum_{n\in\set_{m}}\exp\left(-\sum_{k=1}^{m}\tilde{\mathbf{w}}_{k}^{\top}\mathbf{x}_{n}\right)\bar{\mathbf{P}}_{m-1}\mathbf{x}_{n}=\hat{\mathbf{w}}_{m}\,,\label{eq: w tilde-1-1}
\end{equation}
such that 
\begin{equation}
\mathbf{P}_{m-1}\tilde{\mathbf{w}}_{m}=0\,\mathrm{and}\,\bar{\mathbf{P}}_{m}\tilde{\mathbf{w}}_{m}=0,\label{eq: w tilde constraints-1}
\end{equation}
we have a unique solution. From eq. \ref{eq: w tilde constraints-1},
we can modify eq. \ref{eq: w tilde-1-1} to 
\[
\sum_{n\in\set_{m}}\exp\left(-\sum_{k=1}^{m}\tilde{\mathbf{w}}_{k}^{\top}\bar{\mathbf{P}}_{k-1}\mathbf{x}_{n}\right)\bar{\mathbf{P}}_{m-1}\mathbf{x}_{n}=\hat{\mathbf{w}}_{m}\,,.
\]
To prove this, without loss of generality, and with a slight abuse
of notation, we will denote $\set_{m}$ as $\set_{1}$, $\bar{\mathbf{P}}_{m-1}\mathbf{x}_{n}$
as $\mathbf{x}_{n}$ and $\beta_{n}=\exp\left(-\sum_{k=1}^{m-1}\tilde{\mathbf{w}}_{k}^{\top}\bar{\mathbf{P}}_{k-1}\mathbf{x}_{n}\right)$,
so we can write the above equation as
\[
\sum_{n\in\set_{1}}\mathbf{x}_{n}\beta_{n}\exp\left(-\mathbf{x}_{n}^{\top}\tilde{\mathbf{w}}_{1}\right)=\what_{1}
\]
In the following Lemma \ref{lem: existence of solutions} we prove
this equation $\forall\boldsymbol{\beta}\in\mathbb{R}_{>0}^{\left|\set_{1}\right|}$.

\begin{lem}
$\forall\boldsymbol{\beta}\in\mathbb{R}_{>0}^{\left|\set_{1}\right|}$ we can find a unique $\tilde{\mathbf{w}}$ such that 
\begin{equation}
\sum_{n\in\set_{1}}\mathbf{x}_{n}\beta_{n}\exp\left(-\mathbf{x}_{n}^{\top}\tilde{\mathbf{w}}_{1}\right)=\what_{1}\label{eq: w_tilde equation}
\end{equation}
and for $\forall\mathbf{z}\in\mathbb{R}^{d}$ such that $\mathbf{z}^{\top}\mathbf{X}_{\set_{1}}=0$
we would have $\tilde{\mathbf{w}}_{1}^{\top}\mathbf{z}=0$.\label{lem: existence of solutions} 
\end{lem}

\begin{proof}
Let $K=\mathrm{rank}\left(\mathbf{X}_{\set_{1}}\right)$. Let and $\mathbf{U}=\left[\mathbf{u}_{1},\dots,\mathbf{u}_{d}\right]\in\mathbb{R}^{d\times d}$
be a set of orthonormal vectors (\emph{i.e.}, \textbf{$\mathbf{U}\mathbf{U}^{\top}=\mathbf{U}^{\top}\mathbf{U}=\mathbf{I}$})
such that $\mathbf{u}_{1}=\mathbf{\what}_{1}/\left\Vert \mathbf{\what}_{1}\right\Vert $,
and 
\begin{equation}
\forall\mathbf{z}\neq0,\forall n\in\set_{1}:\,\mathbf{z}^{\top}\left[\mathbf{u}_{1},\dots,\mathbf{u}_{K}\right]^{\top}\mathbf{x}_{n}\neq0\,,\label{eq: v no zero eigen value}
\end{equation}
while 
\begin{equation}
\forall i>K:\,\forall n\in\set_{1}:\,\mathbf{u}_{i}^{\top}\mathbf{x}_{n}=0\,.\label{eq: null space}
\end{equation}
In other words, $\mathbf{u}_{1}$ is in the direction of $\mathbf{\what}_{1}$,
$\left[\mathbf{u}_{1},\dots,\mathbf{u}_{K}\right]$ are in the space
spanned by the columns of $\mathbf{X}_{\set_{1}}$, and $\left[\mathbf{u}_{K+1},\dots,\mathbf{u}_{d}\right]$
are orthogonal to the columns of $\mathbf{X}_{\set_{1}}$.

We define $\mathbf{v}_{n}=\mathbf{U}^{\top}\mathbf{x}_{n}$ and $\mathbf{s}=\mathbf{U}^{\top}\tilde{\mathbf{w}}_{1}$.
Note that $\forall i>K:\,v_{i,n}=0\,\forall n\in\set_{1}$ from eq.
\ref{eq: null space}, and $\forall i>K:\,s_{i}=0$, since for $\forall\mathbf{z}\in\mathbb{R}^{d}$
such that $\mathbf{z}^{\top}\mathbf{X}_{\set_{1}}=0$ we would have
$\tilde{\mathbf{w}}_{1}^{\top}\mathbf{z}=0$. Lastly, equation \ref{eq: w_tilde equation}
becomes 
\begin{equation}
\sum_{n\in\set_{1}}\mathbf{x}_{n}\beta_{n}\exp\left(-\sum_{j=1}^{K}s_{j}v_{j,n}\right)=\what_{1}\,.\label{eq: s equation}
\end{equation}
Multiplying by $\mathbf{U}^{\top}$ from the left, we obtain 
\[
\forall i\leq K:\sum_{n\in\set_{1}}v_{i,n}\beta_{n}\exp\left(-\sum_{j=1}^{K}s_{j}v_{j,n}\right)=\mathbf{u}_{i}^{\top}\what_{1}\,.
\]
Since $\mathbf{u}_{1}=\mathbf{\what}_{1}/\left\Vert \mathbf{\what}_{1}\right\Vert $,
we have that 
\begin{equation}
\forall i\leq K:\sum_{n\in\set_{1}}v_{i,n}\beta_{n}\exp\left(-\sum_{j=1}^{K}s_{j}v_{j,n}\right)=\left\Vert \mathbf{\what}_{1}\right\Vert \delta_{i,1}\,.\label{eq: sum v full}
\end{equation}
We recall that $v_{1,n}=\hat{\mathbf{w}}_{1}^{\top}\mathbf{x}_{n}/\left\Vert \mathbf{\what}_{1}\right\Vert =1/\left\Vert \mathbf{\what}_{1}\right\Vert ,$
$\forall n\in\set_{1}$. Given $\left\{ s_{j}\right\} _{j=2}^{K}$,
we examine eq. \ref{eq: sum v full} for $i=1$, 
\[
\exp\left(-\frac{s_{1}}{\left\Vert \mathbf{\what_1}\right\Vert }\right)\left[\sum_{n\in\set_{1}}\beta_{n}\exp\left(-\sum_{j=2}^{K}s_{j}v_{j,n}\right)\right]=\left\Vert \mathbf{\what}_{1}\right\Vert ^{2}\,.
\]

This equation always has the unique solution 
\begin{equation}
s_{1}=\left\Vert \mathbf{\what}_{1}\right\Vert \log\left[\left\Vert \mathbf{\what}_{1}\right\Vert ^{-2}\sum_{n\in\set_{1}}\beta_{n}\exp\left(-\sum_{j=2}^{K}s_{j}v_{j,n}\right)\right]\,,\label{eq: s1 solution}
\end{equation}
given $\left\{ s_{j}\right\} _{j=2}^{K}$. Next, we similarly examine
eq. \ref{eq: sum v full} for $2\leq i\leq K$ as a function of $s_{i}$
\begin{equation}
\sum_{n\in\set_{1}}\beta_{n}v_{i,n}\exp\left(-s_{1}/\left\Vert \mathbf{\what}_{1}\right\Vert -\sum_{j=2}^{K}s_{j}v_{j,n}\right)=0\,.\label{eq: sum v i>1}
\end{equation}
multiplying by $\exp\left(s_{1}/\left\Vert \mathbf{\what}_{1}\right\Vert \right)$
we obtain 
\begin{align*}
0 & =\sum_{n\in\set_{1}}\beta_{n}v_{i,n}\exp\left(-\sum_{j=2}^{K}s_{j}v_{j,n}\right)=-\frac{\partial}{\partial s_{i}}\left[E\left(s_{2},\dots,s_{K}\right)\right]\,,
\end{align*}
where we defined 
\[
E\left(s_{2},\dots,s_{K}\right)=\sum_{n\in\set_{1}}\beta_{n}\exp\left(-\sum_{j=2}^{K}s_{j}v_{j,n}\right)\,.
\]
Therefore, any critical point of $E\left(s_{2},\dots,s_{K}\right)$
would be a solution of eq. \ref{eq: sum v i>1} for $2\leq i\leq K$,
and substituting this solution into eq. \ref{eq: s1 solution} we
obtain $s_{1}$. Since $\beta_{n}>0$, $E\left(s_{2},\dots,s_{K}\right)$
is a convex function, as positive linear combination of convex function
(exponential). Therefore, any finite critical point is a global minimum.
All that remains is to show that a finite minimum exists and that
it is unique.

From the definition of $\set_{1}$, $\exists\boldsymbol{\alpha}\in\mathbb{R}_{>0}^{\left|\set_{1}\right|}$
such that $\what_{1}=\sum_{n\in\set_{1}}\alpha_{n}\mathbf{x}_{n}$
. Multiplying this equation by $\mathbf{U}^{\top}$ we obtain that
$\exists\boldsymbol{\alpha}\in\mathbb{R}_{>0}^{\left|\set_{1}\right|}$
such that $2\leq i\leq K$ 
\begin{equation}
\,\sum_{n\in\set_{1}}v_{i,n}\alpha_{n}=0\,.\label{eq: alpha existence-1}
\end{equation}
Therefore, $\forall\left(s_{2,}\dots,s_{K}\right)\neq\mathbf{0}$
we have that 
\begin{equation}
\sum_{n\in\set_{1}}\left(\sum_{j=2}^{K}s_{j}v_{j,n}\right)\alpha_{n}=0\,.\label{eq: s v alpha}
\end{equation}
Recall, from eq. \ref{eq: v no zero eigen value} that $\forall\left(s_{2,}\dots,s_{K}\right)\neq\mathbf{0},\exists n\in\set_{1}:\,\sum_{j=2}^{K}s_{j}v_{j,n}\neq0$,
and that $\alpha_{n}>0$. Therefore, eq. \ref{eq: s v alpha} implies
that $\exists n\in\set_{1}$ such that $\sum_{j=2}^{K}s_{j}v_{j,n}>0$
and also $\exists m\in\set_{1}$ such that $\sum_{j=2}^{K}s_{j}v_{j,m}<0$.

Thus, in any direction we take a limit in which $\left|s_{i}\right|\rightarrow\infty$
$\forall2\leq i\leq K$, we obtain that $E\left(s_{2},\dots,s_{K}\right)\rightarrow\infty$,
since at least one exponent in the sum diverge. Since $E\left(s_{2},\dots,s_{K}\right)$,
is a continuous function, it implies it has a finite global minimum.
This proves the existence of a finite solution. To prove uniqueness
we will show the function is strictly convex, since the hessian is
(strictly) positive definite, \emph{i.e.}, that the following expression
is strictly positive: 
\begin{align*}
 & \sum_{i=2}^{K}\sum_{k=2}^{K}q_{i}q_{k}\frac{\partial}{\partial s_{i}}\frac{\partial}{\partial s_{k}}E\left(s_{2},\dots,s_{K}\right).\\
= & \sum_{n\in\set_{1}}\beta_{n}\left(\sum_{i=2}^{K}q_{i}v_{i,n}\right)\left(\sum_{k=2}^{K}q_{k}v_{k,n}\right)\exp\left(-\sum_{j=2}^{K}s_{j}v_{j,n}\right)\\
= & \sum_{n\in\set_{1}}\beta_{n}\left(\sum_{i=2}^{K}q_{i}v_{i,n}\right)^{2}\exp\left(-\sum_{j=2}^{K}s_{j}v_{j,n}\right)\,.
\end{align*}
the last expression is indeed strictly positive since $\forall\mathbf{q}\neq\mathbf{0},\exists n\in\set_{1}:\,\sum_{j=2}^{K}q_{j}v_{j,n}\neq0$,
from eq. \ref{eq: v no zero eigen value}. Thus, there exists a unique
solution $\tilde{\mathbf{w}}_{1}$. 
\end{proof}

\subsection{Proof of the existence and uniqueness of the solution to eqs. \ref{eq: w check}-\ref{eq: w check constraints}\label{subsec: existence 2}}
\begin{lem}
For $\forall m>k\geq1$, the equations
\begin{equation}
\sum_{n\in\set_{m}}\exp\left(-\tilde{\mathbf{w}}^{\top}\mathbf{x}_{n}\right)\mathbf{P}_{m-1}\mathbf{x}_{n}=\sum_{k=1}^{m-1}\left[\sum_{n\in\set_{k}}\exp\left(-\tilde{\mathbf{w}}^{\top}\mathbf{x}_{n}\right)\mathbf{x}_{n}\mathbf{x}_{n}^{\top}\right]\check{\mathbf{w}}_{k,m}\label{eq: w check proof}
\end{equation}
under the constraints

\begin{equation}
\mathbf{P}_{k-1}\check{\mathbf{w}}_{k,m}=0\,\mathrm{and}\,\mathbf{\bar{P}}_{k}\check{\mathbf{w}}_{k,m}=0\,\label{eq: w check constraints-1}
\end{equation}
have a unique solution $\check{\mathbf{w}}_{k,m}$.
\end{lem}

\begin{proof}
For this proof we denote $\mathbf{X}_{\set_{k}}$ as the matrix which
columns are $\left\{ \mathbf{x}_{n}|n\in\set_{k}\right\} $, the orthogonal
projection matrix $\mathbf{Q}_{k}=\mathbf{P}_{k}\bar{\mathbf{P}}_{k-1}$where
$\mathbf{Q}_{k}\mathbf{Q}_{m}=0$ $\forall k\neq m$, $\mathbf{Q}_{k}\bar{\mathbf{P}}_{m}=0$
$\forall k<m$, and
\begin{equation}
\forall m:\,\mathbf{I}=\mathbf{P}_{m}+\bar{\mathbf{P}}_{m}=\sum_{k=1}^{m}\mathbf{Q}_{k}+\bar{\mathbf{P}}_{m}\label{eq: identity}
\end{equation}
We will write $\check{\mathbf{w}}_{k,m}=\mathbf{W}_{k,m}\mathbf{u}_{k,m}$
, where $\mathbf{u}_{k,m}\in\mathbb{R}^{d_{k}}$ and $\mathbf{W}_{k,m}\in\mathbb{R}^{d\times d_{k}}$
is a full rank matrix such that $\mathbf{Q}_{k}\mathbf{W}_{k,m}=\mathbf{W}_{k,m}$,
so 
\begin{equation}
\check{\mathbf{w}}_{k,m}=\mathbf{Q}_{k}\check{\mathbf{w}}_{k,m}=\mathbf{Q}_{k}\mathbf{W}_{k,m}\mathbf{u}_{k,m}\,.\label{eq: w wheck re-param}
\end{equation}
 and, furthermore,
\begin{equation}
\mathrm{rank}\left[\mathbf{X}_{\set_{k}}^{\top}\mathbf{Q}_{k}\mathbf{W}_{k,m}\right]=\mathrm{rank}\left(\mathbf{X}_{\set_{k}}^{\top}\mathbf{Q}_{k}\right)=d_{k}\,.\label{eq: rank W}
\end{equation}
Recall that $\forall m:\,\mathbf{\bar{P}}_{m}\mathbf{P}_{m}=\boldsymbol{0}$
and $\forall k\geq1$, $\forall n\in\set_{m}$ $\mathbf{\bar{P}}_{m+k}\mathbf{x}_{n}=\boldsymbol{0}$.
Therefore, $\forall\mathbf{v}\in\mathbb{R}^{d}$ , $\mathbf{P}_{k-1}\mathbf{Q}_{k}\mathbf{v}=\boldsymbol{0}$,
$\mathbf{\bar{P}}_{k}\mathbf{Q}_{k}\mathbf{v}=\mathbf{0}$. Thus,
$\check{\mathbf{w}}_{k,m}$ eq. \ref{eq: w wheck re-param} implies
the constraints in eq. \ref{eq: w check constraints-1} hold. 

Next, we prove the existence and uniqueness of the solution $\check{\mathbf{w}}_{k,m}$
for each $k=1,\dots,m$ separately. We multiply eq. \ref{eq: w check proof}
from the left by the identity matrix, decomposed to orthogonal projection
matrices as in eq. \ref{eq: identity}. Since each matrix projects
to an orthogonal subspace, we can solve each product separately. 

The product with $\bar{\mathbf{P}}_{m}$ is equal to zero for both
sides of the equation. The product with $\mathbf{Q}_{k}$ is equal
to

\[
\sum_{n\in\set_{m}}\exp\left(-\tilde{\mathbf{w}}^{\top}\mathbf{x}_{n}\right)\mathbf{Q}_{k}\mathbf{P}_{m-1}\mathbf{x}_{n}=\left[\sum_{n\in\set_{k}}\exp\left(-\tilde{\mathbf{w}}^{\top}\mathbf{x}_{n}\right)\mathbf{Q}_{k}\mathbf{x}_{n}\mathbf{x}_{n}^{\top}\right]\check{\mathbf{w}}_{k,m}\,.
\]
Substituting eq. \ref{eq: w wheck re-param}, and multiplying by $\mathbf{W}_{k,m}^{\top}$
from the right, we obtain
\begin{equation}
\sum_{n\in\set_{m}}\exp\left(-\tilde{\mathbf{w}}^{\top}\mathbf{x}_{n}\right)\mathbf{W}_{k,m}^{\top}\mathbf{Q}_{k}\mathbf{P}_{m-1}\mathbf{x}_{n}=\left[\sum_{n\in\set_{k}}\exp\left(-\tilde{\mathbf{w}}^{\top}\mathbf{x}_{n}\right)\mathbf{W}_{k,m}^{\top}\mathbf{Q}_{k}\mathbf{x}_{n}\mathbf{x}_{n}^{\top}\mathbf{Q}_{k}\mathbf{W}_{k,m}\right]\mathbf{u}_{k,m}\,.\label{eq: u equation}
\end{equation}
Denoting $\mathbf{E}_{k}\in\mathbb{R}^{\left|\set_{k}\right|\times\left|\set_{k}\right|}$
as diagonal matrix for which $E_{nn,k}=\exp\left(-\frac{1}{2}\tilde{\mathbf{w}}^{\top}\mathbf{x}_{n}\right)$,
the matrix in the square bracket in the left hand side can be written
as
\begin{equation}
\mathbf{W}_{k,m}^{\top}\mathbf{Q}_{k}\mathbf{X}_{\set_{k}}\mathbf{E}_{k}\mathbf{E}_{k}\mathbf{X}_{\set_{k}}^{\top}\mathbf{Q}_{k}\mathbf{W}_{k,m}\,.\label{eq: rank check}
\end{equation}
Since $\mathrm{rank}\left(\mathbf{A}\mathbf{A}^{\top}\right)=\mathrm{rank}\left(\mathbf{A}\right)$
for any matrix $\mathbf{A}$, the rank of this matrix is equal to
\[
\mathrm{rank}\left[\mathbf{E}\mathbf{X}_{\set_{k}}\mathbf{Q}_{k}\mathbf{W}_{k,m}\right]\overset{\left(1\right)}{=}\mathrm{rank}\left[\mathbf{X}_{\set_{k}}\mathbf{Q}_{k}\mathbf{W}_{k,m}\right]\overset{\left(2\right)}{=}d_{k}
\]
where in $\left(1\right)$ we used that $\mathbf{E}_{k}$ is diagonal
and non-zero, and in $\left(2\right)$ we used eq. \ref{eq: rank W}.
This implies that the $d_{k}\times d_{k}$ matrix in eq. \ref{eq: rank check}
is full rank, and so eq. \ref{eq: u equation} has a unique solution
$\mathbf{u}_{k,m}$. Therefore, there exists a unique solution $\check{\mathbf{w}}_{k,m}$.
\end{proof}

\subsection{Proof of Lemma \ref{lem: PhiSum}} \label{subsec: PhiSumProof}

\PhiSum*
\begin{proof}
We define $\psi\left(t\right)=z\left(t\right)+h\left(t\right)$, and
start from eq. \ref{eq: rho bound}
\begin{align*}
 & \phi^{2}\left(t+1\right)\\
 & \leq z\left(t\right)+h\left(t\right)\phi\left(t\right)+\phi^{2}\left(t\right)\\
 & \leq z\left(t\right)+h\left(t\right)\max\left[1,\phi^{2}\left(t\right)\right]+\phi^{2}\left(t\right)\\
 & \leq z\left(t\right)+h\left(t\right)+h\left(t\right)\phi^{2}\left(t\right)+\phi^{2}\left(t\right)\\
 & \leq\psi\left(t\right)+\left(1+h\left(t\right)\right)\phi^{2}\left(t\right)\\
 & \leq\psi\left(t\right)+\left(1+h\left(t\right)\right)\psi\left(t-1\right)+\left(1+h\left(t\right)\right)\left(1+h\left(t-1\right)\right)\phi^{2}\left(t-1\right)\\
 & \leq\psi\left(t\right)+\left(1+h\left(t\right)\right)\psi\left(t-1\right)+\left(1+h\left(t\right)\right)\left(1+h\left(t-1\right)\right)\psi\left(t-2\right)\\
 & +\left(1+h\left(t\right)\right)\left(1+h\left(t-1\right)\right)\left(1+h\left(t-2\right)\right)\phi^{2}\left(t-2\right)
\end{align*}
we keep iterating eq. \ref{eq: rho bound}, until we obtain
\begin{align*}
 & \leq\left[\prod_{m=1}^{t-1}\left(1+h\left(t-m\right)\right)\right]\phi\left(t_{1}\right)+\sum_{k=0}^{t-t_{1}}\left[\prod_{m=0}^{k-1}\left(1+h\left(t-m\right)\right)\right]\psi\left(t-k\right)\\
 & \leq\left[\exp\left(\sum_{m=1}^{t-1}h\left(t-m\right)\right)\right]\phi\left(t_{1}\right)+\sum_{k=0}^{t-1}\left[\exp\left(\sum_{m=1}^{k-1}h\left(t-m\right)\right)\right]\psi\left(t-k\right)\\
 & \leq\exp\left(C\right)\left[\phi\left(1\right)+\sum_{k=0}^{t-1}\psi\left(t-k\right)\right]\\
 & \leq\exp\left(C\right)\left[\phi\left(1\right)+\sum_{u=1}^{t}\psi\left(u\right)\right]\\
 & \leq\exp\left(C\right)\left[\phi\left(1\right)+\sum_{u=1}^{t}\left(z\left(u\right)+h\left(u\right)\right)\right]\\
 & \leq\exp\left(C\right)\left[\phi\left(1\right)+C+\sum_{u=1}^{t}z\left(u\right)\right]
\end{align*}
Therefore, the Lemma holds with $C_{2}=\left(\phi\left(1\right)+C\right)\exp\left(C\right)$
and $C_{3}=\exp\left(C\right)$.
\end{proof}

\newpage
\section{Calculation of convergence rates \label{sec:Calculation-of-convergence rates}}

In this section we calculate the various rates mentioned in section \ref{sec: convergence rates}. 

\subsection{Proof of Theorem \ref{thm: rates}}

From Theorems \ref{thm: refined Theorem} and \ref{theorem: main2}, we can write $\mathbf{w}\left(t\right)=\hat{\mathbf{w}}\log t+\boldsymbol{\rho}\left(t\right)$,
where $\boldsymbol{\rho}\left(t\right)$ has a bounded norm for almost all datasets, while in zero measure case $\boldsymbol{\rho}\left(t\right)$ contains additional  $O(\log\log(t))$ components which are orthogonal to the support vectors in $\set_1$, and, asymptotically, have a positive angle with the other support vectors. In this section we first calculate the various convergence rates for the non-degenerate case of Theorem \ref{thm: refined Theorem}, and then write the correction in the zero measure cases, if there is such a correction.

First, we calculated of the normalized weight vector (eq. \ref{eq: normalized weight vector}), for almost every dataset:
\begin{align}
 & \frac{\mathbf{w}\left(t\right)}{\left\Vert \mathbf{w}\left(t\right)\right\Vert }\nonumber \\
 & =\frac{\boldsymbol{\rho}\left(t\right)+\hat{\mathbf{w}}\log t}{\sqrt{\boldsymbol{\rho}\left(t\right)^{\top}\boldsymbol{\rho}\left(t\right)+\hat{\mathbf{w}}^{\top}\hat{\mathbf{w}}\log^{2}t+2\boldsymbol{\rho}\left(t\right)^{\top}\hat{\mathbf{w}}\log t}}\nonumber \\
 & =\frac{\boldsymbol{\rho}\left(t\right)/\log t+\hat{\mathbf{w}}}{\left\Vert \hat{\mathbf{w}}\right\Vert \sqrt{1+2\boldsymbol{\rho}\left(t\right)^{\top}\hat{\mathbf{w}}/\left(\left\Vert \hat{\mathbf{w}}\right\Vert ^{2}\log t\right)+\left\Vert \boldsymbol{\rho}\left(t\right)\right\Vert ^{2}/\left(\left\Vert \hat{\mathbf{w}}\right\Vert ^{2}\log^{2}t\right)}}\nonumber \\
 & =\frac{1}{\left\Vert \hat{\mathbf{w}}\right\Vert }\left(\boldsymbol{\rho}\left(t\right)\frac{1}{\log t}+\hat{\mathbf{w}}\right)\left[1-\frac{\boldsymbol{\rho}\left(t\right)^{\top}\hat{\mathbf{w}}}{\left\Vert \hat{\mathbf{w}}\right\Vert ^{2}\log t}+\left[\frac{3}{2}\left(\frac{\boldsymbol{\rho}\left(t\right)^\top\hat{\mathbf{w}}}{\left\Vert \hat{\mathbf{w}}\right\Vert ^{2}}\right)^{2}-\frac{\left\Vert \boldsymbol{\rho}\left(t\right)\right\Vert ^{2}}{2\left\Vert \hat{\mathbf{w}}\right\Vert ^{2}}\right]\frac{1}{\log^{2}t}+O\left(\frac{1}{\log^{3}t}\right)\right]\label{eq: normalized weight vector full}\\
 & =\frac{\hat{\mathbf{w}}}{\left\Vert \hat{\mathbf{w}}\right\Vert }+\left(\frac{\boldsymbol{\rho}\left(t\right)}{\left\Vert \hat{\mathbf{w}}\right\Vert }-\frac{\hat{\mathbf{w}}}{\left\Vert \hat{\mathbf{w}}\right\Vert }\frac{\boldsymbol{\rho}\left(t\right)^{\top}\hat{\mathbf{w}}}{\left\Vert \hat{\mathbf{w}}\right\Vert ^{2}}\right)\frac{1}{\log t}+O\left(\frac{1}{\log^{2}t}\right)\nonumber \\
 & =\frac{\hat{\mathbf{w}}}{\left\Vert \hat{\mathbf{w}}\right\Vert }+\left(\mathbf{I}-\frac{\hat{\mathbf{w}}\hat{\mathbf{w}}^{\top}}{\left\Vert \hat{\mathbf{w}}\right\Vert ^{2}}\right)\frac{\boldsymbol{\rho}\left(t\right)}{\left\Vert \hat{\mathbf{w}}\right\Vert }\frac{1}{\log t}+O\left(\frac{1}{\log^{2}t}\right),\nonumber 
\end{align}
where to obtain eq. \ref{eq: normalized weight vector full} we used
$\frac{1}{\sqrt{1+x}}=1-\frac{1}{2}x+\frac{3}{4}x^{2}+O\left(x^{3}\right)$, 
and in the last line we used the fact that $\boldsymbol{\rho}\left(t\right)$
has a bounded norm for almost every dataset. Thus, in this case
\begin{align}
 & \norm{\frac{\mathbf{w}\left(t\right)}{\left\Vert \mathbf{w}\left(t\right)\right\Vert } - \frac{\hat{\mathbf{w}}}{\left\Vert \hat{\mathbf{w}}\right\Vert }} = O\left(\frac{1}{\log t}\right).\nonumber
\end{align} 

For the measure zero cases, we  instead have from eq. \ref{eq: asymptotic form-1}, $\wvec(t)=\sum_{m=1}^M \hat{\wvec}\log^{\circ m}(t)+\mathbf{\rho}(t)$, where $\|\mathbf{\rho}(t)\|$ is bounded (Theorem~\ref{thm: main theorem}). Let $\tilde{\mathbf{\rho}}(t)=\sum_{m=2}^M \hat{\wvec}\log^{\circ m}(t)+\mathbf{\rho}(t)$, such that $\wvec(t)=\hat{\wvec}\log(t)+\tilde{\mathbf{\rho}}(t)$ with $\tilde{\mathbf{\rho}}(t)=O(\log\log(t))$. Repeating the same calculations as above, we have for the degenerate cases, 
\begin{align}
 & \norm{\frac{\mathbf{w}\left(t\right)}{\left\Vert \mathbf{w}\left(t\right)\right\Vert } - \frac{\hat{\mathbf{w}}}{\left\Vert \hat{\mathbf{w}}\right\Vert }} = O\left(\frac{\log \log t}{\log t}\right)\nonumber
\end{align}

Next, we use eq. \ref{eq: normalized weight vector full} to calculate the
angle (eq. \ref{eq: angle})
\begin{align*}
 & \frac{\mathbf{w}\left(t\right)^{\top}\hat{\mathbf{w}}}{\left\Vert \mathbf{w}\left(t\right)\right\Vert \left\Vert \hat{\mathbf{w}}\right\Vert }\\
= & \frac{\hat{\mathbf{w}}^{\top}}{\left\Vert \hat{\mathbf{w}}\right\Vert ^{2}}\left(\boldsymbol{\rho}\left(t\right)\frac{1}{\log t}+\hat{\mathbf{w}}\right)\left(1-\frac{1}{\log t}\frac{\boldsymbol{\rho}\left(t\right)^{\top}\hat{\mathbf{w}}}{\left\Vert \hat{\mathbf{w}}\right\Vert ^{2}}+\left[\frac{3}{4}\left(2\frac{\boldsymbol{\rho}\left(t\right)^{\top}\hat{\mathbf{w}}}{\left\Vert \hat{\mathbf{w}}\right\Vert ^{2}}\right)^{2}-\frac{\left\Vert \boldsymbol{\rho}\left(t\right)\right\Vert ^{2}}{2\left\Vert \hat{\mathbf{w}}\right\Vert ^{2}}\right]\frac{1}{\log^{2}t}+O\left(\frac{1}{\log^{3}t}\right)\right)\\
= & 1+\frac{2\left\Vert \boldsymbol{\rho}\left(t\right)\right\Vert ^{2}}{\left\Vert \hat{\mathbf{w}}\right\Vert ^{2}}\left[\left(\frac{\boldsymbol{\rho}\left(t\right)^{\top}\hat{\mathbf{w}}}{\left\Vert \hat{\mathbf{w}}\right\Vert \left\Vert \boldsymbol{\rho}\left(t\right)\right\Vert }\right)^{2}-\frac{1}{4}\right]\frac{1}{\log^{2}t}+O\left(\frac{1}{\log^{3}t}\right)
\end{align*}
for almost every dataset. Thus, in this case
\begin{align}
 & \frac{\mathbf{w}\left(t\right)^{\top}\hat{\mathbf{w}}}{\left\Vert \mathbf{w}\left(t\right)\right\Vert \left\Vert \hat{\mathbf{w}}\right\Vert } = O\left(\frac{1}{\log^2 t}\right)\nonumber
\end{align} 
Repeating the same calculation for the measure zero case, we have instead
\begin{align}
 & \frac{\mathbf{w}\left(t\right)^{\top}\hat{\mathbf{w}}}{\left\Vert \mathbf{w}\left(t\right)\right\Vert \left\Vert \hat{\mathbf{w}}\right\Vert } = O\left(\left(\frac{\log \log t}{\log t}\right)^2\right)\nonumber
\end{align} 

Next, we calculate the margin (eq. \ref{eq: margin})
\begin{align}
 & \min_{n}\frac{\mathbf{x}_{n}^{\top}\mathbf{w}\left(t\right)}{\norm{\wvec(t)}}- \frac{1}{\left\Vert \hat{\mathbf{w}}\right\Vert }\nonumber \\
 & =\min_{n}\mathbf{x}_{n}^{\top}\left[\left(\frac{\boldsymbol{\rho}\left(t\right)}{\left\Vert \hat{\mathbf{w}}\right\Vert }-\frac{\hat{\mathbf{w}}}{\left\Vert \hat{\mathbf{w}}\right\Vert }\frac{\boldsymbol{\rho}\left(t\right)^{\top}\hat{\mathbf{w}}}{\left\Vert \hat{\mathbf{w}}\right\Vert ^{2}}\right)\frac{1}{\log t}+O\left(\frac{1}{\log^{2}t}\right)\right]\nonumber \\
 & =\frac{1}{\left\Vert \hat{\mathbf{w}}\right\Vert }\left(\min_{n}\mathbf{x}_{n}^{\top}\boldsymbol{\rho}\left(t\right)-\frac{\boldsymbol{\rho}\left(t\right)^{\top}\hat{\mathbf{w}}}{\left\Vert \hat{\mathbf{w}}\right\Vert ^{2}}\right)\frac{1}{\log t}+O\left(\frac{1}{\log^{2}t}\right)\label{eq: margin full}
\end{align}
for almost every dataset, where in eq. \ref{eq: margin full} we used eq. \ref{eq: v1 SVM}. Interestingly the measure zero case has a similar convergence rate, since after a sufficient number of iterations, the $O(\log \log (t))$ correction is orthogonal to $\mathbf{x}_k$, where $k=\mathrm{argmin}_n \xnT \wvec(t)$.
Thus, for all datasets,
\begin{align}
 & \min_{n}\mathbf{x}_{n}^{\top}\mathbf{w}\left(t\right) -\frac{1}{\left\Vert \hat{\mathbf{w}}\right\Vert } = O\left(\frac{1}{\log t}\right)
\end{align}

Calculation of the training loss (eq. \ref{eq: logistic loss convergence}):
\begin{align*}
\mathcal{L}\left(\mathbf{w}\left(t\right)\right) & \leq\sum_{n=1}^{N}\left(1+\exp\left(-\mu_{+}\mathbf{w}\left(t\right)^{\top}\mathbf{x}_{n}\right)\right)\exp\left(-\mathbf{w}\left(t\right)^{\top}\mathbf{x}_{n}\right)\\
 & =\sum_{n=1}^{N}\left(1+\exp\left(-\mu_{+}\left(\boldsymbol{\rho}\left(t\right)+\hat{\mathbf{w}}\log t\right)^{\top}\mathbf{x}_{n}\right)\right)\exp\left(-\left(\boldsymbol{\rho}\left(t\right)+\hat{\mathbf{w}}\log t\right)^{\top}\mathbf{x}_{n}\right)\\
 & =\sum_{n=1}^{N}\left(1+t^{-\mu_{+}\hat{\mathbf{w}}^{\top}\mathbf{x}_{n}}\exp\left(-\mu_{+}\boldsymbol{\rho}\left(t\right)^{\top}\mathbf{x}_{n}\right)\right)\exp\left(-\boldsymbol{\rho}\left(t\right)^{\top}\mathbf{x}_{n}\right)t^{-\hat{\mathbf{w}}^{\top}\mathbf{x}_{n}}\\
 & =\frac{1}{t}\sum_{n\in\set}e^{-\boldsymbol{\rho}\left(t\right)^{\top}\mathbf{x}_{n}}+O\left(t^{-\max\left(\theta,1+\mu_{+}\right)}\right)\,.
\end{align*}
Thus, for all datasets $\mathcal{L}\left(\mathbf{w}\left(t\right)\right)=O(t^{-1})$. Note that the zero measure case has the same behavior, since after a sufficient number of iterations, the $O(\log \log (t))$ correction has a non-negative angle with all the support vectors.

Next, we give an example demonstrating the bounds above, for the non-degenerate case, are strict.
Consider optimization with and exponential loss $\ell\left(u\right)=e^{-u}$,
and a single data point $\mathbf{x}=\left(1,0\right)$. In this case
$\hat{\mathbf{w}}=\left(1,0\right)$ and $\left\Vert \hat{\mathbf{w}}\right\Vert =1$.
We take the limit $\eta\rightarrow0$, and obtain the continuous time
version of GD: 
\begin{align*}
\dot{w}_{1}\left(t\right) & =\exp\left(-w\left(t\right)\right)\,\,;\,\,\dot{w}_{2}\left(t\right)=0.
\end{align*}
We can analytically integrate these equations to obtain 
\begin{align*}
w_{1}\left(t\right) & =\log\left(t+\exp\left(w_{1}\left(0\right)\right)\right)\,\,;\,\,w_{2}\left(t\right)=w_{2}\left(0\right).
\end{align*}

Using this example with $w_{2}\left(0\right)>0$, it is easy to see that the above upper bounds are strict in the non-degenerate case. $\blacksquare$

\subsection{Validation error lower bound}
Lastly, recall that $\mathcal{V}$ is a set of indices for validation
set samples. We calculate of the validation loss for logistic loss,
if the error of the $L_{2}$ max margin vector has some classification
errors on the validation, \emph{i.e., $\exists k\in\mathcal{V}:\,\hat{\mathbf{w}}{}^{\top}\mathbf{x}_{k}<0$:}
\begin{align*}
\mathcal{L}_{\mathrm{val}}\left(\mathbf{w}\left(t\right)\right) & =\sum_{n\in\mathcal{V}}\log\left(1+\exp\left(-\mathbf{w}\left(t\right)^{\top}\mathbf{x}_{n}\right)\right)\\
 & \geq\log\left(1+\exp\left(-\mathbf{w}\left(t\right)^{\top}\mathbf{x}_{k}\right)\right)\\
 & =\log\left(1+\exp\left(-\left(\boldsymbol{\rho}\left(t\right)+\hat{\mathbf{w}}\log t\right)^{\top}\mathbf{x}_{k}\right)\right)\\
 & =\log\left(\exp\left(-\left(\boldsymbol{\rho}\left(t\right)+\hat{\mathbf{w}}\log t\right)^{\top}\mathbf{x}_{k}\right)\left(1+\exp\left(\left(\boldsymbol{\rho}\left(t\right)+\hat{\mathbf{w}}\log t\right)^{\top}\mathbf{x}_{k}\right)\right)\right)\\
 & \geq-\left(\boldsymbol{\rho}\left(t\right)+\hat{\mathbf{w}}\log t\right)^{\top}\mathbf{x}_{k}+\log\left(1+\exp\left(\left(\boldsymbol{\rho}\left(t\right)+\hat{\mathbf{w}}\log t\right)^{\top}\mathbf{x}_{k}\right)\right)\\
 & \geq-\log t\hat{\mathbf{w}}^{\top}\mathbf{x}_{k}+\boldsymbol{\rho}\left(t\right)^{\top}\mathbf{x}_{k}
\end{align*}
Thus, for all datasets $\mathcal{L}_{\mathrm{val}}\left(\mathbf{w}\left(t\right)\right)=\Omega(\log (t))$.

\section{Softmax output with cross-entropy loss\label{sec:Softmax-output-with-cross-entropy-loss}}

We examine multiclass classification. In the case the labels are the
class index $y_{n}\in\left\{ 1,\dots,K\right\} $ and we have a weight
matrix $\mathbf{W}\in\mathbb{R}^{K\times d}$ with $\mathbf{w}_{k}$
being the $k$-th row of $\mathbf{W}$.

Furthermore, we define $\mathbf{w}=\mathrm{vec}$$\left(\mathbf{W}^{\top}\right)$,
a basis vector $\mathbf{e}_{k}\in\mathbb{R}^{K}$ so that$\left(\mathbf{e}_{k}\right)_{i}=\delta_{ki}$,
and the matrix $\mathbf{A}_{k}\in\mathbb{R}^{dK\times d}$ so that
$\mathbf{A}_{k}=\mathbf{e}_{k}\otimes\mathbf{I}_{d}$, where $\otimes$
is the Kronecker product and $\mathbf{I}_{d}$ is the $d$-dimension
identity matrix. Note that $\mathbf{A}_{k}^{\top}\w=\w_{k}$.

Consider the cross entropy loss with softmax output 
\begin{align*}
\mathcal{L}\left(\mathbf{W}\right) & =-\sum_{n=1}^{N}\log\left(\frac{\exp\left(\mathbf{w}_{y_{n}}^{\top}\mathbf{x}_{n}\right)}{\sum_{k=1}^{K}\exp\left(\mathbf{w}_{k}^{\top}\mathbf{x}_{n}\right)}\right)
\end{align*}

Using our notation, this loss can be re-written as 
\begin{align}
\mathcal{L}\left(\mathbf{w}\right) & =-\sum_{n=1}^{N}\log\left(\frac{\exp\left(\mathbf{w}^{\top}\mathbf{A}_{y_{n}}\mathbf{x}_{n}\right)}{\sum_{k=1}^{K}\exp\left(\mathbf{w}^{\top}\mathbf{A}_{k}\mathbf{x}_{n}\right)}\right)\nonumber \\
 & =\sum_{n=1}^{N}\log\left(\sum_{k=1}^{K}\exp\left(\mathbf{w}^{\top}\left(\mathbf{A}_{k}-\mathbf{A}_{y_{n}}\right)\mathbf{x}_{n}\right)\right)\label{eq: cross-entropy loss}
\end{align}
Therefore 
\begin{align*}
\nabla\mathcal{L}\left(\mathbf{w}\right) & =\sum_{n=1}^{N}\frac{\sum_{k=1}^{K}\exp\left(\mathbf{w}^{\top}\left(\mathbf{A}_{k}-\mathbf{A}_{y_{n}}\right)\mathbf{x}_{n}\right)\left(\mathbf{A}_{k}-\mathbf{A}_{y_{n}}\right)\mathbf{x}_{n}}{\sum_{r=1}^{K}\exp\left(\mathbf{w}^{\top}\left(\mathbf{A}_{r}-\mathbf{A}_{y_{n}}\right)\mathbf{x}_{n}\right)}\\
 & =\sum_{n=1}^{N}\sum_{k=1}^{K}\frac{1}{\sum_{r=1}^{K}\exp\left(\mathbf{w}^{\top}\left(\mathbf{A}_{r}-\mathbf{A}_{k}\right)\mathbf{x}_{n}\right)}\left(\mathbf{A}_{k}-\mathbf{A}_{y_{n}}\right)\mathbf{x}_{n}\,.
\end{align*}

If, again, we make the assumption that the data is linearly
separable, \emph{i.e.}, in our notation

{\assm $\exists\mathbf{w}_{*}$ such that $\mathbf{w}_{*}^{\top}\left(\mathbf{A}_{k}-\mathbf{A}_{y_{n}}\right)\mathbf{x}_{n}<0$
$\forall k\neq y_{n}$. \label{assum: multi-class seperability} }

then the expression 
\begin{align*}
\mathbf{w}_{*}^{\top}\nabla\mathcal{L}\left(\mathbf{w}\right) & =\sum_{n=1}^{N}\sum_{k=1}^{K}\frac{\mathbf{w}_{*}^{\top}\left(\mathbf{A}_{k}-\mathbf{A}_{y_{n}}\right)\mathbf{x}_{n}}{\sum_{r=1}^{K}\exp\left(\mathbf{w}^{\top}\left(\mathbf{A}_{r}-\mathbf{A}_{k}\right)\mathbf{x}_{n}\right)}\,.
\end{align*}
is strictly negative for any finite $\mathbf{w}$. However, from Lemma
\ref{lem: GD convergence}, in gradient descent with an appropriately small learning rate, we have that $\nabla L\left(\mathbf{w}\left(t\right)\right)\rightarrow\mathbf{0}$.
This implies that: $\left\Vert \mathbf{w}\left(t\right)\right\Vert \rightarrow\infty$,
and $\forall k\neq y_{n},\exists r:\,\mathbf{w}\left(t\right)^{\top}\left(\mathbf{A}_{r}-\mathbf{A}_{k}\right)\mathbf{x}_{n}\rightarrow\infty$,
which implies $\forall k\neq y_{n},\max_{k}\mathbf{w}\left(t\right)^{\top}\left(\mathbf{A}_{k}-\mathbf{A}_{y_{n}}\right)\mathbf{x}_{n}\rightarrow-\infty$.
Examining the loss (eq. \ref{eq: cross-entropy loss}) we find that
$\mathcal{L}\left(\mathbf{w}\left(t\right)\right)\rightarrow\mathbf{0}$
in this case. Thus, we arrive to an equivalent Lemma to Lemma \ref{lem: convergence of linear classifiers},
for this case: 
\begin{lem}
\label{lem: convergence of softmax-cross-entropy}Let $\mathbf{w}\left(t\right)$
be the iterates of gradient descent (eq. \ref{eq: gradient descent linear})
with an appropriately small learning rate, for cross-entropy loss operating on a softmax
output, under the assumption of strict linear separability (Assumption
\ref{assum: multi-class seperability}), then: (1) $\lim_{t\rightarrow\infty}\mathcal{L}\left(\mathbf{w}\left(t\right)\right)=0$,
(2) $\lim_{t\rightarrow\infty}\left\Vert \mathbf{w}\left(t\right)\right\Vert =\infty$,
and (3) $\forall n,k\neq y_{n}:\,\lim_{t\rightarrow\infty}\mathbf{w}\left(t\right)^{\top}\left(\mathbf{A}_{y_{n}}-\mathbf{A}_{k}\right)\mathbf{x}_{n}=\infty$. 
\end{lem}
Using Lemma \ref{lem: GD convergence} and Lemma \ref{lem: convergence of softmax-cross-entropy}, we prove the following Theorem (equivalent to Theorem \ref{thm: main theorem almost everywhere}) in the next section:
\mainMulti*
\remove{
\noindent\textbf{More refined analysis: characterizing the residual}\\ For each of the K classes, we define $\mathcal{S}_k=\argmin_n(\hat{\vect{w}}_{y_n}-\hat{\vect{w}}_k)^\top\vect{x}_n$ (the k'th class support vectors). Using this definition, we define  $\mathcal{S}\triangleq{\bigcup\limits_{k=1}^{K} \mathcal{S}_k}$\\
\begin{thm} \label{thm:resMulti}
	Under the conditions of Theorem \ref{thm:mainMulti}, if, in addition, each class support vectors span the data ($i.e.$ $\forall k:\  rank(X_{\mathcal{S}_k})=rank(X)$), then $\lim_{t\to\infty} \rho(t)=\tilde{w}$, where $\tilde{w}$ is the unique solution to 
	\begin{equation} \label{eq:alpha_def}
	\forall k,\ \forall n\in \mathcal{S}_k\ : \ \eta \e( (\tilde{\vect{w}}_k - \tilde{\vect{w}}_{y_n})^\top \vect{x}_n) = \alpha_{n,k}
	\end{equation}
\end{thm}
}

\subsection{Notations and Definitions}
To prove Theorem \ref{thm:mainMulti} we require additional notation.  we define $\xtilde \triangleq (\vect{A}_{y_n}-\vect{A}_k)\vect{x}_n$. Using this notation, we can re-write eq. \ref{eq: K-class SVM} (K-class SVM) as
\begin{equation}
    \arg \min_\wvec \Vert \wvec \Vert^2\,\textrm{s.t.}\,\forall n, \forall k\neq y_n: \wvec^\top \xtilde \ge 1  
\end{equation}
From the KKT optimality conditions, we have for some $\alpha_{n,k}\ge 0$, 
\begin{equation} \label{eq: what def with alpha for CE}
	\what = \sumn \sumk \alpha_{n,k} \xtilde \indicator{n\in \mathcal{S}_k}
\end{equation}
In addition, for each of the K classes, we define $\mathcal{S}_k=\arg \min_n(\hat{\vect{w}}_{y_n}-\hat{\vect{w}}_k)^\top\vect{x}_n$ (the k'th class support vectors).\\ Using this definition, we define
$\vect{X}_{\mathcal{S}_k}\in\mathcal{R}^{dK\times |S_k|}$ as the matrix which columns are $\xtilde, \ \forall n\in \mathcal{S}_k$. We also define $\mathcal{S}\triangleq{\bigcup\limits_{k=1}^{K}\mathcal{S}_k} $ and $\tilde{\vect{X}}_\mathcal{S}\triangleq{\bigcup\limits_{k=1}^{K} \vect{X}_{\mathcal{S}_k}}$.\\
We recall that we defined $\vect{W}\in \mathbb{R}^{K\times d}$ with $\vect{w}_k$ being the k-th row of $\vect{W}$ and $\vect{w}=\mathrm{vec}(\vect{W}^\top)$. Similarly, we define:\\
1. $\hat{\vect{W}}\in \mathbb{R}^{K\times d}$ with $\hat{\vect{w}}_k$ being the k-th row of $\hat{\vect{W}}$\\
2. $\vect{P}\in \mathbb{R}^{K\times d}$ with $\bm{\rho}_k$ being the k-th row of $\vect{P}$\\
3. $\tilde{\vect{W}}\in \mathbb{R}^{K\times d}$ with $\tilde{\vect{w}}_k$ being the k-th row of $\tilde{\vect{W}}$\\
and $\hat{\vect{w}}=\mathrm{vec}(\hat{\vect{W}}^\top), \bm{\rho}=\mathrm{vec}(\vect{P}^\top), \tilde{\vect{w}}=\mathrm{vec}(\tilde{\vect{W}}^\top)$.\\
Using our notations, eq. \ref{eq:5} can be re-written as $ \vect{w} = \hat{\vect{w}} \log(t)+\bm{\rho}(t) $ when $\bm{\rho}(t)$ is bounded.\\
For any solution $\wvec(t)$, we define 
\begin{equation} \label{eq:rdef}
\vect{r}(t)=\vect{w}(t)-\hat{\vect{w}}\log t -\tilde{\vect{w}},
\end{equation}
where $\hat{\vect{w}}$ is the concatenation of $\what_1,...,\what_k$ which are the K-class SVM solution, so\\
\begin{equation} \label{eq:1}
\forall k,\ \forall n\in\mathcal{S}_k: \xtilde^\top\hat{\vect{w}}=1\ ;\ \theta = \min_{k}\left[\min_{n\notin\mathcal{S}_k}\xtilde^\top\hat{\vect{w}}\right]>1 
\end{equation}
and $\wtilde$ satisfies the equation:
\begin{equation} \label{eq:alpha_def}
\forall k,\ \forall n\in \mathcal{S}_k\ : \ \eta \e( (\tilde{\vect{w}}_k - \tilde{\vect{w}}_{y_n})^\top \vect{x}_n) = \alpha_{n,k}
\end{equation}
As we assume in the Theorem, this equation has a solution.\\
For each of the K classes, we define $\mathcal{\vect{P}}^k_1\in\mathcal{R}^{d \times d}$ as the orthogonal projection matrix to the subspace spanned by the support vector of the k'th class, and $\mathcal{\bar{\vect{P}}}^k_1=\vect{I}-\mathcal{\vect{P}}^k_1$ as the complementary projection.
Finally, we define $ \mathcal{\vect{P}}_1\in \mathcal{R}^{Kd \times Kd}$ and $ \mathcal{\bar{\vect{P}}}_1\in \mathcal{R}^{Kd \times Kd}$ as follows: 
$$
\mathcal{\vect{P}}_1= \mathrm{diag}(\mathcal{\vect{P}}^1_1,\mathcal{\vect{P}}^2_1,...,\mathcal{\vect{P}}^K_1)\ ,\  
\mathcal{\bar{\vect{P}}}_1= \mathrm{diag}(\mathcal{\bar{\vect{P}}}^1_1,\mathcal{\bar{\vect{P}}}^2_1,...,\mathcal{\bar{\vect{P}}}^K_1)\ 
$$
$$(\mathcal{\vect{P}}_1+\mathcal{\bar{\vect{P}}}_1=\vect{I}\in \mathcal{R}^{Kd \times Kd})$$
In the following section we will also use $\indicator{A}$, the indicator function, which is $1$ if $A$ is satisfied and 0 otherwise.
\subsection{Auxiliary Lemma}
\begin{restatable}{lemR}{multiAuxlemma}
	\label{lemma:1}
	We have
	\begin{equation} \label{eq:lemma}
	\exists C_1, t_1\ :\ \forall t>t_1 : (\vect{r}(t+1)-\vect{r}(t))^\top\vect{r}(t)\le C_1t^{-\theta}+C_2t^{-2}
	\end{equation}
	Additionally, $\forall \epsilon_1>0$, $\exists C_2,t_2$, such that $\forall t>t_2$, such that if
	\begin{equation} \label{eq:P_ge_epsilon1}
	||\mathcal{\vect{P}}_1\vect{r}(t)||>\epsilon_1
	\end{equation}
	then we can improve this bound to
	\begin{equation} \label{eq:lemmaImproved}
	\left(\vect{r}(t+1)-\vect{r}(t)\right)^\top\vect{r}(t)\le-C_3t^{-1}<0
	\end{equation}
\end{restatable}
We prove the Lemma below, in appendix section \ref{appendix: proof of auxilary Lemma CE}

\subsection{Proof of Theorem \ref{thm:mainMulti}}
Our goal is to show that $||\vect{r}(t)||$ is bounded, and therefore $\bm{\rho}(t) = \vect{r}(t)+\tilde{\vect{w}}$ is bounded.\\
To show this, we will upper bound the following equation
\begin{align} \label{eq: r(t+1) norm, CE}
\left\Vert \mathbf{r}\left(t+1\right)\right\Vert ^{2} & =\left\Vert \mathbf{r}\left(t+1\right)-\mathbf{r}\left(t\right)\right\Vert ^{2}+2\left(\mathbf{r}\left(t+1\right)-\mathbf{r}\left(t\right)\right)^{\top}\mathbf{r}\left(t\right)+\left\Vert \mathbf{r}\left(t\right)\right\Vert ^{2}
\end{align}
First, we note that first term in this equation can be upper-bounded
by 
\begin{align} \label{eq: square r difference, CE case}
& \left\Vert \mathbf{r}\left(t+1\right)-\mathbf{r}\left(t\right)\right\Vert ^{2}\nonumber \\
& \overset{\left(1\right)}{=}\left\Vert \mathbf{w}\left(t+1\right)-\hat{\mathbf{w}}\log\left(t+1\right)-\tilde{\mathbf{w}}-\mathbf{w}\left(t\right)+\hat{\mathbf{w}}\log\left(t\right)+\tilde{\mathbf{w}}\right\Vert ^{2}\nonumber \\
& \overset{\left(2\right)}{=}\left\Vert -\eta\nabla\mathcal{L}\left(\mathbf{w}\left(t\right)\right)-\hat{\mathbf{w}}\left[\log\left(t+1\right)-\log\left(t\right)\right]\right\Vert ^{2}\nonumber \\
& =\eta^{2}\left\Vert \nabla\mathcal{L}\left(\mathbf{w}\left(t\right)\right)\right\Vert ^{2}+\left\Vert \hat{\mathbf{w}}\right\Vert ^{2}\log^{2}\left(1+t^{-1}\right)+2\eta\hat{\mathbf{w}}^{\top}\nabla\mathcal{L}\left(\mathbf{w}\left(t\right)\right)\log\left(1+t^{-1}\right)\nonumber \\
& \overset{\left(3\right)}{\leq}\eta^{2}\left\Vert \nabla\mathcal{L}\left(\mathbf{w}\left(t\right)\right)\right\Vert ^{2}+\left\Vert \hat{\mathbf{w}}\right\Vert ^{2}t^{-2},
\end{align}
where in (1) we used eq. \ref{eq:rdef}, in (2) we used eq 2.2, and in (3) we used $\forall x>0: x\ge \log(1+x)>0$, and also that
\begin{equation}
\hat{\vect{w}}^\top \nabla \mathcal{L}(\vect{w})=\sumn \sumk \frac{\hat{\vect{w}}^\top(\vect{A}_{y_n}-\vect{A}_k)\vect{x}_n}{\sumr \e(\vect{w}^\top(\vect{A}_r-\vect{A}_k)\vect{x}_n)}<0
\end{equation}
since $\hat{\vect{w}}^\top(\vect{A}_r-\vect{A}_k)\vect{x}_n=(\hat{\vect{w}}_r-\hat{\vect{w}}_{y_n})\vect{x}_n<0, \forall k \ne y_n$ (we recall that $\hat{\vect{w}}_k$ is the K-class SVM solution).\\
Also, from Lemma \ref{lem: GD convergence} we know that 
\begin{equation}
\left\Vert \nabla\mathcal{L}\left(\mathbf{w}\left(t\right)\right)\right\Vert ^{2}=o\left(1\right)\,\mathrm{and}\,\sum_{t=0}^{\infty}\left\Vert \nabla\mathcal{L}\left(\mathbf{w}\left(t\right)\right)\right\Vert ^{2}<\infty\,.\label{eq: norm grad squared 1/t, CE case}
\end{equation}
Substituting eq. \ref{eq: norm grad squared 1/t, CE case} into eq. \ref{eq: square r difference, CE case},
and recalling that a $t^{-\nu}$ power series converges for any $\nu>1$,
we can find $C_{0}$ such that 
\begin{equation} \label{eq: square norm of r difference-1, CE}
\left\Vert \mathbf{r}\left(t+1\right)-\mathbf{r}\left(t\right)\right\Vert ^{2}=o\left(1\right)\,\mathrm{and}\,\sum_{t=0}^{\infty}\left\Vert \mathbf{r}\left(t+1\right)-\mathbf{r}\left(t\right)\right\Vert ^{2}=C_{0}<\infty\,.
\end{equation}
Note that this equation also implies that $\forall\epsilon_{0}$ 
\begin{equation}
\exists t_{0}:\forall t>t_{0}:\left|\left\Vert \mathbf{r}\left(t+1\right)\right\Vert -\left\Vert \mathbf{r}\left(t\right)\right\Vert \right|<\epsilon_{0}\,.
\end{equation}
Next, we would like to bound the second term in eq. \ref{eq: r(t+1) norm, CE}. From eq. \ref{eq:lemma} in Lemma \ref{lemma:1}, we can find $t_1,C_1$ such that $\forall t>t_1$:
\begin{equation} \label{eq:6}
(\vect{r}(t+1)-\vect{r}(t))^\top\vect{r}(t)\le C_1t^{-\theta}+C_2t^{-2}
\end{equation}
Thus, by combining eqs. \ref{eq:6} and \ref{eq: square norm of r difference-1, CE} into eq. \ref{eq: r(t+1) norm, CE}, we find:
\begin{align*}
&||\vect{r}(t)||^2-||\vect{r}(t_1)||^2\\
&=\sum_{u=t_1}^{t-1}\left[||\vect{r}(u+1)||^2-||\vect{r}(u)||^2\right]\\
&\le C_0+2\sum_{u=t_1}^{t-1} \left[C_1u^{-\theta}+C_2u^{-2}\right]
\end{align*}
which is bounded, since $\theta>1$ (eq. \ref{eq:1}). Therefore, $||\vect{r}(t)||$ is bounded.

\subsection{Proof of Lemma \ref{lemma:1}} \label{appendix: proof of auxilary Lemma CE}
\multiAuxlemma*
\noindent We wish to bound $(\rvec(t+1)-\rvec(t))^\top \rvec(t)$. 
First, we recall we defined $\xtilde \triangleq (\vect{A}_{y_n}-\vect{A}_k)\vect{x}_n$.
%\begin{equation}
\begin{align} \label{eq: (r(t+1)-r(t))r(t) bound Multi}
&(\vect{r}(t+1)-\vect{r}(t))^\top \vect{r}(t) = (-\eta \nabla \mathcal{L}(\vect{w}(t)) - \hat{\vect{w}}[\log(t+1)-\log(t)])^\top \vect{r}(t) \nonumber\\
&=  \left( \eta \sumn \frac{\sumk \e(-\wvec(t)^\top\xtilde)\xtilde}{\sumr \e(-\wvec(t)^\top\xtilder)} - \hat{\vect{w}}\log(1+t^{-1})\right) ^\top \vect{r}(t) \nonumber\\
&= \hat{\vect{w}}^\top \vect{r}(t)[t^{-1}-\log(1+t^{-1})] \\
&+\eta \sumn\sumk \left[  \frac{\e\left(-\wvec(t)^\top\xtilde\right)\xtilde^\top\rvec(t)}{\sumr \e\left(-\wvec(t)^\top\xtilder\right)} - t^{-1}\e\left(-\wtilde^\top \xtilde\right)\xtilde^\top\rvec(t)\indicator{n\in\mathcal{S}_k}\right],
\end{align}
%\end{equation}
where in the last line we used eqs. \ref{eq: what def with alpha for CE} and \ref{eq:alpha_def} to obtain
\begin{equation*}
	 \hat{\vect{w}}=\eta \sumn \sumk \alpha_{n,k}\xtilde \indicator{n\in\mathcal{S}_k} = \eta \sumn \sumk \e\left( -\tilde{\vect{w}}^\top\xtilde)\right) \xtilde \indicator{n\in\mathcal{S}_k},
\end{equation*}
where $\indicator{A}$ is the indicator function which is $1$ if $A$ is satisfied and 0 otherwise.
\newpage
The first term can be upper bounded by 
\begin{align}
& \hat{\mathbf{w}}^{\top}\mathbf{r}\left(t\right)\left[t^{-1}-\log\left(1+t^{-1}\right)\right]\nonumber \\
\leq & \max\left[\hat{\mathbf{w}}^{\top}\mathbf{r}\left(t\right),0\right]\left[t^{-1}-\log\left(1+t^{-1}\right)\right]\nonumber \\
\overset{\left(1\right)}{\leq} & \max\left[\hat{\mathbf{w}}^{\top}\mathbf{P}_{1}\mathbf{r}\left(t\right),0\right]t^{-2}\nonumber \\
\overset{\left(2\right)}{\leq} & \begin{cases}
\mathbf{\left\Vert \hat{\mathbf{w}}\right\Vert }\epsilon_{1}t^{-2} & ,\,\mathrm{if}\,\left\Vert \mathbf{P}_{1}\mathbf{r}\left(t\right)\right\Vert \leq\epsilon_{1}\\
o\left(t^{-1}\right) & ,\,\mathrm{if}\,\left\Vert \mathbf{P}_{1}\mathbf{r}\left(t\right)\right\Vert >\epsilon_{1}
\end{cases}\label{eq: w_hat r bound 1 Multi}
\end{align}
where in $\left(1\right)$ we used that ${\mathbf{P}}_{2}\hat{\mathbf{w}}=0$, and in $\left(2\right)$ we used that $\hat{\mathbf{w}}^{\top}\mathbf{r}\left(t\right)=o\left(t\right)$,
since 
\begin{align*}\hat{\mathbf{w}}^{\top}\mathbf{r}\left(t\right) & =\hat{\mathbf{w}}^{\top}\left(\mathbf{w}\left(0\right)-\eta\sum_{u=0}^{t}\nabla\mathcal{L}\left(\mathbf{w}\left(u\right)\right)-\hat{\mathbf{w}}\log\left(t\right)-\tilde{\mathbf{w}}\right)\\
 & \leq\hat{\mathbf{w}}^{\top}\left(\mathbf{w}\left(0\right)-\tilde{\mathbf{w}}-\hat{\mathbf{w}}\log\left(t\right)\right)+\eta\left\Vert \hat{\mathbf{w}}\right\Vert \sum_{u=0}^{t}\left\Vert \nabla\mathcal{L}\left(\mathbf{w}\left(u\right)\right)\right\Vert \\
 & \leq O\left(\log\left(t\right)\right)+\eta\left\Vert \hat{\mathbf{w}}\right\Vert \sum_{u=0}^{\left\lceil \sqrt{t}\right\rceil }\left\Vert \nabla\mathcal{L}\left(\mathbf{w}\left(u\right)\right)\right\Vert +\eta\left\Vert \hat{\mathbf{w}}\right\Vert \sum_{u=\left\lceil \sqrt{t}\right\rceil }^{t}\left\Vert \nabla\mathcal{L}\left(\mathbf{w}\left(u\right)\right)\right\Vert \\
 & \leq O\left(\log\left(t\right)\right)+\eta\left\lceil \sqrt{t}\right\rceil \left\Vert \hat{\mathbf{w}}\right\Vert \max_{0\leq u\leq\left\lceil \sqrt{t}\right\rceil }\left\Vert \nabla\mathcal{L}\left(\mathbf{w}\left(u\right)\right)\right\Vert +\eta t\left\Vert \hat{\mathbf{w}}\right\Vert \max_{\left\lceil \sqrt{t}\right\rceil \leq u\leq t}\left\Vert \nabla\mathcal{L}\left(\mathbf{w}\left(u\right)\right)\right\Vert \\
 & =O\left(\log\left(t\right)\right)+\left\lceil \sqrt{t}\right\rceil O\left(1\right)+o\left(1\right)t=o\left(t\right) \, ,
\end{align*}
where in the last line we used that $\nabla\mathcal{L}\left(\mathbf{w}\left(t\right)\right)=o\left(1\right)$,
from Lemma \ref{lem: GD convergence}.\\
Next, we wish to upper bound the second term in eq. \ref{eq: (r(t+1)-r(t))r(t) bound Multi}:
\begin{align}
&\eta \sumn  \sumk \left[ \frac{\e\left(-\wvec(t)^\top\xtilde\right)\xtilde^\top\rvec(t)}{\sumr \e\left(-\wvec(t)^\top\xtilder\right)} - t^{-1}\e\left(-\wtilde^\top \xtilde\right)\xtilde^\top\rvec(t) \indicator{n\in\mathcal{S}_k} \right]
\end{align}
\newpage
We examine each term $n$ in the sum:
\begin{flalign} \label{eq: second term}
&\sumk \left[ \frac{\e\left(-\wvec(t)^\top\xtilde\right)\xtilde^\top\rvec(t)}{\sumr \e\left(-\wvec(t)^\top\xtilder\right)} - t^{-1}\e\left(-\wtilde^\top \xtilde\right)\xtilde^\top\rvec(t)\indicator{n\in\mathcal{S}_k} \right] &\nonumber\\
& = \sumk \left[ \frac{\e\left(-\wvec(t)^\top\xtilde\right)\xtilde^\top\rvec(t)}{1+\sumrney \e\left(-\wvec(t)^\top\xtilder\right)} -  t^{-1}\e\left(-\wtilde^\top \xtilde\right)\xtilde^\top\rvec(t)\indicator{n\in\mathcal{S}_k} \right] \nonumber\\
& \overset{(1)}{\le}\sumk \left( \e\left(-\wvec(t)^\top\xtilde\right) 
- t^{-1}\e\left(-\wtilde^\top \xtilde\right)\indicator{n\in\mathcal{S}_k } \right)\indicator{ \xtilde^\top\rvec(t)\ge 0 }\xtilde^\top\rvec(t) \nonumber\\
& +
\sumk \left( \e\left(-\wvec(t)^\top\xtilde\right)\left(1-\sumrney \e\left(-\wvec(t)^\top\xtilder\right)\right)\right.\nonumber\\
& \left.-  t^{-1}\e\left(-\wtilde^\top \xtilde\right) \indicator{ n\in\mathcal{S}_k } \right)\indicator{ \xtilde^\top\rvec(t)< 0 } \xtilde^\top\rvec(t) \nonumber\\
& = \sumk \left( \e\left(-\wvec(t)^\top\xtilde\right)
- t^{-1}\e\left(-\wtilde^\top \xtilde\right)\indicator{ n \in \mathcal{S}_k } \right)\xtilde^\top\rvec(t) \nonumber\\
& -\sumk \sumrney \e\left(-\wvec(t)^\top(\xtilde+\xtilder)\right)\xtilde^\top\rvec(t) \indicator{ \xtilde^\top\rvec(t)< 0 } \nonumber\\
& \overset{(2)}{\le} \sumk \left( \e\left(-\wvec(t)^\top\xtilde\right)
- t^{-1}\e\left(-\wtilde^\top \xtilde\right)\indicator{ n \in \mathcal{S}_k } \right)\xtilde^\top\rvec(t) \nonumber\\
& -K^2 \e\left(-\wvec(t)^\top(\xtildef{k_1}+\xtildef{r_1})\right)\xtildef{k_1}^\top\rvec(t)\indicator{ \xtildef{k_1}^\top\rvec(t)< 0 } ,
\end{flalign}
where in (1) we used $\forall x\ge0:\ 1-x\le \frac{1}{1+x} \le 1$ and in (2) we defined:
\begin{flalign}
&(k_1,r_1) = \argmax_{k,r} \left\vert \e\left(-\wvec(t)^\top(\xtilde+\xtilder)\right)\xtilde^\top\rvec(t) \indicator{ \xtilde^\top\rvec(t)< 0} \right\vert \nonumber
\end{flalign}
\newpage
Recalling that $\wvec(t) = \what \log(t) + \wtilde + \rvec(t)$, eq. \ref{eq: second term} can be upper bounded by
\begin{flalign}
& \sumk t^{-\what^\top\xtilde}\e\left(-\wtilde^\top\xtilde\right)\e\left(-\rvec(t)^\top\xtilde\right)\xtilde^\top\rvec(t) \indicator{ \xtilde^\top\rvec(t)\ge 0\ ,\ n\notin \mathcal{S}_k  }&\nonumber\\
&
+ \sumk t^{-1}\e\left(-\wtilde^\top\xtilde\right) \left[\e\left(-\rvec(t)^\top\xtilde\right)-1\right]\xtilde^\top\rvec(t)\indicator{ \xtilde^\top\rvec(t)\ge 0\ ,\ n\in\mathcal{S}_k } \nonumber\\
& +
\sumk t^{-\what^\top\xtilde}\e\left(-\wtilde^\top\xtilde\right)\e\left(-\rvec(t)^\top\xtilde\right)\xtilde^\top\rvec(t) \indicator{ \xtilde^\top\rvec(t) < 0\ ,\ n\notin \mathcal{S}_k }\nonumber\\
& + \sumk t^{-1}\e\left(-\wtilde^\top\xtilde\right) \left[\e\left(-\rvec(t)^\top\xtilde\right)-1\right]\xtilde^\top\rvec(t)\indicator{ \xtilde^\top\rvec(t) < 0\ ,\ n\in\mathcal{S}_k } \nonumber\\
& -K^2 \e(-\wvec(t)^\top \xtildef{r_1}) t^{-\what^\top\xtildef{k_1}}\e\left(-\wtilde^\top\xtildef{k_1}\right) \e\left(-\rvec(t)^\top\xtildef{k_1}\right)\xtildef{k_1}^\top\rvec(t)\indicator{ \xtildef{k_1}^\top\rvec(t)< 0} \nonumber\\
& \overset{(1)}{\le} K t^{-\theta}\e\left(-\min_{n,k}\wtilde^\top\xtilde\right)
+ \phi(t),
\end{flalign}
where in (1) we used $xe^{-x}<1,\ \forall x:\ (e^{-x}-1)x<0$, $\theta = \min_{k}\left[\min_{n\notin\mathcal{S}_k}\xtilde^\top\hat{\vect{w}}\right]>1$ (eq. \ref{eq:1}) and denoted:{\small
\begin{flalign*}
& \phi(t)=
\sumk t^{-\what^\top\xtilde}\e\left(-\wtilde^\top\xtilde\right)\e\left(-\rvec(t)^\top\xtilde\right)\xtilde^\top\rvec(t) \indicator{ \xtilde^\top\rvec(t) < 0,\ n\notin \mathcal{S}_k } &\nonumber\\
& + \sumk t^{-1}\e\left(-\wtilde^\top\xtilde\right) \left[\e\left(-\rvec(t)^\top\xtilde\right)-1\right]\xtilde^\top\rvec(t)\indicator{ \xtilde^\top\rvec(t) < 0, n\in\mathcal{S}_k } \nonumber\\
& -K^2 \e(-\wvec(t)^\top \xtildef{r_1}) t^{-\what^\top\xtildef{k_1}}\e\left(-\wtilde^\top\xtildef{k_1}\right) \e\left(-\rvec(t)^\top\xtildef{k_1}\right)\xtildef{k_1}^\top\rvec(t)\indicator{ \xtildef{k_1}^\top\rvec(t)< 0}.
\end{flalign*}}
We use the fact that $\forall x:\ (e^{-x}-1)x<0$ and therefore $\forall (n,k)$:
\begin{flalign}
& t^{-\what^\top\xtilde}\e\left(-\wtilde^\top\xtilde\right)\e\left(-\rvec(t)^\top\xtilde\right)\xtilde^\top\rvec(t) \indicator{ \xtilde^\top\rvec(t) < 0 }<0 \nonumber\\
& t^{-1}\e\left(-\wtilde^\top\xtilde\right) \left[\e\left(-\rvec(t)^\top\xtilde\right)-1\right]\xtilde^\top\rvec(t)\indicator{ \xtilde^\top\rvec(t) < 0 } < 0 ,
\end{flalign}
to show that $\phi(t)$ is strictly negative. If $\xtildef{k_1}^\top \rvec \ge 0$ then from the last two equations:
\begin{flalign}
& \phi(t)=
\sumk t^{-\what^\top\xtilde}\e\left(-\wtilde^\top\xtilde\right)\e\left(-\rvec(t)^\top\xtilde\right)\xtilde^\top\rvec(t) \indicator{ \xtilde^\top\rvec(t) < 0,\ n\notin \mathcal{S}_k } &\nonumber\\
& + \sumk t^{-1}\e\left(-\wtilde^\top\xtilde\right) \left[\e\left(-\rvec(t)^\top\xtilde\right)-1\right]\xtilde^\top\rvec(t)\indicator{ \xtilde^\top\rvec(t) < 0,\ n\in\mathcal{S}_k }<0
\end{flalign}
If $\xtildef{k_1}^\top \rvec < 0$ then we note that $-\xtildef{r_1}^\top\rvec(t) \le - \xtildef{k_1}^\top\rvec(t)$ since:\\
1. If $\xtildef{r_1}^\top\rvec(t)\ge0$ then this is immediate since $-\xtildef{r_1}^\top\rvec(t) \le 0 \le - \xtildef{k_1}^\top\rvec(t)$.\\
2. If $\xtildef{r_1}^\top\rvec(t)<0$ then from $(k_1,r_1)$ definition:
\begin{flalign*}
&\left\vert \e\left(-\wvec(t)^\top(\xtildef{k_1}+\xtildef{r_1})\right)\xtildef{r_1}^\top\rvec(t) \right\vert \le  \left\vert \e\left(-\wvec(t)^\top(\xtildef{k_1}+\xtildef{r_1})\right)\xtildef{k_1}^\top\rvec(t) \right\vert,
\end{flalign*}
and therefore
\begin{flalign*}
-\xtildef{r_1}^\top\rvec(t) = \left\vert\xtildef{r_1}^\top\rvec(t)\right\vert \le \left\vert \xtildef{k_1}^\top\rvec(t) \right\vert = - \xtildef{k_1}^\top\rvec(t).
\end{flalign*}
We divide into cases:\\
1. If $n \notin \mathcal{S}_{k_1}$ then we examine the sum
\begin{flalign}
& t^{-\what^\top\xtildef{k_1}}\e\left(-\wtilde^\top\xtildef{k_1}\right)\e\left(-\rvec(t)^\top\xtildef{k_1}\right)\xtildef{k_1}^\top\rvec(t) \indicator{ \xtildef{k_1}^\top\rvec(t) < 0 } & \nonumber\\
& -K^2 \e(-\wvec(t)^\top \xtildef{r_1}) t^{-\what^\top\xtildef{k_1}}\e\left(-\wtilde^\top\xtildef{k_1}\right) \e\left(-\rvec(t)^\top\xtildef{k_1}\right)\xtildef{k_1}^\top\rvec(t)\indicator{ \xtildef{k_1}^\top\rvec(t)< 0} \nonumber
\end{flalign}
The first term is negative and the second is positive. From Lemma \ref{lem: convergence of softmax-cross-entropy} $\wvec(t)^\top \xtildef{r_1}\to\infty$. Therefore $\exists t_3$ so that $\forall t>t_3:\ \e(-\wvec(t)^\top \xtildef{r_1})<K^2$ and therefore this sum is strictly negative since
\begin{flalign*}
&\left\vert \frac{ K^2 \e(-\wvec(t)^\top \xtildef{r_1}) t^{-\what^\top\xtildef{k_1}}\e\left(-\wtilde^\top\xtildef{k_1})\right) \e\left(-\xtildef{k_1}^\top\rvec(t)\right)\xtildef{k_1}^\top\rvec(t)\indicator{ \xtildef{k_1}^\top\rvec(t)< 0 }}{t^{-\what^\top\xtildef{k_1}}\e\left(-\wtilde^\top\xtildef{k_1}\right)\e\left(-\rvec(t)^\top\xtildef{k_1}\right)\xtildef{k_1}^\top\rvec(t) \indicator{ \xtildef{k_1}^\top\rvec(t) < 0 }} \right \vert & \nonumber\\
& = \left\vert K^2 \e(-\wvec(t)^\top \xtildef{r_1}) \right \vert < 1,\ \forall t>t_3
\end{flalign*}
2. If $n \in \mathcal{S}_{k_1}$ then we examine the sum{\small
\begin{flalign*}
& t^{-1}\e\left(-\wtilde^\top\xtildef{k_1}\right) \left[\e\left(-\rvec(t)^\top\xtildef{k_1}\right)-1\right]\xtildef{k_1}^\top\rvec(t)\indicator{ \xtildef{k_1}^\top\rvec(t) < 0 } & \nonumber\\
& -K^2 \e(-\wvec(t)^\top \xtildef{r_1}) t^{-\what^\top\xtildef{k_1}}\e\left(-\wtilde^\top\xtildef{k_1}\right) \e\left(-\rvec(t)^\top\xtildef{k_1}\right)\xtildef{k_1}^\top\rvec(t)\indicator{ \xtildef{k_1}^\top\rvec(t)< 0}
\end{flalign*}}
a. If $|\xtildef{k_1}^\top\rvec(t)|> C_0$ then $\exists t_4$ such that $\forall t>t_4$ this sum can be upper bounded by zero since
\begin{flalign} \label{eq: g(t) 2 terms sum, n s.v. and xr not bounded}
&\left \vert \frac{K^2 \e(-\wvec(t)^\top \xtildef{r_1}) t^{-\what^\top\xtildef{k_1}}\e\left(-\wtilde^\top\xtildef{k_1})\right) \e\left(-\xtildef{k_1}^\top\rvec(t)\right)\xtildef{k_1}^\top\rvec(t)\indicator{ \xtildef{k_1}^\top\rvec(t)< 0 }}{t^{-1}\e\left(-\wtilde^\top\xtildef{k_1}\right) \left[\e\left(-\rvec(t)^\top\xtildef{k_1}\right)-1\right]\xtildef{k_1}^\top\rvec(t)\indicator{ \xtildef{k_1}^\top\rvec(t) < 0 }} \right \vert &\nonumber\\
& =  \frac{K^2 \e(-\wvec(t)^\top \xtildef{r_1})}{ 1-\e\left(\rvec(t)^\top\xtildef{k_1}\right)} \le \frac{K^2 \e(-\wvec(t)^\top \xtildef{r_1})}{ 1-\e\left(-C_0\right)}  < 1,\ \forall t>t_4
\end{flalign}
where in the last transition we used Lemma \ref{lem: convergence of softmax-cross-entropy}.\\
b.  If $|\xtildef{k_1}^\top\rvec(t)|\le C_0$ then we can find constant $C_5$ so that eq. \ref{eq: g(t) 2 terms sum, n s.v. and xr not bounded} can be upper bounded by
\begin{flalign}
K^2 t^{-\what^\top(\xtildef{k_1}+\xtildef{r_1})}\e\left(-\wtilde^\top(\xtildef{k_1}+\xtildef{r_1})\right) \e\left(2C_0\right)C_0 \le C_5t^{-2},
\end{flalign}
since $-\xtildef{r_1}^\top\rvec(t) \le -\xtildef{k_1}^\top\rvec(t)\le C_0$ and by definition, $\forall(n,k):$ $\what^\top\xtilde\ge1$.\\
Therefore, eq. \ref{eq: second term} can be upper bounded by
\begin{flalign}
K t^{-\theta}\e\left(-\min_{n,k}\wtilde^\top\xtilde\right)+C_5t^{-2}
\end{flalign}
If, in addition, $\exists k,n\in\mathcal{S}_k:\ |\xtilde^\top\rvec(t)|> \epsilon_2$ then 
\begin{flalign}
&t^{-1}\e\left(-\wtilde^\top\xtilde\right) \left[\e\left(-\rvec(t)^\top\xtilde\right)-1\right]\xtilde^\top\rvec(t)\\
& \le \begin{cases}
-t^{-1}\e\left(-\max_{n,k}\wtilde^\top\xtilde\right) \left[1-\e\left(-\epsilon_{2}\right)\right]\epsilon_{2} & ,\text{ if }\rvec(t)^\top\xtilde\ge0\\
-t^{-1}\e\left(-\max_{n,k}\wtilde^\top\xtilde\right) \left[\e\left(\epsilon_{2}\right)-1\right]\epsilon_{2} & ,\text{ if }\rvec(t)^\top\xtilde < 0
\end{cases}
\end{flalign}
and we can improve this bound to
\begin{flalign} \label{eq: second term improved bound}
-C'' t^{-1}<0,
\end{flalign}
where $C''$ is the minimum between $\e\left(-\max_{n,k}\wtilde^\top\xtilde\right) \left[1-\e\left(-\epsilon_{2}\right)\right]\epsilon_{2}$ and \\ $\e\left(-\max_{n,k}\wtilde^\top\xtilde\right) \left[\e\left(\epsilon_{2}\right)-1\right]\epsilon_{2}$.
%%%%%%%%%%%%%%%%%%%%%%%%%%%%%%
To conclude: \\
1. If $\left\Vert \mathbf{P}_{1}\mathbf{r}\left(t\right)\right\Vert \geq\epsilon_{1}$
(as in Eq. \ref{eq: w_hat r bound 1 Multi}), we have that 
\begin{equation}
\max_{k,n\in\mathcal{S}_k}\left|\xtilde^{\top}\mathbf{r}\left(t\right)\right|^{2}\overset{\left(1\right)}{\geq}\frac{1}{\left|\mathcal{S}\right|}\sum_{k,n\in\mathcal{S}_k}\left|\xtilde^{\top}\mathbf{P}_{1}\mathbf{r}\left(t\right)\right|^{2}=\frac{1}{\left|\mathcal{S}\right|}\left\Vert \mathbf{X}_{\mathcal{S}}^{\top}\mathbf{P}_{1}\mathbf{r}\left(t\right)\right\Vert ^{2}\overset{\left(2\right)}{\geq}\frac{1}{\left|\mathcal{S}\right|}\sigma_{\min}^{2}\left(\mathbf{X}_{\mathcal{S}}\right)\epsilon_{1}^{2}
\end{equation}
where in $\left(1\right)$ we used $\mathbf{P}_{1}^{\top}\xtilde=\xtilde$
$\forall k,\ n\in\mathcal{S}_k$, in $\left(2\right)$ we denoted by $\sigma_{\min}\left(\mathbf{X}_{\mathcal{S}}\right)$,
the minimal non-zero singular value of $\mathbf{X}_{\mathcal{S}}$ and used
eq. \ref{eq:P_ge_epsilon1}. Therefore, for some $(n,k)$, $\left|\xtilde^{\top}\mathbf{r}\right|\geq\epsilon_{2}\triangleq\left|\mathcal{S}\right|^{-1}\sigma_{\min}^{2}\left(\mathbf{X}_{\mathcal{S}}\right)\epsilon_{1}^{2}$.
If $||\mathcal{\vect{P}}_1 \vect{r}(t)||\ge\epsilon_1$, then combining eq. \ref{eq: w_hat r bound 1 Multi} with eq.  \ref{eq: second term improved bound} we find that eq. \ref{eq: (r(t+1)-r(t))r(t) bound Multi} can be upper bounded by:
$$ \left(\vect{r}(t+1)- \vect{r}(t)\right)^\top \vect{r} (t)\le-C''t^{-1}+o(t^{-1})$$
This implies that $\exists C_2<C''$ and $\exists t_2>0$ such that eq. \ref{eq:lemmaImproved} holds. This implies also that eq. \ref{eq:lemma} holds for $||\mathcal{\vect{P}}_1 \vect{r}(t)||\ge\epsilon_1$.\\
2. If $||\mathcal{\vect{P}}_1r(t)||<\epsilon_1$, we obtain (for some positive constants $C_3,C_4$):
$$(\vect{r}(t+1)- \vect{r}(t))^\top \vect{r}(t)\le C_3 t^{-\theta}+C_4t^{-2}$$\\
Therefore, $\exists t_1>0$ and $C_1$ such that eq. \ref{eq:lemma} holds.

\section{An experiment with stochastic gradient descent \label{sec:Additional-Figures}}

\begin{figure}[H]
\begin{centering}
\includegraphics[width=1\columnwidth]{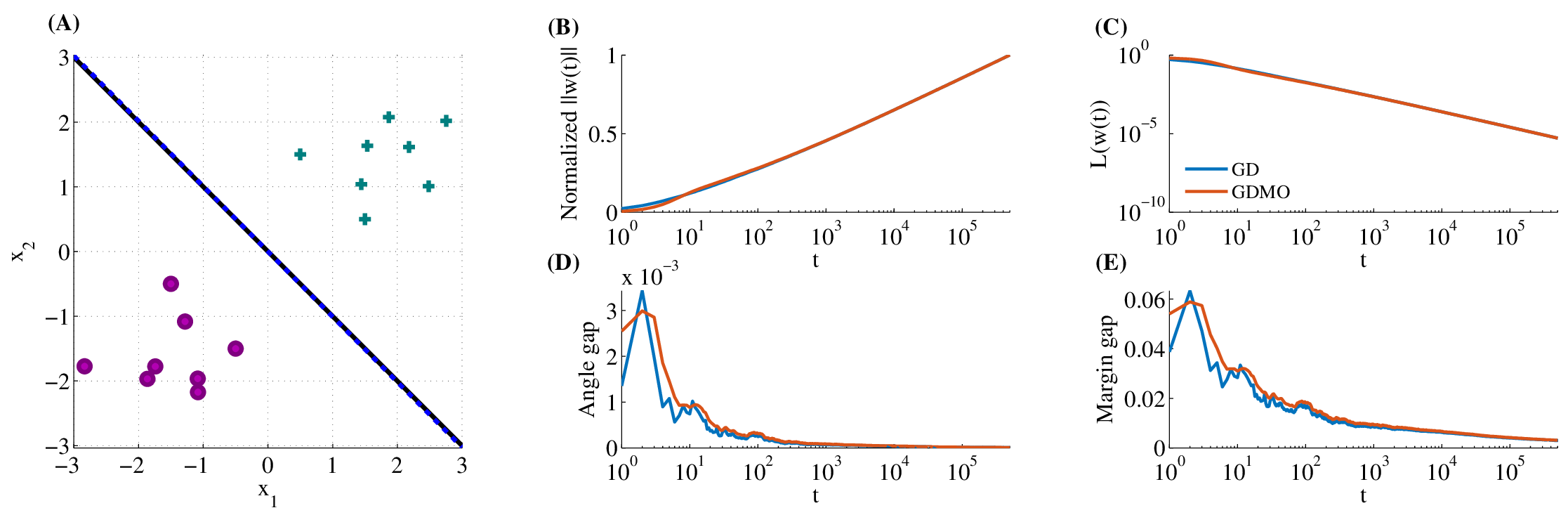} 
\par\end{centering}
\caption{Same as Fig. \ref{fig:Synthetic-dataset}, except stochastic gradient
decent is used (with mini-batch of size 4), instead of GD.\label{fig: SGD} }
\end{figure}

\pagebreak

\end{document}